\documentclass{article}

\PassOptionsToPackage{numbers,sort&compress}{natbib}
\usepackage[preprint]{neurips_2026}

\usepackage[utf8]{inputenc}
\usepackage[T1]{fontenc}
\usepackage[strings]{underscore}
\usepackage{hyperref}
\usepackage{url}
\usepackage{booktabs}
\usepackage{tabularx}
\usepackage{longtable}
\usepackage{amsfonts}
\usepackage{nicefrac}
\usepackage{microtype}
\usepackage[table]{xcolor}

\usepackage{amsmath,amssymb,amsthm}
\usepackage{multirow}
\usepackage{graphicx}
\usepackage{float}
\usepackage{morefloats}
\usepackage[section]{placeins}
\usepackage{dblfloatfix}
\usepackage{subcaption}
\usepackage{import}
\usepackage{titletoc}
\usepackage{listings}

\DeclareUnicodeCharacter{00B1}{\ensuremath{\pm}}
\DeclareUnicodeCharacter{00B2}{\textsuperscript{2}}
\DeclareUnicodeCharacter{00B3}{\textsuperscript{3}}
\DeclareUnicodeCharacter{00B9}{\textsuperscript{1}}
\DeclareUnicodeCharacter{2070}{\textsuperscript{0}}
\DeclareUnicodeCharacter{2074}{\textsuperscript{4}}
\DeclareUnicodeCharacter{2075}{\textsuperscript{5}}
\DeclareUnicodeCharacter{2076}{\textsuperscript{6}}
\DeclareUnicodeCharacter{2077}{\textsuperscript{7}}
\DeclareUnicodeCharacter{2078}{\textsuperscript{8}}
\DeclareUnicodeCharacter{2079}{\textsuperscript{9}}
\DeclareUnicodeCharacter{00D7}{times}
\DeclareUnicodeCharacter{2192}{\ensuremath \to }
\DeclareUnicodeCharacter{223C}{\ensuremath{\sim}}
\DeclareUnicodeCharacter{2248}{\ensuremath{\approx}}
\DeclareUnicodeCharacter{2212}-

\definecolor{tableheadgray}{gray}{0.96}
\definecolor{quadcodebg}{HTML}{F7F9FA}
\definecolor{quadcodekw}{HTML}{1F5E8C}
\definecolor{quadcodestring}{HTML}{8A4B08}
\definecolor{quadcodecomment}{HTML}{6B7280}
\definecolor{quadcodeaccent}{HTML}{0F766E}
\lstdefinestyle{appendixpython}{
  language=Python,
  basicstyle=\ttfamily\scriptsize,
  columns=fullflexible,
  keepspaces=true,
  breaklines=true,
  showstringspaces=false,
  frame=single,
  framerule=0.25pt,
  rulecolor=\color{black!30},
  xleftmargin=1em,
  xrightmargin=1em
}
\lstdefinestyle{quadnormcode}{
  style=appendixpython,
  backgroundcolor=\color{quadcodebg},
  rulecolor=\color{quadcodeaccent!45},
  keywordstyle=\color{quadcodekw}\bfseries,
  stringstyle=\color{quadcodestring},
  commentstyle=\color{quadcodecomment}\itshape,
  emph={QuadNorm,trapezoidal_weights_nd,torch,nn},
  emphstyle=\color{quadcodeaccent}
}
\newcommand{\tableheadrulebelow}{\noalign{\kern-\belowrulesep{\color{tableheadgray}\hrule height \belowrulesep}}}
\newcommand{\tableheadruleabove}{\noalign{{\color{tableheadgray}\hrule height \aboverulesep}\kern-\aboverulesep}}
\theoremstyle{plain}
\newtheorem{theorem}{Theorem}
\newtheorem{proposition}[theorem]{Proposition}
\newtheorem{definition}{Definition}
\theoremstyle{remark}
\newtheorem{remark}[theorem]{Remark}
\theoremstyle{definition}

\theoremstyle{plain}

\extrafloats{1000}
\newcommand{\R}{\mathbb{R}}
\newcommand{\E}{\mathbb{E}}

\newcommand{\norm}[1]{\lVert #1 \rVert}

\DeclareRobustCommand{\Eqref}[1]{\hyperref[#1]{Eq.~\ref*{#1}}}
\DeclareRobustCommand{\Eqrefrange}[2]{\hyperref[#1]{Eqs.~\ref*{#1}} to~\hyperref[#2]{\ref*{#2}}}
\renewcommand{\eqref}[1]{\Eqref{#1}}

\title{QuadNorm: Resolution-Robust Normalization \\ for Neural Operators}
\author{
  Bum Jun Kim \quad Makoto Kawano \quad Yusuke Iwasawa \quad Yutaka Matsuo\\
  \normalfont The University of Tokyo, Japan\\
  \normalfont\texttt{\{bumjun.kim,kawano,iwasawa,matsuo\}@weblab.t.u-tokyo.ac.jp}
}

\begin{document}

\maketitle

\begin{abstract}
	Normalization layers in neural operators usually compute statistics by uniformly averaging discrete grid values, making the normalization itself discretization-dependent and thereby a source of transfer error across different resolutions or meshes. To enable discretization robustness, we introduce a quadrature normalization family that replaces existing uniform averaging in normalization layers with numerical quadrature: QuadNorm and BlendQuadNorm. On endpoint-inclusive uniform grids, the proposed quadrature moments are $O(h^2)$-consistent across discretizations, meaning that their cross-resolution mismatch decays quadratically with grid spacing. A transfer-error bound then predicts how normalization-induced mismatch scales with both the resolution gap and network depth. The experiments show the same gap- and depth-scaling trends predicted by the transfer-error bound. On Darcy, QuadNorm delivers the best cross-resolution performance at every tested target resolution from $64^2$ to $256^2$; on real-data benchmarks, Transolver with QuadNorm achieves nearly resolution-invariant transfer. The largest gains appear on nonperiodic PDEs and nonspectral architectures, where native-resolution improvements also emerge. We also validate BlendQuadNorm, which stays close to LayerNorm behavior and serves as a conservative default for periodic FNO settings. These results identify normalization as a previously overlooked source of resolution dependence in neural operators.
\end{abstract}

\begingroup
\let\FloatBarrier\relax
\section{Introduction}
\label{sec:intro}
\endgroup
Neural operators are designed to approximate maps between infinite-dimensional function spaces from finite-dimensional discretizations \citep{kovachki2023neural, li2021fno, lu2021deeponet, kovachki2021universal}. A defining appeal of this paradigm is discretization invariance: A model trained on one grid should generalize to different resolutions without retraining, a goal that also motivates physics-informed, factorized, and transformer-based operator variants \citep{li2022pino, tran2023ffno, li2022transformerpde, wang2025cvit}. Accordingly, cross-resolution evaluation, where models are trained on one discretization and evaluated on other discretizations, such as $32^2 \to 64^2$ or $32^2 \to 128^2$, is a common generalization setting in neural operators.

In practice, however, several components of neural operator architectures are not resolution-invariant. Normalization layers are a critical yet under-examined example. Consider a standard LayerNorm \citep{ba2016layer} applied to a feature map $x$ of shape $(B, C, H, W)$ sampled on an $H$-by-$W$ grid. The per-sample mean is
\begin{equation}\label{eq:discrete-mean}
	\hat{\mu}^{\text{disc}} = \frac{1}{CHW} \sum_{c=1}^{C} \sum_{i=1}^{H} \sum_{j=1}^{W} x_{c}(i,j).
\end{equation}
If the same continuous field is resampled at a different resolution $(H', W') \neq (H, W)$, the discrete mean changes, even though the underlying function is the same. This discretization dependence creates a pathway for the normalization to break the resolution invariance that the spectral convolution carefully maintains \citep{li2021fno, reno2023}. It is also orthogonal to recent analyses of discretization, truncation, and mismatch error in operator kernels and training pipelines \citep{lanthaler2024discretization, subedi2024controlling, gao2025discretization}.

\begin{figure}[t]
	\centering
	\includegraphics[width=\textwidth]{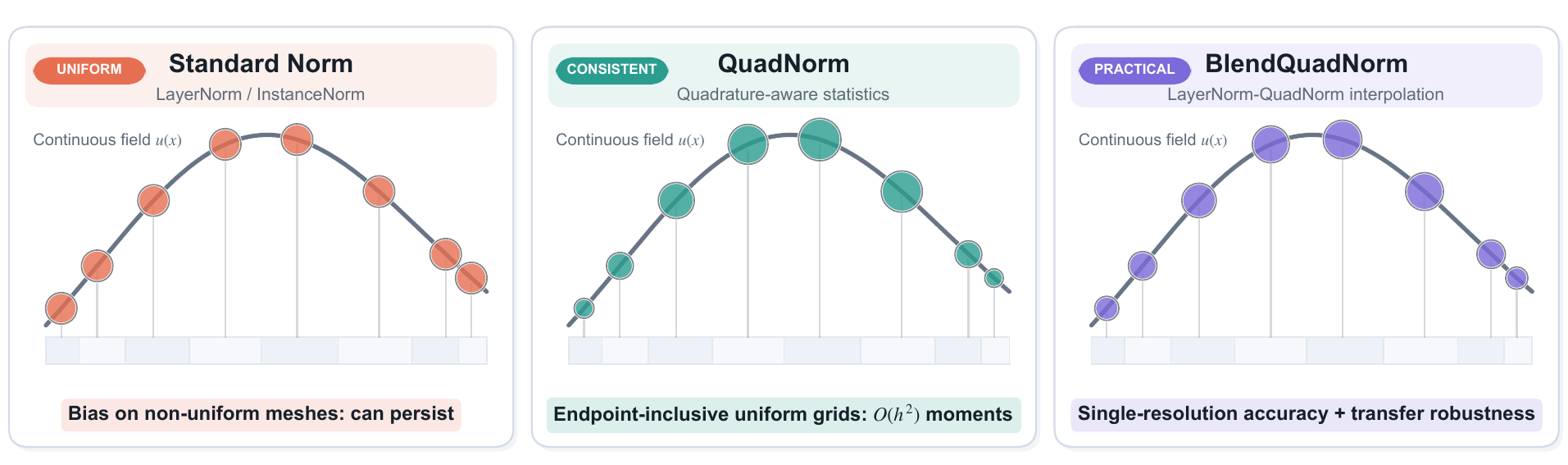}
	\caption{Schematic comparison of normalization strategies for neural operators. The left schematic shows standard norms such as LayerNorm and InstanceNorm, which assign uniform point weights $w_i = 1/N$ and can induce nonvanishing bias on nonuniform mesh families. The middle schematic shows QuadNorm, which uses quadrature and control-volume weights $w_i = |V_i|$ and yields second-order-consistent moments on endpoint-inclusive uniform grids. The right schematic shows BlendQuadNorm, which interpolates between the two and combines single-resolution accuracy with transfer robustness.}
	\label{fig:schematic}
\end{figure}

The consequences are both practical and measurable. Cross-resolution degradation appears when an existing neural operator trained at $64^2$ and tested at $128^2$ encounters shifted normalization statistics, thereby degrading prediction accuracy during resolution transfer. An empirical error-comparison ablation in Appendix~\ref{sec:exp-error-decomp} shows that quadrature-consistent normalization reduces this degradation to only 0.22 times the no-normalization reference on Darcy. Inconsistency across grids also appears on nonuniform meshes, where discrete averaging assigns equal weight to all grid points regardless of their spatial coverage and therefore further distorting the statistics, as shown in Figure~\ref{fig:nonuniform-bias}.

\paragraph{Why the issue remained overlooked.} Most benchmarks evaluate at the training resolution, which masks the effect of normalization-induced cross-resolution degradation. Yet the consequences vary depending on both the architecture and the underlying partial differential equation (PDE). In nonspectral architectures, ensuring normalization consistency can produce substantial gains even at the native resolution, far larger than the effect observed at a 2$\times$ gap in the Fourier Neural Operator (FNO).

As neural operators increasingly scale toward foundation models \citep{herde2024poseidon, hao2024dpot, mccabe2024multiple, subramanian2024towards} and long-horizon neural PDE solvers \citep{lippe2023pderefiner}, ensuring discretization consistency at every layer becomes critical. To solve this problem, we introduce two variants: 1) QuadNorm, which replaces uniform weights with trapezoidal quadrature, the classical rule that gives endpoint samples half weight and interior samples full weight on a one-dimensional, endpoint-inclusive uniform grid; and 2) BlendQuadNorm, a drop-in replacement for normalization layers with no additional hyperparameters. Figure~\ref{fig:schematic} illustrates the core idea by replacing uniform weights with quadrature weights in the normalization statistics. The QuadNorm advantage amplifies with model scale, rising from $18\%$ at 307K to $38\%$ at 9.6M, suggesting principled normalization will matter more as models grow. BlendQuadNorm shows native-resolution equivalence to LayerNorm for FNO on the Darcy benchmark, as reported in Appendix Table~\ref{tab:tost}.

\paragraph{Contributions.} While the method is a principled application of numerical quadrature, our experimental results reveal consequences that are architecture- and PDE-dependent, nonobvious, and not predictable from numerical analysis alone. We identify normalization as an overlooked source of discretization dependence in neural operators and formalize discretization consistency in Definition~\ref{def:consistency}. We also derive, in Appendix~\ref{app:proof-transfer}, a conditional transfer-error propagation bound with an $O(L \cdot \max(h^p,h'^p))$ normalization-induced term. That bound yields two testable scaling laws, one for resolution gaps and one for depth, and both are confirmed experimentally, with the scaling comparison reported in Figure~\ref{fig:scaling}. The quadrature normalization family spans a practical spectrum: pure quadrature-weighted moments, including the QuadNorm variants studied here and the LayerNorm-style endpoint with $\alpha = 0$, have provable $O(h^2)$ consistency on endpoint-inclusive uniform grids, while BlendQuadNorm preserves native accuracy in practice in a 10-seed native Darcy comparison using FNO, reported in Appendix Table~\ref{tab:tost}.

We also conduct a broad empirical study across multiple PDEs, architectures, and various settings, each evaluated with multiple seeds and rigorous statistical testing. For nonperiodic PDEs with nonspectral architectures, QuadNorm can yield substantial native-resolution improvements, including 26\% for the Poisson using Galerkin and 21\% for variable-coefficient diffusion using Transolver. For FNO, the benefit grows with the resolution gap, reaching $35\%$ at an 8$\times$ gap and $42\%$ at a 13$\times$ gap, and it also grows with depth, from 2.9 times at 4 layers to 4.7 times at 8 layers. The advantage further amplifies with model scale, increasing from $18\%$ at 307K parameters to $38\%$ at 9.6M parameters, and the same qualitative picture persists under stronger recipes and foundation-style multi-resolution training.

Finally, the empirical study uncovers phenomena beyond a straightforward transfer of classical quadrature ideas. These include a spectral mismatch, namely an $O(h)$ boundary perturbation resulting from endpoint correction in Proposition~\ref{prop:spectral-mismatch} that is empirically harmful for FNO's periodic features; an empirical overcompensation effect, where QuadNorm reduces degradation to 0.22 times the no-normalization reference in the empirical error comparison study; an $\alpha$-reversal, where QuadNorm outperforms BlendQuadNorm on nonspectral architectures but not for FNO; and a model-scale amplification from $18\%$ to $38\%$, which suggests that discretization consistency becomes a bottleneck at scale. A detailed related-work discussion appears in Appendix~\ref{app:related}.

\section{The Quadrature Normalization Family}
\label{sec:method}

\subsection{Problem: Discretization-Dependent Normalization}
\label{sec:problem}
For a single sample, let $u: \Omega \to \R^C$ be a continuous field defined on $\Omega = [0,1]^d$, and let $x$ be its discretization with shape $(C, n_1, \cdots, n_d)$ on an endpoint-inclusive uniform grid with coordinate spacings $h_k = (n_k - 1)^{-1}$. In the isotropic case considered below, we write $h = h_1 = \cdots = h_d$. Standard normalization layers compute statistics as follows:
\begin{equation}
	\hat{\mu}_c^{\text{disc}} = \frac{1}{\prod_{k} n_k} \sum_{\mathbf{i}} x_c(\mathbf{i}), \qquad
	\hat{\sigma}_c^{2,\text{disc}} = \frac{1}{\prod_{k} n_k} \sum_{\mathbf{i}} (x_c(\mathbf{i}) - \hat{\mu}_c^{\text{disc}})^2.
\end{equation}
These are uniformly weighted discrete averages. On endpoint-inclusive uniform grids, they can yield first-order mismatch in $h$ (Proposition~\ref{prop:standard-inconsistent}), and this statistic-level discrepancy can translate into substantial cross-resolution errors in practice; on nonuniform mesh families, equal point weights can additionally induce nonvanishing bias because the point density does not match cell volume (Appendix~\ref{sec:exp-nonuniform}, especially Figure~\ref{fig:nonuniform-mechanism}).

\begin{definition}[Discretization Consistency]
	\label{def:consistency}
	A family of normalization operators $\{\mathcal{N}_h\}$ is discretization-consistent if, for any $u \in C^2(\Omega;\R^C)$ and discretizations $x_h, x_{h'}$ of $u$ at grid spacings $h, h'$:
	\begin{equation}
		\norm{\mathcal{N}_h(x_h) - P_{h' \to h}\mathcal{N}_{h'}(x_{h'})}_{L^2} = O(\max(h^p, h'^p))
	\end{equation}
	for some $p > 0$, where $P_{h' \to h}$ denotes interpolation from grid $h'$ to grid $h$.
\end{definition}

Whenever two discretizations are compared on the $h$-grid, we use the quadrature-weighted discrete comparison norm
\begin{align*}
	\|z_h\|_{L_h^2}^2 := \sum_{\mathbf{i}} \omega_{\mathbf{i}}^{(h)}  \|z_h(\mathbf{r}_{\mathbf{i}})\|_2^2,
\end{align*}
where $\omega_{\mathbf{i}}^{(h)}$ are the spatial quadrature weights on the comparison grid. In the uniform-grid results below, these weights are tensor-product trapezoidal weights, normalized so that $\sum_{\mathbf{i}} \omega_{\mathbf{i}}^{(h)} = |\Omega|$. When the comparison grid is clear, we write $\|\cdot\|_{L^2}$ for this norm. With this convention, an $L^\infty$ bound of order $O(\eta)$ implies the same order in $L^2$. Standard normalization layers do not, in general, satisfy this definition of discretization consistency at second order (Proposition~\ref{prop:standard-inconsistent}; Appendix~\ref{sec:exp-nonuniform}).

\subsection{QuadNorm: Resolution-Consistent Moments via Quadrature}
\label{sec:quadnorm}
Our goal is to design a normalization layer that is robust to discretization, and we achieve this by computing mean and variance estimates that remain consistent across different discretizations. Specifically, we replace discrete averages with quadrature-weighted integrals. The spatial quadrature can be paired with any fixed, resolution-independent choice of channel reduction axes. As a canonical example, the per-channel spatial moments used by an InstanceNorm-style pure quadrature variant are obtained by assigning quadrature weights $w_i > 0$ to each spatial location $\mathbf{r}_i$ and defining:
\begin{equation}
	\mu_c = \frac{\sum_i w_i  x_c(\mathbf{r}_i)}{\sum_i w_i}, \quad
	\sigma_c^2 = \frac{\sum_i w_i  (x_c(\mathbf{r}_i) - \mu_c)^2}{\sum_i w_i}.
\end{equation}
When the weights $w_i$ are uniform, these moments reduce to the usual uniformly weighted per-channel spatial statistics used by standard normalizations such as InstanceNorm.

The key insight is that these weighted sums approximate continuous integrals:
\begin{equation}
	\mu_c \approx \frac{1}{|\Omega|}\int_\Omega u_c(\mathbf{r})  d\mathbf{r}, \qquad
	\sigma_c^2 \approx \frac{1}{|\Omega|}\int_\Omega (u_c(\mathbf{r}) - \bar{u}_c)^2  d\mathbf{r},
	\qquad
	\bar{u}_c := \frac{1}{|\Omega|}\int_\Omega u_c(\mathbf{r})  d\mathbf{r}.
\end{equation}

For $d$-dimensional uniform grids, the weights are tensor products of 1D quadrature weights. Specifically, for the trapezoidal rule on $[0,1]$ with $n$ endpoint-inclusive points:
\begin{equation}
	w_j = \begin{cases}
		\frac{h}{2} & j = 0 \text{ or } j = n-1 \\
		h           & \text{otherwise}
	\end{cases}, \qquad h = \frac{1}{n-1}.
\end{equation}
This is the classical composite trapezoidal rule; its second-order behavior on endpoint-inclusive intervals and its especially favorable behavior on periodic grids are classical \citep{trefethen2014trapezoidal, fornberg2021improving}.

The normalization then proceeds as standard:
\begin{equation}
	y_c(\mathbf{r}_i) = \gamma_c \frac{x_c(\mathbf{r}_i) - \mu_c}{\sqrt{\sigma_c^2 + \varepsilon}} + \beta_c,
\end{equation}
where $\gamma_c$ and $\beta_c$ are learnable per-channel affine parameters. The same quadrature construction can be combined with LayerNorm-style or GroupNorm-style channel reduction by multiplying the spatial weights across the chosen channel reduction axes; this construction is the version used by BlendQuadNorm below. The nonuniform-grid extension, which uses control-volume weights and includes supporting bias analysis, is presented in Appendix~\ref{sec:exp-nonuniform}.

\subsection{BlendQuadNorm: Blending LayerNorm with Quadrature-Weighted Statistics}
\label{sec:blendquadnorm}
QuadNorm might slightly reduce single-resolution accuracy because it changes the effective weighting of boundary points. Here, we introduce BlendQuadNorm, a drop-in replacement for LayerNorm that keeps LayerNorm's reduction axes together with the same spatially constant affine parameterization used in our implementations, while interpolating between standard LayerNorm statistics and their quadrature-weighted counterparts. Let $\mu_{\text{WLN}}$ and $v_{\text{WLN}}$ denote the Weighted LayerNorm (WLN) statistics obtained by applying quadrature weights to the LayerNorm reduction axes within BlendQuadNorm:
\begin{align}
	\mu_{\text{WLN}} & = \frac{\sum_{c,i} w_i  x_c(\mathbf{r}_i)}{C \sum_i w_i}, \qquad
	v_{\text{WLN}} = \frac{\sum_{c,i} w_i  (x_c(\mathbf{r}_i) - \mu_{\text{WLN}})^2}{C \sum_i w_i}, \nonumber \\
	\mu_b            & = \alpha  \mu_{\text{LayerNorm}} + (1 - \alpha)  \mu_{\text{WLN}}, \label{eq:blend-mu} \\
	v_b              & = \alpha  v_{\text{LayerNorm}} + (1 - \alpha)  v_{\text{WLN}}
	+ \alpha(1-\alpha)(\mu_{\text{LayerNorm}} - \mu_{\text{WLN}})^2, \label{eq:blend-var}                     \\
	y                & = \gamma \frac{x - \mu_b}{\sqrt{v_b + \varepsilon}} + \beta. \label{eq:blend-norm}
\end{align}
WLN shares QuadNorm's quadrature-weighted moment construction, differing only in that its mean and variance are computed over all channels and spatial positions, as in LayerNorm. Equivalently, if the normalized uniform and quadrature weights over the reduced axes are denoted by $p_{\text{LayerNorm}}$ and $p_{\text{WLN}}$, then $p_b = \alpha p_{\text{LayerNorm}} + (1-\alpha)p_{\text{WLN}}$, $\mu_b$ is the exact mean, and \eqref{eq:blend-var} is its exact second central moment by the law of total variance.

In experiments, QuadNorm denotes the pure quadrature endpoint matched to the native reduction pattern of the layer being replaced. The per-channel spatial formulas above recover the InstanceNorm-style pure quadrature variant, while the $\alpha = 0$ endpoint of \Eqrefrange{eq:blend-mu}{eq:blend-norm} recovers WLN, the corresponding pure quadrature endpoint when the reduction axes match LayerNorm.

Unless stated otherwise, our BlendQuadNorm uses a fixed $\alpha = 0.3$ across all layers in all experiments. Appendix~\ref{app:details} gives the corresponding LayerNorm-style per-sample forward pass used in our experiments. We discuss architecture-dependent $\alpha$ choices in Appendix Table~\ref{tab:alpha-matrix}.

\section{Theoretical Analysis}
\label{sec:theory}
\paragraph{Scope of the analysis.} This section formalizes the consistency guarantee of quadrature-weighted normalization by proving the second-order cross-resolution agreement of its statistics and outputs, in contrast to the first-order mismatch that uniform averaging can exhibit. This second-order behavior is the key robustness mechanism of the quadrature normalization family: the normalization-induced statistic mismatch decays rapidly under grid refinement, making cross-resolution transfer less sensitive to discretization changes. The results below are stated for a fixed channel-reduction pattern. For pure quadrature variants with per-channel statistics, all statements are interpreted componentwise over channels. For LayerNorm-style BlendQuadNorm, the statistics are scalars because the reduction axes include channels. Because the number of channels or groups is finite and resolution-independent, passing from a scalar statement to finitely many components does not change the order in $h$. Throughout, LayerNorm-style refers to this reduction pattern together with affine parameters that are constant over spatial positions, either scalar or per-channel, as in our implementations. Proofs for all results in this section are given in Appendix~\ref{app:theory-proofs}.

\begin{proposition}[Second-order consistency of quadrature-weighted means]
	\label{thm:consistency}
	Let $f \in C^2(\Omega)$ with $\Omega = [0,1]^d$. Let $\hat{\mu}_h(f)$ and $\hat{\mu}_{h'}(f)$ denote the tensor-product trapezoidal estimates of $|\Omega|^{-1}\int_\Omega f(\mathbf{r})d\mathbf{r}$ on endpoint-inclusive uniform grids with spacings $h$ and $h'$. Then we have
	\begin{equation}
		|\hat{\mu}_h(f) - \hat{\mu}_{h'}(f)| = O(h^2 + h'^2).
	\end{equation}
\end{proposition}

\begin{proposition}[Endpoint-inclusive uniform averaging can be first-order]
	\label{prop:standard-inconsistent}
	Let $f \in C^2([0,1])$ and sample it at $m \ge 3$ endpoint-inclusive points $x_j = j/(m-1)$ with spacing $h = 1/(m-1)$. Define
	\begin{align*}
		\hat{\mu}_h^{\mathrm{disc}} = \frac{1}{m}\sum_{j=0}^{m-1} f(x_j),
		\qquad
		\hat{\mu}_h^{\mathrm{trap}} = \frac{h}{2}(f(0) + f(1)) + h \sum_{j=1}^{m-2} f(x_j).
	\end{align*}
	Then we have
	\begin{equation}
		\hat{\mu}_h^{\mathrm{disc}} - \hat{\mu}_h^{\mathrm{trap}}
		= \frac{m-2}{m(m-1)}
		\left(
		\frac{f(0)+f(1)}{2} - \bar{f}_{\mathrm{int}}
		\right)
		= O(h),
	\end{equation}
	where $\bar{f}_{\mathrm{int}} = (1/(m-2))\sum_{j=1}^{m-2} f(x_j)$. If
	\begin{align*}
		\left|
		\frac{f(0)+f(1)}{2} - \bar{f}_{\mathrm{int}}
		\right| \ge c > 0
	\end{align*}
	uniformly along a refinement family, then $|\hat{\mu}_h^{\mathrm{disc}} - \hat{\mu}_h^{\mathrm{trap}}| = \Theta(h)$. Consequently,
	\begin{align*}
		|\hat{\mu}_h^{\mathrm{disc}} - \hat{\mu}_{h'}^{\mathrm{disc}}| = O(h + h').
	\end{align*}
\end{proposition}

\begin{theorem}[Output consistency of trapezoidal quadrature-weighted normalization]
	\label{thm:output-consistency}
	Let $u \in C^2(\Omega;\R^C)$, and let $x_h, x_{h'}$ be its exact nodal samples on endpoint-inclusive uniform tensor-product grids with spacings $h$ and $h'$. Let $\mathcal{N}_h$ denote a quadrature-weighted normalization layer with a fixed channel-reduction pattern, whose spatial weights are the corresponding tensor-product trapezoidal weights. Assume the affine parameters $\gamma$ and $\beta$ are resolution-independent and constant over the interpolated spatial axes, for example, as scalar or per-channel affine parameters. Assume the interpolant $P_{h' \to h}$ acts channel-wise on the spatial grid, is linear, preserves spatially constant fields, and satisfies
	\begin{align*}
		\|x_h - P_{h' \to h}x_{h'}\|_{L^\infty} = O(h^2 + h'^2)
	\end{align*}
	and assume the corresponding weighted variance scalars or vectors are bounded below componentwise by $v_{\min} > 0$. Then we have
	\begin{equation}
		\norm{\mathcal{N}_h(x_h) - P_{h' \to h}\mathcal{N}_{h'}(x_{h'})}_{L^2} = O(h^2 + h'^2),
	\end{equation}
	where $P_{h' \to h}$ is the interpolation operator used to compare different resolutions.
\end{theorem}

\begin{proposition}[Quadrature collapses to uniform weights on periodic fast Fourier transform (FFT) grids]
	\label{prop:periodic}
	Let $\Omega = [0,1]^d$ be a periodic and sample a field on the corresponding periodic grid
	\begin{align*}
		\left\{\left(\frac{j_1}{n_1}, \ldots, \frac{j_d}{n_d}\right) : 0 \leq j_k \leq n_k - 1\right\},
	\end{align*}
	that is, without duplicating the endpoint. Then the composite trapezoidal weights are uniform. Each spatial point receives a weight $\prod_{k=1}^d n_k^{-1}$. Consequently, for any fixed reduction axes, the quadrature-weighted statistics are equal to the uniformly weighted statistics at every resolution. In particular, the weighted LayerNorm statistics exactly equal the LayerNorm statistics, and BlendQuadNorm collapses to LayerNorm on such grids.
\end{proposition}

Thus, on periodic FFT grids, any difference between LayerNorm and a pure quadrature variant can only come from a different channel-reduction pattern; if the reduction axes match LayerNorm, the pure quadrature statistics coincide exactly with the LayerNorm statistics.

\begin{proposition}[Endpoint correction induces an $O(h)$ boundary perturbation]
	\label{prop:spectral-mismatch}
	Let $u \in C^1([0,1])$ and sample it on an endpoint-inclusive uniform grid $x_j = j/(m-1)$, $j=0,\dots,m-1$, with spacing $h = 1/(m-1)$. Assume $m \ge 3$. Let $\hat{\mu}_h^{\text{trap}}$ and $\hat{\mu}_h^{\text{unif}}$ denote the trapezoidal and uniform averages on this grid. Then
	\begin{equation}
		\hat{\mu}_h^{\text{trap}} - \hat{\mu}_h^{\text{unif}}
		= \frac{m-2}{m(m-1)}
		\left(
		\bar{u}_{\mathrm{int}} - \frac{u(0) + u(1)}{2}
		\right)
		= h\left(
		\bar{u}_{\mathrm{int}} - \frac{u(0) + u(1)}{2}
		\right) + O(h^2),
	\end{equation}
	where $\bar{u}_{\mathrm{int}} = (1/(m-2))\sum_{j=1}^{m-2} u(x_j)$. Consequently,
	\begin{align*}
		\left|\hat{\mu}_h^{\text{trap}} - \hat{\mu}_h^{\text{unif}}\right| = O(h).
	\end{align*}
	Moreover,
	\begin{align*}
		\left|\hat{\mu}_h^{\text{trap}} - \hat{\mu}_h^{\text{unif}}\right| = \Theta(h)
	\end{align*}
	whenever
	\begin{align*}
		\left|\bar{u}_{\mathrm{int}} - (u(0) + u(1))/2\right| = \Theta(1).
	\end{align*}
\end{proposition}

At the network level, Appendix~\ref{app:proof-transfer} states a conditional transfer-error propagation bound showing that normalization contributes an $O(L \cdot \max(h^p,h'^p))$ term to the cross-resolution discrepancy. This bound yields the gap- and depth-scaling hypotheses tested below.

\section{Experiments}
\label{sec:experiments}
We evaluate QuadNorm and BlendQuadNorm in a large suite of experiments across various settings spanning multiple PDE families and architecture variants. These experiments include FNO, Galerkin Transformer, and Transolver, with sizes up to about 9.6M parameters; each setting is tested using multiple seeds. The baseline normalization families include LayerNorm, InstanceNorm, GroupNorm, and RMSNorm, along with no normalization. Appendix~\ref{app:supporting-experiments} reports normalization baselines, dataset descriptions, training protocols, the empirical error-comparison ablation, the Transolver Helmholtz and multi-resolution studies, the model-scaling analysis, and additional supporting ablations. It opens with a short roadmap of the full experimental matrix, and Appendix Table~\ref{tab:key-results-summary} collects a compact headline summary. Additional appendix studies include architecture-dependent $\alpha$ sweeps, bootstrap confidence summaries, moderate-gap nonspectral transfer, Darcy architecture transfer, Transolver $\alpha$-sensitivity analyses, Poisson--Dirichlet follow-up studies, extreme-gap studies with a 4$\times$ gap on non-Darcy PDEs, and further ablations.

\subsection{Extreme Cross-Resolution Transfer}
\label{sec:exp-extreme-crossres}

\begin{table}[t]
	\centering
	\caption{Extreme cross-resolution transfer on Darcy, reporting the relative $L^2$ error in percent as the mean $\pm$ 95\% confidence interval (CI) over 10 seeds. Models are trained at $32^2$ and evaluated at $32^2$, $64^2$, $128^2$, and $256^2$, which gives a refinement ratio of up to 8. Bold indicates the best method or a method statistically tied with the best.}
	\label{tab:extreme-crossres}
	\begin{tabular}{lcccc}
		\toprule
		\tableheadrulebelow
		\rowcolor{tableheadgray}
		Method        & $32^2$ (native)          & $32^2\to64^2$            & $32^2\to128^2$           & $32^2\to256^2$           \\
		\tableheadruleabove
		\midrule
		None          & 3.49 $\pm$ 0.08          & 7.90 $\pm$ 0.18          & 11.15 $\pm$ 0.26         & 12.76 $\pm$ 0.29         \\
		LayerNorm     & \textbf{3.26 $\pm$ 0.05} & 6.87 $\pm$ 0.20          & 9.62 $\pm$ 0.29          & 10.99 $\pm$ 0.34         \\
		InstanceNorm  & 4.27 $\pm$ 0.06          & 7.77 $\pm$ 0.09          & 10.49 $\pm$ 0.11         & 11.87 $\pm$ 0.11         \\
		GroupNorm     & 3.36 $\pm$ 0.05          & 7.02 $\pm$ 0.15          & 9.82 $\pm$ 0.23          & 11.22 $\pm$ 0.27         \\
		RMSNorm       & 3.59 $\pm$ 0.08          & 7.31 $\pm$ 0.13          & 10.13 $\pm$ 0.20         & 11.55 $\pm$ 0.23         \\
		QuadNorm      & 4.35 $\pm$ 0.05          & \textbf{5.61 $\pm$ 0.14} & \textbf{6.60 $\pm$ 0.20} & \textbf{7.10 $\pm$ 0.23} \\
		BlendQuadNorm & \textbf{3.27 $\pm$ 0.05} & 6.63 $\pm$ 0.18          & 9.24 $\pm$ 0.26          & 10.54 $\pm$ 0.30         \\
		\bottomrule
	\end{tabular}
\end{table}

We train the FNO at $32^2$ and evaluate it up to $256^2$, which is an 8-fold refinement (Table~\ref{tab:extreme-crossres}). At this 8$\times$ gap, QuadNorm achieves $7.10\%$, a 35\% relative reduction in error from LayerNorm's $10.99\%$; the exact bootstrap confidence interval for this comparison is reported in Appendix Table~\ref{tab:bootstrap-cis}. This extreme-gap result confirms that the benefit grows with the gap, consistent with the conditional transfer bound presented in Appendix~\ref{app:proof-transfer} and the gap-scaling trend summarized in Figure~\ref{fig:scaling}(a). Here, QuadNorm sacrifices native accuracy for superior transfer, while BlendQuadNorm recovers native accuracy with a weaker transfer improvement, again as shown in Table~\ref{tab:extreme-crossres}. The ratio of $\Delta_{\text{LayerNorm}}$ to $\Delta_{\text{QuadNorm}}$ remains stable at roughly 2.8 to 2.9 times across all gaps from 2$\times$ to 8$\times$ in Table~\ref{tab:extreme-crossres}, consistent with the expected difference between first-order and second-order scaling of the statistic mismatch. Figure~\ref{fig:scaling}(a) summarizes the same multiplicative separation, although the table measures end-to-end prediction error rather than the statistic order directly. A complementary deep 8-layer FNO transfer study appears in Table~\ref{tab:deep-crossres}.

\subsection{Deep Cross-Resolution Transfer}
\label{sec:exp-deep-crossres}

\begin{table}[t]
	\centering
	\caption{Deep model cross-resolution transfer on Darcy, reporting the relative $L^2$ error in percent as the mean $\pm$ 95\% CI over 10 seeds. The model is an 8-layer FNO with a width of 48 and 12 modes, trained at $64^2$ and evaluated at $64^2$, $128^2$, and $256^2$. Under the conditional transfer bound stated in Appendix~\ref{app:proof-transfer}, the normalization-induced transfer discrepancy can accumulate at most linearly with depth $L$.}
	\label{tab:deep-crossres}
	\begin{tabular}{lccc}
		\toprule
		\tableheadrulebelow
		\rowcolor{tableheadgray}
		Method        & $64^2$ (native)          & $64^2\to128^2$           & $64^2\to256^2$           \\
		\tableheadruleabove
		\midrule
		None          & \textbf{2.30 $\pm$ 0.03} & 4.19 $\pm$ 0.08          & 5.60 $\pm$ 0.11          \\
		LayerNorm     & \textbf{2.26 $\pm$ 0.03} & \textbf{3.89 $\pm$ 0.07} & 5.14 $\pm$ 0.09          \\
		InstanceNorm  & 3.54 $\pm$ 0.06          & 4.78 $\pm$ 0.04          & 5.89 $\pm$ 0.04          \\
		GroupNorm     & 2.46 $\pm$ 0.05          & 4.03 $\pm$ 0.04          & 5.28 $\pm$ 0.06          \\
		RMSNorm       & 2.97 $\pm$ 0.05          & 4.40 $\pm$ 0.03          & 5.60 $\pm$ 0.04          \\
		QuadNorm      & 3.56 $\pm$ 0.07          & \textbf{3.90 $\pm$ 0.07} & \textbf{4.17 $\pm$ 0.07} \\
		BlendQuadNorm & \textbf{2.26 $\pm$ 0.04} & \textbf{3.86 $\pm$ 0.09} & 5.09 $\pm$ 0.13          \\
		\bottomrule
	\end{tabular}
\end{table}

The conditional transfer bound in Appendix~\ref{app:proof-transfer} gives a conditional linear-in-depth upper bound for the normalization-induced transfer discrepancy. Figure~\ref{fig:scaling}(b) summarizes the depth-scaling comparison. The degradation ratio between LayerNorm and QuadNorm grows from about 2.9 times in the 4-layer extreme cross-resolution setting to 4.7 times in the 8-layer deep cross-resolution setting. In Table~\ref{tab:deep-crossres}, QuadNorm's degradation at $64 \to 256$ is only $0.61\%$, compared with $2.88\%$ for LayerNorm, consistent with the amplified depth effect.

\begin{figure}[t]
	\centering
	\includegraphics[width=\textwidth]{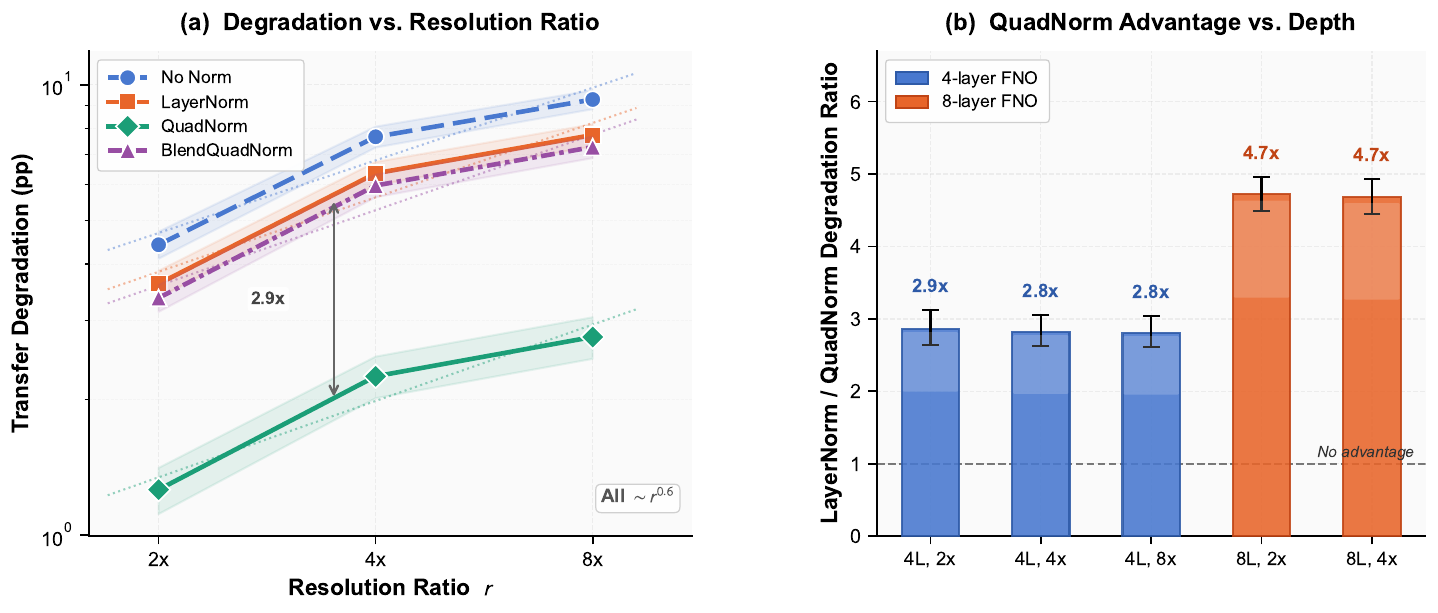}
	\caption{Scaling behavior of QuadNorm's advantage. The first plot shows transfer degradation, measured in percentage points above native-resolution error, plotted against the resolution ratio on log-log axes. All methods scale approximately as $\sim r^{0.6}$, but QuadNorm has a 2.9-fold lower intercept than LayerNorm, confirming a persistent multiplicative advantage consistent with Theorem~\ref{thm:transfer-bound}. The second plot shows the degradation ratio between LayerNorm and QuadNorm against depth. This advantage amplifies from about 2.9 times at 4 layers to about 4.7 times at 8 layers, consistent with the conditional linear-in-$L$ dependence of the normalization-induced term in Theorem~\ref{thm:transfer-bound}.}
	\label{fig:scaling}
\end{figure}

\paragraph{Architecture-dependent $\alpha$ choice.} The $\alpha$-sweep study on nonspectral, nonperiodic settings, reported in Appendix Table~\ref{tab:alpha-matrix}, favors lower $\alpha$, often the pure quadrature endpoint, while BlendQuadNorm remains the conservative default for the Darcy setting using FNO.

\paragraph{Extended baselines and supporting experiments.} Spectral-domain baselines such as SpectralNorm, resolution-invariant spectral band normalization (RI-BandNorm), and QuadBandNorm fail catastrophically under resolution transfer, exceeding $97\%$ error at $256^2$ in Appendix Table~\ref{tab:extended-baselines}; additional supporting experiments appear in Appendix~\ref{app:supporting-experiments}. Additional experiments probe model-scale behavior, Poisson--Dirichlet, and consistent-kernel scaling. In these supporting studies, fixed $\alpha = 0.3$ outperforms learned per-layer $\alpha$; the normalization hierarchy persists at scale.

\subsection{Nonperiodic Multi-Architecture Transfer}
\label{sec:exp-multiarch-nonperiodic}

\begin{table}[t]
	\centering
	\caption{Galerkin Transformer on three nonperiodic PDEs with training at $32^2$: cavity flow, elasticity, and Poisson--Dirichlet. The normalization advantage consistently transfers across different architectures on nonperiodic domains. Bold indicates the best method within each PDE--resolution column.}
	\label{tab:multiarch-nonperiodic}
	\begin{tabular}{llccc}
		\toprule
		\tableheadrulebelow
		\rowcolor{tableheadgray}
		PDE                         & Method        & $32^2$ (native)          & $64^2$ (2$\times$)          & $128^2$ (4$\times$)         \\
		\tableheadruleabove
		\midrule
		\multirow{4}{*}{\shortstack[c]{Cavity                                                                                          \\flow}} & None & 1.63 $\pm$ 0.36 & 2.89 $\pm$ 0.31 & 3.38 $\pm$ 0.39 \\
		                            & LayerNorm     & 1.85 $\pm$ 0.22          & 4.73 $\pm$ 1.15           & 5.41 $\pm$ 1.17           \\
		                            & QuadNorm      & 1.41 $\pm$ 1.35          & 4.44 $\pm$ 0.76           & 6.22 $\pm$ 0.78           \\
		                            & BlendQuadNorm & \textbf{1.30 $\pm$ 0.23} & \textbf{2.32 $\pm$ 0.35}  & \textbf{2.89 $\pm$ 0.54}  \\
		\midrule
		\multirow{4}{*}{Elasticity} & None          & 11.40 $\pm$ 1.03         & 11.99 $\pm$ 0.95          & 12.56 $\pm$ 0.87          \\
		                            & LayerNorm     & 12.43 $\pm$ 0.50         & 12.81 $\pm$ 0.49          & 13.12 $\pm$ 0.52          \\
		                            & QuadNorm      & 10.03 $\pm$ 1.06         & 11.25 $\pm$ 1.36          & 12.54 $\pm$ 1.68          \\
		                            & BlendQuadNorm & \textbf{9.72 $\pm$ 0.85} & \textbf{10.00 $\pm$ 0.95} & \textbf{10.21 $\pm$ 1.04} \\
		\midrule
		\multirow{4}{*}{\shortstack[c]{Poisson--                                                                                       \\Dirichlet}} & None & 10.54 $\pm$ 0.41 & 11.05 $\pm$ 0.44 & 11.56 $\pm$ 0.57 \\
		                            & LayerNorm     & 11.41 $\pm$ 0.87         & 11.80 $\pm$ 0.84          & 12.14 $\pm$ 0.86          \\
		                            & QuadNorm      & \textbf{8.77 $\pm$ 0.28} & 9.99 $\pm$ 0.49           & 11.21 $\pm$ 0.89          \\
		                            & BlendQuadNorm & 8.95 $\pm$ 0.52          & \textbf{9.24 $\pm$ 0.35}  & \textbf{9.47 $\pm$ 0.23}  \\
		\bottomrule
	\end{tabular}
\end{table}

Table~\ref{tab:multiarch-nonperiodic} evaluates the Galerkin Transformer on cavity flow, elasticity, and Poisson--Dirichlet. BlendQuadNorm achieves the lowest degradation across all three nonperiodic PDEs and attains the best native accuracy on cavity flow and elasticity in Table~\ref{tab:multiarch-nonperiodic}. On Poisson--Dirichlet, QuadNorm is slightly better natively at $8.77\%$, compared with BlendQuadNorm at $8.95\%$, also reported in Table~\ref{tab:multiarch-nonperiodic}. In this nonspectral setting, the gains from quadrature-weighted statistics persist, supporting the interpretation that the FNO reversal arises from its Fourier basis and endpoint treatment. The Darcy benchmark results presented below complement this synthetic nonperiodic evidence with official Darcy studies for FNO, Galerkin, and Transolver under stronger training recipes. Appendix~\ref{sec:exp-scaling} provides additional ablations.

\subsection{Darcy Benchmark Studies}
\label{sec:exp-real-darcy-studies}
\paragraph{Comprehensive Darcy Benchmark Evaluation.}\label{sec:exp-real-comprehensive} We evaluate FNO, Galerkin, and Transolver on the official FNO Darcy benchmark \citep{li2021fno}. The full protocol for this stronger benchmark study is given in Appendix~\ref{sec:exp-real-summary}. Table~\ref{tab:real-darcy} reports the resulting native and cross-resolution performance. On Darcy, which is nonperiodic, QuadNorm reduces cross-resolution degradation across all architectures shown in Table~\ref{tab:real-darcy}. At a 4$\times$ gap, the improvement is $32\%$ for FNO, with $3.40\%$ compared with $4.99\%$, and $14\%$ for Galerkin, with $5.54\%$ compared with $6.43\%$. Transolver is nearly resolution-invariant under QuadNorm, with performance changing from $4.24\%$ to $4.46\%$, corresponding to only $5\%$ degradation compared with $38\%$ for LayerNorm. BlendQuadNorm stays near LayerNorm at native resolution for FNO and Galerkin, while Transolver also improves natively. These main-table results on a standard benchmark with a strong training recipe confirm that the benefit is a genuine architectural property, not an artifact of weak baselines.

Appendix Table~\ref{tab:real-benchmark-best-summary} summarizes the best improvements of either QuadNorm or BlendQuadNorm over LayerNorm across the same Darcy benchmark configurations, and Appendix~\ref{sec:exp-real-summary} provides detailed per-resolution tables for the Darcy benchmark cross-resolution transfer study, the Transolver study, and the setting with a 13$\times$ gap (Tables~\ref{tab:real-darcy-crossres}, \ref{tab:transolver-darcy}, and \ref{tab:real-darcy-extreme}). These appendix results confirm the same pattern at a finer granularity: QuadNorm's advantage grows with the resolution gap on the official benchmark, reaching a $42\%$ gain for FNO at $32^2 \to 421^2$ in the extreme-gap setting, and keeps Transolver nearly resolution-invariant.

\begin{table}[t]
	\centering
	\caption{Relative $L^2$ error in percent on the official FNO Darcy benchmark under the comprehensive Darcy benchmark protocol summarized in Appendix~\ref{sec:exp-real-summary}.}
	\label{tab:real-darcy}
	\small
	\begin{tabular}{lccc}
		\toprule
		\tableheadrulebelow
		\rowcolor{tableheadgray}
		Method        & $64^2 \to 64^2$          & $64^2 \to 128^2$         & $64^2 \to 256^2$         \\
		\tableheadruleabove
		\midrule
		\multicolumn{4}{l}{FNO}                                                                        \\
		\midrule
		LayerNorm     & \textbf{2.04 $\pm$ 0.04} & 3.72 $\pm$ 0.08          & 4.99 $\pm$ 0.11          \\
		QuadNorm      & 2.66 $\pm$ 0.04          & \textbf{3.08 $\pm$ 0.08} & \textbf{3.40 $\pm$ 0.10} \\
		BlendQuadNorm & 2.04 $\pm$ 0.04          & 3.67 $\pm$ 0.11          & 4.91 $\pm$ 0.17          \\
		No Norm       & 2.15 $\pm$ 0.06          & 3.78 $\pm$ 0.16          & 5.02 $\pm$ 0.22          \\
		\midrule
		\multicolumn{4}{l}{Galerkin}                                                                   \\
		\midrule
		LayerNorm     & \textbf{4.93 $\pm$ 0.08} & 5.67 $\pm$ 0.08          & 6.43 $\pm$ 0.08          \\
		QuadNorm      & 5.13 $\pm$ 0.11          & \textbf{5.34 $\pm$ 0.13} & \textbf{5.54 $\pm$ 0.15} \\
		BlendQuadNorm & 4.98 $\pm$ 0.03          & 5.71 $\pm$ 0.02          & 6.46 $\pm$ 0.05          \\
		No Norm       & 6.00 $\pm$ 0.25          & 6.70 $\pm$ 0.23          & 7.46 $\pm$ 0.21          \\
		\midrule
		\multicolumn{4}{l}{Transolver}                                                                 \\
		\midrule
		LayerNorm     & 4.62 $\pm$ 0.10          & 5.51 $\pm$ 0.15          & 6.38 $\pm$ 0.17          \\
		QuadNorm      & \textbf{4.24 $\pm$ 0.09} & \textbf{4.36 $\pm$ 0.08} & \textbf{4.46 $\pm$ 0.07} \\
		BlendQuadNorm & 4.28 $\pm$ 0.25          & 5.14 $\pm$ 0.27          & 5.97 $\pm$ 0.27          \\
		No Norm       & 5.00 $\pm$ 0.06          & 5.78 $\pm$ 0.15          & 6.57 $\pm$ 0.17          \\
		\bottomrule
	\end{tabular}
\end{table}

\section{Conclusion}
\label{sec:conclusion}
We identify normalization layers as an overlooked source of discretization dependence in neural operators, proposing a quadrature normalization family consisting of QuadNorm and BlendQuadNorm. The pure quadrature members of this family have moments that are provably $O(h^2)$-consistent on endpoint-inclusive uniform grids. This provides a theoretical basis for the faster decay of normalization-induced mismatch compared with standard normalization statistics, whose endpoint-inclusive uniform-grid error is only $O(h)$. The theory predicts that normalization-induced transfer mismatch grows with resolution gap and depth. Our experiments confirm this trend and show that QuadNorm sharply reduces cross-resolution degradation, especially for nonperiodic PDEs, nonspectral architectures, larger gaps, and larger models. In practice, BlendQuadNorm is the safest drop-in default with no extra hyperparameters when deployment conditions are uncertain, while lower-$\alpha$ settings or QuadNorm are better choices when discretization consistency is the dominant bottleneck in nonspectral, nonperiodic regimes.

\bibliographystyle{plainnat}
\bibliography{references}

\clearpage
\appendix
\startcontents[appendix]
\phantomsection
\pdfbookmark[1]{Appendix Table of Contents}{app:toc}
\section*{Appendix Table of Contents}
\printcontents[appendix]{l}{1}{\setcounter{tocdepth}{2}}

\clearpage
\section{Notation}
\label{app:notation}

\begin{table*}[t!]
	\centering
	\caption{List of notations used throughout the paper.}
	\label{tab:notation}
	\footnotesize
	\resizebox{\textwidth}{!}{%
		\begin{tabular}{ll}
			\toprule
			\tableheadrulebelow
			\rowcolor{tableheadgray}
			Symbol                                                                                  & Meaning                                                                                  \\
			\tableheadruleabove
			\midrule
			$\Omega = [0,1]^d, d, |\Omega|$                                                         & Domain, dimension, and measure.                                                          \\
			$u : \Omega \to \R^C$                                                                   & Continuous $C$-channel field.                                                            \\
			$x, x_h, x_{h'}$                                                                        & Grid samples of $u$ at spacings $h$ and $h'$; $x$ has shape $(B, C, H, W)$.              \\
			$B, C, H, W, n_k$                                                                       & Batch size, channels, height, width, and grid size along axis $k$.                       \\
			$h_k, h$                                                                                & Grid spacing along axis $k$ and the common isotropic spacing.                            \\
			$\mathbf{i}, \mathbf{r}_{\mathbf{i}}$                                                   & Spatial multi-index and grid point.                                                      \\
				$w_i, w_{ij}, \omega_{\mathbf{i}}^{(h)}$                                                & Control-volume weights; $\omega_{\mathbf{i}}^{(h)}$ is the $L_h^2$ comparison weight.    \\
			$\hat{\mu}^{\mathrm{disc}}, \hat{\sigma}^{2,\mathrm{disc}}$                             & Uniform discrete mean and variance from standard normalization.                          \\
			$\hat{\mu}_h^{\mathrm{disc}}, \hat{\mu}_h^{\mathrm{trap}}, \hat{\mu}_h^{\mathrm{unif}}$ & Uniform, trapezoidal, and endpoint-inclusive uniform averages.                           \\
			$\mu_c, \sigma_c^2, \bar{u}_c$                                                          & Quadrature-weighted channel-$c$ mean, variance and continuous spatial mean.              \\
			$\mu_{\text{LayerNorm}}, v_{\text{LayerNorm}}$                                          & Standard LayerNorm mean and variance over the reduction axes.                            \\
			$\mu_{\text{WLN}}, v_{\text{WLN}}, \mu_b, v_b$                                          & Quadrature-weighted LayerNorm stats and blended stats in BlendQuadNorm.                  \\
			$\alpha, \gamma, \beta, \varepsilon$                                                    & Blend coefficient, affine scale, affine shift, and stabilization constant.               \\
			$\mathcal{N}_h, P_{h' \to h}$                                                           & Normalization layer at resolution $h$ and interpolation from $h'$ to $h$.                \\
			$\|\cdot\|_{L_h^2}, \|\cdot\|_{L^2}$                                                    & Discrete quadrature norm on the $h$-grid and its shorthand.                              \\
			$p, \delta_2, \delta_p$                                                                 & Consistency order; $\delta_2 = \max(h^2,h'^2)$ and $\delta_p = \max(h^p,h'^p)$.          \\
			$\hat{\mu}_h(f), \hat{\mu}_{h'}(f)$                                                     & Trapezoidal estimates of the spatial average of $f$ on the two grids.                    \\
			$\mathcal{G}_{\theta,h}, \theta$                                                        & Neural operator with parameters $\theta$ evaluated at spacing $h$.                       \\
				$f_{0,h}, f_{\ell,h}, f_{L+1,h}$                                                        & Lift map, processor-block map $\ell$, and projection map in the transfer decomposition.  \\
			$T_{\ell,h}, \sigma, L, \ell$                                                           & Resolution-invariant part of block $\ell$, activation, total layers, and layer index.    \\
			$L_{T,\ell}, L_{N,\ell}, L_\sigma$                                                      & Lipschitz or stability constants for $T_{\ell,h}$, $\mathcal{N}_{\ell,h}$, and $\sigma$. \\
			$C_{T,\ell}$, $C_{N,\ell}$, $C_{\sigma,\ell}$, $C_{\mathrm{net}}$                       & Resolution-mismatch constants in layerwise and network transfer bounds.                  \\
			$M_\ell, A$                                                                             & Block amplification factor and a uniform bound on cumulative products.                   \\
			$\Delta_{\text{method}}$                                                                & Cross-resolution degradation relative to native-resolution error.                        \\
			\bottomrule
		\end{tabular}
	}
\end{table*}

\paragraph{Quadrature in This Paper.} In this paper, quadrature means approximating a spatial integral from discrete field samples by assigning each sample a deterministic numerical-integration weight that reflects the physical measure it represents. On endpoint-inclusive uniform grids, these are trapezoidal weights; on nonuniform meshes, they are control-volume weights. Thus, the term ``quadrature-weighted'' normalization means that means and variances are computed with these spatial weights rather than with uniform point averaging.

\section{Theoretical Details and Proofs}
\label{app:theory-proofs}
\label{app:proof-transfer}

\begin{proof}[Proof of Proposition~\ref{thm:consistency}]
	For the 1D composite trapezoidal rule on an endpoint-inclusive uniform grid and $f \in C^2([0,1])$, there is a constant $C_f$ such that
	\begin{align*}
		\left|\hat{\mu}_h(f) - \int_0^1 f(t)dt\right| \le C_f h^2.
	\end{align*}
	Applying the rule dimension-by-dimension on the tensor-product grid gives
	\begin{align*}
		\left|\hat{\mu}_h(f) - |\Omega|^{-1}\int_\Omega f(\mathbf{r})d\mathbf{r}\right| = O(h^2),
	\end{align*}
	and the same estimate holds with $h'$ in place of $h$. The claim follows from the triangle inequality.
\end{proof}

\begin{proof}[Proof of Proposition~\ref{prop:standard-inconsistent}]
	Write $u_j = f(x_j)$ and $\bar{u}_{\mathrm{int}} = (1/(m-2))\sum_{j=1}^{m-2} u_j$. Then
	\begin{align*}
		\hat{\mu}_h^{\mathrm{disc}} = \frac{u_0 + u_{m-1} + (m-2)\bar{u}_{\mathrm{int}}}{m},
		\qquad
		\hat{\mu}_h^{\mathrm{trap}} = \frac{u_0 + u_{m-1}}{2(m-1)} + \frac{m-2}{m-1}\bar{u}_{\mathrm{int}}.
	\end{align*}
	Subtracting yields the displayed identity. Since the composite trapezoidal rule satisfies $\hat{\mu}_h^{\mathrm{trap}} = \int_0^1 f(t)dt + O(h^2)$ for $f \in C^2([0,1])$, we obtain $|\hat{\mu}_h^{\mathrm{disc}} - \int_0^1 f(t)dt| = O(h)$ and therefore $|\hat{\mu}_h^{\mathrm{disc}} - \hat{\mu}_{h'}^{\mathrm{disc}}| = O(h + h')$.
\end{proof}

\begin{proof}[Proof of Theorem~\ref{thm:output-consistency}]
	Let $(\mu_h,v_h)$ and $(\mu_{h'},v_{h'})$ denote the scalar or vector weighted statistics computed on the two grids. Write $\mathcal{N}(x;\mu,v) = \gamma \odot ((x-\mu)/\sqrt{v+\varepsilon}) + \beta$. Because $\mu_{h'}$ and $v_{h'}$ are spatial constants, $\gamma$ and $\beta$ are independent of spatial position, and $P_{h' \to h}$ acts channel-wise, is linear, and preserves spatially constant fields,
	\begin{align*}
		P_{h' \to h}\mathcal{N}_{h'}(x_{h'})
		=
		\mathcal{N}\left(P_{h' \to h}x_{h'};\mu_{h'},v_{h'}\right).
	\end{align*}
	For any two triples $(x_1,\mu_1,v_1)$ and $(x_2,\mu_2,v_2)$ with $v_1, v_2 \geq v_{\min} > 0$ componentwise, $\|x_1\|_\infty,\|x_2\|_\infty \leq M$, and $\|\mu_1\|_\infty,\|\mu_2\|_\infty \leq M_\mu$, adding and subtracting $\gamma \odot (x_2-\mu_2)/\sqrt{v_1+\varepsilon}$ gives
	\begin{equation}
		\begin{aligned}
			\|\mathcal{N}(x_1; \mu_1, v_1) - \mathcal{N}(x_2; \mu_2, v_2)\|_\infty
			 & \leq
			\frac{\|\gamma\|_\infty}{\sqrt{v_{\min} + \varepsilon}}
			\left(\|x_1 - x_2\|_\infty + \|\mu_1 - \mu_2\|_\infty\right)
			\\
			 & \quad
			+ \frac{\|\gamma\|_\infty(M + M_\mu)}{2(v_{\min} + \varepsilon)^{3/2}} \|v_1 - v_2\|_\infty.
		\end{aligned}
	\end{equation}
	Apply this estimate pointwise with $x_1 = x_h$, $x_2 = P_{h' \to h}x_{h'}$, $(\mu_1, v_1) = (\mu_h, v_h)$, and $(\mu_2, v_2) = (\mu_{h'}, v_{h'})$. The interpolation assumption gives $\|x_1 - x_2\|_{L^\infty} = O(h^2 + h'^2)$, and Proposition~\ref{thm:consistency} gives $\|\mu_h - \mu_{h'}\|_\infty = O(h^2 + h'^2)$. Because $x_h$ and $x_{h'}$ are exact nodal samples of $u$, the corresponding weighted means satisfy
	\begin{align*}
		\mu_h = \hat{E}_h[u],
		\qquad
		\mu_{h'} = \hat{E}_{h'}[u].
	\end{align*}
	For the variances, write componentwise
	\begin{align*}
		v_h = \hat{E}_h[u^2] - \hat{E}_h[u]^2,
		\qquad
		v_{h'} = \hat{E}_{h'}[u^2] - \hat{E}_{h'}[u]^2,
	\end{align*}
	where $\hat{E}_h$ and $\hat{E}_{h'}$ denote the corresponding quadrature-weighted averages on the two grids. Since $u^2 \in C^2(\Omega)$, Proposition~\ref{thm:consistency} gives $\|\hat{E}_h[u^2] - \hat{E}_{h'}[u^2]\|_\infty = O(h^2 + h'^2)$, and boundedness of $u$ on the compact domain $\Omega$ implies $\|\hat{E}_h[u]^2 - \hat{E}_{h'}[u]^2\|_\infty = O(h^2 + h'^2)$. Hence, $\|v_h - v_{h'}\|_\infty = O(h^2 + h'^2)$: the pointwise discrepancy is $O(h^2 + h'^2)$. The same order in the comparison norm follows from the quadrature-weighted definition of $\|\cdot\|_{L^2}$ and the identity $\sum_{\mathbf{i}}\omega_{\mathbf{i}}^{(h)} = |\Omega|$.
\end{proof}

\begin{proof}[Proof of Proposition~\ref{prop:periodic}]
	On a periodic FFT grid, each cell has the same spatial volume $\prod_k n_k^{-1}$ and there are no distinguished boundary points. The composite trapezoidal rule therefore assigns the same weight to every node, which is exactly the uniform average.
\end{proof}

\begin{proof}[Proof of Proposition~\ref{prop:spectral-mismatch}]
	Using the same notation as in the proof of Proposition~\ref{prop:standard-inconsistent},
	\begin{align*}
		\hat{\mu}_h^{\text{trap}} = \frac{u_0 + u_{m-1}}{2(m-1)} + \frac{m-2}{m-1}\bar{u}_{\mathrm{int}},
		\qquad
		\hat{\mu}_h^{\text{unif}} = \frac{u_0 + u_{m-1} + (m-2)\bar{u}_{\mathrm{int}}}{m}.
	\end{align*}
	Subtracting gives the exact formula. Since $(m-2)/(m(m-1)) = h + O(h^2)$, the perturbation is $O(h)$ whenever the bracketed term is $O(1)$, and under the proposition's hypothesis that the bracketed term is $\Theta(1)$, this sharpens to $\Theta(h)$.
\end{proof}

\subsection{Conditional Transfer Error Propagation Bound}
\begin{theorem}[Conditional Cross-Resolution Transfer Error Propagation Bound]
	\label{thm:transfer-bound}
	Let $\mathcal{G}_{\theta,h} = f_{L+1,h} \circ f_{L,h} \circ \cdots \circ f_{1,h} \circ f_{0,h}$ be a neural operator evaluated on grid spacing $h$, with $L$ normalization layers. Assume each processor block has the form $f_{\ell,h} = \sigma \circ \mathcal{N}_{\ell,h} \circ T_{\ell,h}$ for $1 \le \ell \le L$, and write
	\begin{align*}
		\delta_2 = \max(h^2,h'^2),
		\qquad
		\delta_p = \max(h^p,h'^p).
	\end{align*}
	Assume the following resolution-compatibility conditions. The lifting map $f_{0,h}$ and projection map $f_{L+1,h}$ admit constants $L_0, C_0, L_{\mathrm{out}}, C_{\mathrm{out}}$ independent of $L$ such that
	\begin{align*}
		\|f_{0,h}(x_h) - P_{h' \to h}f_{0,h'}(x_{h'})\|_{L^2}
		\le
		L_0 \|x_h - P_{h' \to h}x_{h'}\|_{L^2} + C_0 \delta_2,
	\end{align*}
	and such that for any corresponding cross-resolution feature maps $z_h, z_{h'}$,
	\begin{align*}
		\|f_{L+1,h}(z_h) - P_{h' \to h}f_{L+1,h'}(z_{h'})\|_{L^2}
		\le
		L_{\mathrm{out}} \|z_h - P_{h' \to h}z_{h'}\|_{L^2} + C_{\mathrm{out}} \delta_2.
	\end{align*}
	The map $T_{\ell,h}$ denotes the resolution-invariant part of the block, namely the residual connection, kernel, and pointwise linear map, and there exist constants $L_{T,\ell}, C_{T,\ell}$ such that
	\begin{align*}
		\|T_{\ell,h}(z_h) - P_{h' \to h}T_{\ell,h'}(z_{h'})\|_{L^2}
		\le
		L_{T,\ell}\|z_h - P_{h' \to h}z_{h'}\|_{L^2} + C_{T,\ell}\delta_2.
	\end{align*}
	For every pair of corresponding pre-normalization feature maps $z_h, z_{h'}$ produced by the two forward passes being compared, $\mathcal{N}_{\ell,h}$ is resolution-compatible with constants $L_{N,\ell}, C_{N,\ell}$ in the sense that
	\begin{align*}
		\|\mathcal{N}_{\ell,h}(z_h) - P_{h' \to h}\mathcal{N}_{\ell,h'}(z_{h'})\|_{L^2}
		\le
		L_{N,\ell}\|z_h - P_{h' \to h}z_{h'}\|_{L^2} + C_{N,\ell}\delta_p.
	\end{align*}
	Finally, $\sigma$ is $L_\sigma$-Lipschitz, and there exists $C_{\sigma,\ell}$ such that for the corresponding pre-activations $z_h, z_{h'}$ at resolutions $h, h'$,
	\begin{align*}
		\|\sigma(z_h) - P_{h' \to h}\sigma(z_{h'})\|_{L^2}
		\le
		L_\sigma \|z_h - P_{h' \to h}z_{h'}\|_{L^2}
		+ C_{\sigma,\ell}\delta_2.
	\end{align*}
	Define the per-block amplification factors
	\begin{align*}
		M_\ell = L_\sigma L_{N,\ell} L_{T,\ell},
	\end{align*}
	and assume the cumulative products are uniformly bounded in the sense that
	\begin{align*}
		\max_{0 \le k \le L} \prod_{j=k+1}^{L} M_j \le A
	\end{align*}
	for a constant $A$ independent of $L$. The empty product is defined to equal $1$. Assume additionally that the layerwise quantities admit depth-independent uniform bounds
	\begin{align*}
		\begin{aligned}
			\sup_{1 \le \ell \le L} L_{T,\ell} & \le \overline{L}_T, \qquad
			\sup_{1 \le \ell \le L} L_{N,\ell} \le \overline{L}_N,          \\
			\sup_{1 \le \ell \le L} C_{T,\ell} & \le \overline{C}_T, \qquad
			\sup_{1 \le \ell \le L} C_{N,\ell} \le \overline{C}_N, \qquad
			\sup_{1 \le \ell \le L} C_{\sigma,\ell} \le \overline{C}_\sigma,
		\end{aligned}
	\end{align*}
	with $\overline{L}_T,\overline{L}_N,\overline{C}_T,\overline{C}_N,\overline{C}_\sigma$ independent of $L$. Let $x_h$ and $x_{h'}$ denote discretizations of the same input function $u \in C^2(\Omega)$. Then
	\begin{equation}
		\norm{\mathcal{G}_{\theta,h}(x_h) - P_{h' \to h}\mathcal{G}_{\theta,h'}(x_{h'})}_{L^2}
		\leq
		C_{\mathrm{net}}
		\left(
		\norm{x_h - P_{h' \to h}x_{h'}}_{L^2}
		+ L \delta_p
		+ (L+1)\delta_2
		\right),
	\end{equation}
	where $C_{\mathrm{net}}$ depends only on the lifting and projection constants $L_0, C_0, L_{\mathrm{out}}, C_{\mathrm{out}}$, on $A$, $L_\sigma$, and on the uniform bounds $\overline{L}_T,\overline{L}_N,\overline{C}_T,\overline{C}_N,\overline{C}_\sigma$.
\end{theorem}

\begin{remark}[Quadrature-consistent trajectory specialization]
	Under the hypotheses of Theorem~\ref{thm:transfer-bound}, suppose additionally that for each $\ell$ the compared pre-normalization features are exact nodal samples of a common $C^2$ feature field on endpoint-inclusive uniform tensor-product grids, the normalization layer uses the corresponding trapezoidal weights, these features satisfy the interpolation hypothesis of Theorem~\ref{thm:output-consistency}, and their weighted variances are bounded away from zero. Then Theorem~\ref{thm:output-consistency} yields the stronger trajectory-wise estimate
	\begin{align*}
		\|\mathcal{N}_{\ell,h}(z_h) - P_{h' \to h}\mathcal{N}_{\ell,h'}(z_{h'})\|_{L^2}
		\le
		\widetilde{C}_{N,\ell}\delta_2
	\end{align*}
	for some constant $\widetilde{C}_{N,\ell}$ independent of $h,h'$ along that trajectory. Thus, the normalization assumption above holds with $p=2$ there, for example, by taking $L_{N,\ell}=0$ and $C_{N,\ell}=\widetilde{C}_{N,\ell}$. If $\norm{x_h - P_{h' \to h}x_{h'}}_{L^2} = O(\delta_2)$, the theorem therefore gives an overall cross-resolution transfer bound of order $O(L\cdot\delta_2)$ for $L \ge 1$ under this cumulative-stability assumption.
\end{remark}

\begin{proof}[Proof of Theorem~\ref{thm:transfer-bound}]
	We decompose the forward pass of $\mathcal{G}_{\theta,h}$ into a lifting map $f_{0,h}$, processor blocks $\{f_{\ell,h}\}_{\ell=1}^L$, and a projection map $f_{L+1,h}$. Each processor block has the form
	\begin{align*}
		f_{\ell,h} = \sigma \circ \mathcal{N}_{\ell,h} \circ T_{\ell,h},
	\end{align*}
	and we write $\delta_2 = \max(h^2,h'^2)$ and $\delta_p = \max(h^p,h'^p)$.

	\paragraph{Lifting.} Let
	\begin{align*}
		e_{-1} = \|x_h - P_{h' \to h}x_{h'}\|_{L^2},
		\qquad
		e_0 = \|f_{0,h}(x_h) - P_{h' \to h}f_{0,h'}(x_{h'})\|_{L^2}.
	\end{align*}
	By the lifting assumption,
	\begin{align*}
		e_0 \le L_0 e_{-1} + C_0 \delta_2.
	\end{align*}

	\paragraph{Per-block error propagation.} For $1 \le \ell \le L$, let $e_\ell$ be the discrepancy after block $\ell$. If $z_{\ell,h}$ and $z_{\ell,h'}$ denote the inputs to $\mathcal{N}_{\ell,h}$ and $\mathcal{N}_{\ell,h'}$ after applying $T_{\ell,h}$ and $T_{\ell,h'}$, then the resolution-compatibility of $T_{\ell,h}$ gives
	\begin{align*}
		\|z_{\ell,h} - P_{h' \to h}z_{\ell,h'}\|_{L^2}
		\le
		L_{T,\ell} e_{\ell-1} + C_{T,\ell}\delta_2.
	\end{align*}
	Applying the resolution-compatible normalization map yields
	\begin{align*}
		\|\mathcal{N}_{\ell,h}(z_{\ell,h}) - P_{h' \to h}\mathcal{N}_{\ell,h'}(z_{\ell,h'})\|_{L^2}
		\le
		L_{N,\ell}(L_{T,\ell} e_{\ell-1} + C_{T,\ell}\delta_2)
		+ C_{N,\ell}\delta_p.
	\end{align*}
	Applying $\sigma$ and using the activation--interpolation commutator bound gives
	\begin{align*}
		e_\ell
		\le
		L_\sigma L_{N,\ell}L_{T,\ell} e_{\ell-1}
		+ L_\sigma C_{N,\ell}\delta_p
		+ (L_\sigma L_{N,\ell} C_{T,\ell} + C_{\sigma,\ell})\delta_2.
	\end{align*}
	With $M_\ell = L_\sigma L_{N,\ell}L_{T,\ell}$, this is
	\begin{align*}
		e_\ell
		\le
		M_\ell e_{\ell-1}
		+ L_\sigma C_{N,\ell}\delta_p
		+ (L_\sigma L_{N,\ell} C_{T,\ell} + C_{\sigma,\ell})\delta_2.
	\end{align*}

	\paragraph{Telescoping through the processor stack.} Unrolling the recursion yields
	\begin{align}
		e_L
		 & \le
		\left(\prod_{\ell=1}^{L} M_\ell\right)e_0
		+ \sum_{\ell=1}^{L}\left(\prod_{j=\ell+1}^{L} M_j\right)L_\sigma C_{N,\ell}\delta_p \nonumber \\
		 & \quad
		+ \sum_{\ell=1}^{L}\left(\prod_{j=\ell+1}^{L} M_j\right)(L_\sigma L_{N,\ell} C_{T,\ell} + C_{\sigma,\ell})\delta_2.
	\end{align}
	Using the tail product bound together with the uniform bounds on $L_{N,\ell}$, $C_{T,\ell}$, $C_{N,\ell}$, and $C_{\sigma,\ell}$ gives
	\begin{align*}
		e_L
		\le
		A e_0
		+ A L_\sigma \overline{C}_N L \delta_p
		+ A(L_\sigma \overline{L}_N \overline{C}_T + \overline{C}_\sigma)L \delta_2.
	\end{align*}
	Substituting the lifting bound for $e_0$ and absorbing $A$, $L_0$, $C_0$, $L_\sigma$, and the uniform bounds into a constant $C_{\mathrm{mid}}$ yields
	\begin{align*}
		e_L
		\le
		C_{\mathrm{mid}}(e_{-1} + L\delta_p + (L+1)\delta_2).
	\end{align*}

	\paragraph{Projection.} Apply the resolution-compatible projection map:
	\begin{align*}
		\|\mathcal{G}_{\theta,h}(x_h) - P_{h' \to h}\mathcal{G}_{\theta,h'}(x_{h'})\|_{L^2}
		\le
		L_{\mathrm{out}} e_L + C_{\mathrm{out}}\delta_2.
	\end{align*}
	Combining this with the previous estimate and absorbing $L_{\mathrm{out}}$ and $C_{\mathrm{out}}$ into a final constant gives
	\begin{align*}
		\|\mathcal{G}_{\theta,h}(x_h) - P_{h' \to h}\mathcal{G}_{\theta,h'}(x_{h'})\|_{L^2}
		\le
		C_{\mathrm{net}}(e_{-1} + L\delta_p + (L+1)\delta_2),
	\end{align*}
	which is exactly the stated bound.
\end{proof}

\section{Implementation and Experimental Setup}
\label{app:details}
\label{app:additional-details}

\paragraph{FNO architecture.} The architecture uses a lifting layer that maps the input together with two coordinate channels to the model width, followed by $L$ Fourier blocks with spectral convolution, pointwise linear layers, residual connections, normalization, and Gaussian error linear unit activations, and then a projection multilayer perceptron that maps the width to 128 and then to the output.

\paragraph{Optional adaptive residual gain.} \label{app:residual-gain} For FNO blocks with residual connections $y = x + F(x)$, one may optionally apply an energy-balancing gain:
\begin{equation}\label{eq:gain}
	y = x + \alpha_0 \sqrt{\frac{\E[\|x\|^2]}{\E[\|F(x)\|^2] + \varepsilon}} \cdot F(x),
\end{equation}
where $\alpha_0 \ge 0$ is a learnable scalar, such as one parameterized via a softplus and initialized at $0.5$. At the expectation level, Cauchy--Schwarz gives
\begin{align*}
	\E\|y\|^2 \leq (1 + \alpha_0)^2 \E\|x\|^2,
\end{align*}
so the residual branch has aggregate energy-growth control. In our experiments, this module is disabled by default. It is orthogonal to the normalization contribution and can be activated independently if desired.

\paragraph{Galerkin Transformer.} This model uses Galerkin-type attention with 4 layers, 4 heads, and width 64, with normalization after each block.

\paragraph{Quadrature weights.} On 2D uniform grids, the weight tensor is the outer product of 1D trapezoidal weights $w_{ij} = w_i^{(1)} \cdot w_j^{(2)}$. For nonuniform one-dimensional grids, the weights are $w_0 = (x_1 - x_0)/2$, $w_i = (x_{i+1} - x_{i-1})/2$ for $1 \le i \le n-2$, and $w_{n-1} = (x_{n-1} - x_{n-2})/2$. In higher dimensions and on irregular meshes, we use cell volumes or Voronoi volumes.

\paragraph{Reference PyTorch implementation.} Listing~\ref{lst:quadnorm-code} gives the compact \texttt{nn.Module} implementation of QuadNorm corresponding to the definitions used in this paper. The experiment code uses the same tensor convention, with tensors arranged as $(B,C,\ast\text{spatial})$.

\begin{lstlisting}[style=quadnormcode,caption={Minimal PyTorch implementation of QuadNorm.},label={lst:quadnorm-code}]
from typing import Literal, Tuple

import torch
import torch.nn as nn


def trapezoidal_weights_nd(
    shape: Tuple[int, ...],
    device: torch.device,
    dtype: torch.dtype,
) -> torch.Tensor:
    weights_1d = []
    for n in shape:
        if n == 1:
            w = torch.ones(1, device=device, dtype=dtype)
        else:
            h = 1.0 / (n - 1)
            w = torch.full((n,), h, device=device, dtype=dtype)
            w[0] = h / 2
            w[-1] = h / 2
        weights_1d.append(w)

    if len(weights_1d) == 1:
        return weights_1d[0]
    if len(weights_1d) == 2:
        return torch.outer(weights_1d[0], weights_1d[1])
    if len(weights_1d) == 3:
        return torch.einsum("i,j,k->ijk", *weights_1d)

    out = weights_1d[0]
    for w in weights_1d[1:]:
        out = out.unsqueeze(-1) * w
    return out


class QuadNorm(nn.Module):
    """Pure quadrature-weighted normalization for neural operators."""

    def __init__(
        self,
        num_features: int,
        eps: float = 1e-5,
        affine: bool = True,
        mode: Literal["layer", "instance", "group"] = "layer",
        num_groups: int = 1,
    ):
        super().__init__()
        self.eps = eps
        self.mode = mode
        self.num_groups = num_groups
        if affine:
            self.weight = nn.Parameter(torch.ones(num_features))
            self.bias = nn.Parameter(torch.zeros(num_features))
        else:
            self.register_parameter("weight", None)
            self.register_parameter("bias", None)

    def _affine(self, y: torch.Tensor) -> torch.Tensor:
        if self.weight is None:
            return y
        shape = [1, -1] + [1] * (y.ndim - 2)
        return y * self.weight.view(*shape) + self.bias.view(*shape)

    def forward(self, x: torch.Tensor) -> torch.Tensor:
        spatial = x.shape[2:]
        w = trapezoidal_weights_nd(spatial, x.device, x.dtype)
        w = w.view(1, 1, *spatial)

        if self.mode == "instance":
            dims = tuple(range(2, x.ndim))
            denom = w.sum()
            mean = (x * w).sum(dim=dims, keepdim=True) / denom
            var = ((x - mean).pow(2) * w).sum(dim=dims, keepdim=True) / denom
            return self._affine((x - mean) / torch.sqrt(var + self.eps))

        if self.mode == "layer":
            dims = tuple(range(1, x.ndim))
            denom = w.sum() * x.shape[1]
            mean = (x * w).sum(dim=dims, keepdim=True) / denom
            var = ((x - mean).pow(2) * w).sum(dim=dims, keepdim=True) / denom
            return self._affine((x - mean) / torch.sqrt(var + self.eps))

        if self.mode == "group":
            bsz, channels = x.shape[:2]
            groups = self.num_groups
            if channels % groups != 0:
                raise ValueError("num_features must be divisible by num_groups")
            xg = x.view(bsz, groups, channels // groups, *spatial)
            wg = w.unsqueeze(2)
            dims = tuple(range(2, xg.ndim))
            denom = w.sum() * (channels // groups)
            mean = (xg * wg).sum(dim=dims, keepdim=True) / denom
            var = ((xg - mean).pow(2) * wg).sum(dim=dims, keepdim=True) / denom
            y = (xg - mean) / torch.sqrt(var + self.eps)
            return self._affine(y.view(bsz, channels, *spatial))

        raise ValueError(f"unknown QuadNorm mode: {self.mode}")
\end{lstlisting}

\paragraph{BlendQuadNorm forward pass.} For each batch element, let $x$ have shape $(C, H, W)$, and let $w$ with shape $(H, W)$ denote the cached trapezoidal weights for that resolution. The LayerNorm and quadrature-weighted statistics are
\begin{align*}
	S_w                    & = \sum_{i,j} w_{ij},                                                \\
	\mu_{\text{LayerNorm}} & = \frac{1}{CHW} \sum_{c,i,j} x_c(i,j),                              \\
	v_{\text{LayerNorm}}   & = \frac{1}{CHW} \sum_{c,i,j} (x_c(i,j) - \mu_{\text{LayerNorm}})^2,
	\mu_{\text{WLN}} & = \frac{1}{C S_w} \sum_{c,i,j} w_{ij} x_c(i,j),                        \\
	v_{\text{WLN}}   & = \frac{1}{C S_w} \sum_{c,i,j} w_{ij} (x_c(i,j) - \mu_{\text{WLN}})^2.
\end{align*}
We then blend these statistics by
\begin{align*}
	\mu_b & = \alpha \mu_{\text{LayerNorm}} + (1 - \alpha) \mu_{\text{WLN}},                                                             \\
	v_b   & = \alpha v_{\text{LayerNorm}} + (1 - \alpha) v_{\text{WLN}} + \alpha(1-\alpha)(\mu_{\text{LayerNorm}} - \mu_{\text{WLN}})^2,
\end{align*}
and return
\begin{align*}
	y = \gamma \odot \frac{x - \mu_b}{\sqrt{v_b + \varepsilon}} + \beta.
\end{align*}

The blending coefficient $\alpha \in [0, 1]$ controls the trade-off. The choice $\alpha = 1$ recovers pure LayerNorm and prioritizes single-resolution accuracy without a consistency guarantee. The choice $\alpha = 0$ yields pure quadrature-weighted LayerNorm statistics with the same reduction axes as LayerNorm, and those moments are $O(h^2)$-consistent on endpoint-inclusive uniform grids. Intermediate values interpolate between these two endpoints.

\paragraph{Cross-resolution evaluation.} We train at the source resolution and evaluate at the target resolution using bicubic interpolation. Since FNO's spectral convolution, pointwise linear, lifting, and projection layers are all resolution-invariant, the normalization layer is the primary source of resolution-dependent behavior.

\paragraph{Statistical testing.} We use paired $t$-tests together with Holm--Bonferroni correction \citep{holm1979simple} at $\alpha = 0.05$, pairing observations by seed.

\paragraph{Data splits.} The Darcy benchmark uses 900 training examples and 200 test examples in the main Darcy benchmark transfer studies, 700 training examples and 200 test examples in the Transolver Darcy benchmark study, and 900 training examples, 100 validation examples, and 200 test examples in the comprehensive Darcy benchmark evaluation. The real Navier--Stokes one-step benchmark in the appendix study of Galerkin on real Navier--Stokes uses 8000 training examples and 2000 test sample pairs. The larger-scale Darcy setting uses 5000 training examples and 500 test examples at $64^2$. All inputs include two coordinate channels; outputs are single-channel solution fields.

\paragraph{Hyperparameter sensitivity.} All normalization variants use identical hyperparameters. The only BlendQuadNorm-specific hyperparameter is $\alpha = 0.3$, fixed across experiments; architecture-specific lower-$\alpha$ alternatives are evaluated separately in Appendix Table~\ref{tab:alpha-matrix}. Beyond this fixed $\alpha$, QuadNorm and BlendQuadNorm introduce no additional hyperparameters. Trapezoidal quadrature suffices; higher-order rules perform identically, as shown in Appendix~\ref{app:exp-quadrature-ablation}, and weights are fully determined by grid geometry.

\paragraph{Reproducibility Statement.}\phantomsection\label{app:reproducibility} All experiments use PyTorch 2.8. Source code is provided in the supplementary material. The data sources are the official FNO Darcy benchmark \citep{li2021fno}, obtained from its Zenodo release, the PDE benchmark suite (PDEBench) \citep{tran2023pdebench}, and synthetic sine-series data. Key hyperparameters are width-32 FNO with 4 layers and 12 modes, deep FNO with width 48 and 8 layers, Galerkin with 4 layers, 4 heads, and width 64, and Transolver with width 64 and 6 layers. Adam with decoupled weight decay (AdamW) uses a learning rate of $5 \times 10^{-4}$ for FNO and $3 \times 10^{-4}$ for Galerkin, with weight decay of $10^{-4}$, together with cosine annealing, gradient clipping at 1.0, and an early-stopping patience of 25 to 30 epochs. Seed counts vary by study. The largest transfer experiments use 10 seeds, most ablations and supporting studies use 5, and some real-benchmark extensions use 3.

\section{Supporting Main-Text Studies}
\label{app:supporting-experiments}

\paragraph{Experimental Roadmap.}\phantomsection\label{app:exp-overview-summary} The appendix is organized around five supporting themes that mirror our empirical claims. We first collect studies on gap and depth scaling, then architecture- and PDE-dependent behavior, including moderate-gap transfer and the $\alpha$ recommendation sweep. We next summarize the stronger Darcy benchmark protocol and results, followed by training-regime and model-scale studies. The remaining subsections cover breadth beyond Darcy together with mechanism, robustness, and statistical support analyses, including quadrature ablations, bootstrap confidence intervals, nonuniform mesh bias, and the Two One-Sided Tests (TOST) equivalence analysis.

\paragraph{Summary.} Table~\ref{tab:key-results-summary} summarizes our main empirical message: the strongest gains appear on nonperiodic PDEs and nonspectral architectures, large FNO transfer gains emerge at wider resolution gaps, and the improvement grows with model scale.

\begin{table}[t]
	\centering
	\caption{Summary of key results. Improvement is relative to LayerNorm, so positive values indicate better performance. The final column reports supporting evidence in the form of exact 95\% bootstrap CIs where they are tabulated later in Tables~\ref{tab:bootstrap-cis} and~\ref{tab:bootstrap-cis-breadth}, and otherwise rounded task-specific confidence intervals or study metadata for the 3-seed strong-recipe benchmark rows.}
	\label{tab:key-results-summary}
	\small
	\resizebox{\textwidth}{!}{%
		\begin{tabular}{llllcc}
			\toprule
			\tableheadrulebelow
			\rowcolor{tableheadgray}
			Setting                                  & PDE       & Arch       & Best          & Improvement (\%)                 & Support                    \\
			\tableheadruleabove
			\midrule
			\multicolumn{6}{l}{Nonperiodic PDEs with nonspectral architectures, our primary use case}                                                       \\
			Cavity flow transfer                     & Cavity    & Galerkin   & BlendQuadNorm & $\mathbf{+47}$                   & [33, 57]                   \\
			Poisson--Dirichlet 10-seed               & Poisson   & Galerkin   & QuadNorm      & $\mathbf{+26}$                   & [25, 27]                   \\
			Variable-coefficient native (Galerkin)   & Diffusion & Galerkin   & QuadNorm      & $\mathbf{+18}$                   & [17, 19]                   \\
			Variable-coefficient native (Transolver) & Diffusion & Transolver & QuadNorm      & $\mathbf{+21}$                   & [19, 23]                   \\
			Helmholtz 10-seed                        & Helmholtz & Galerkin   & BlendQuadNorm & $\mathbf{+16}$                   & [10, 21]                   \\
			Foundation multi-resolution training     & Darcy     & Transolver & QuadNorm      & $\mathbf{+24}$ to $\mathbf{+28}$ & task-specific CIs [22, 29] \\
			\midrule
			\multicolumn{6}{l}{Darcy benchmark results in the comprehensive evaluation}                                                                       \\
			Darcy Benchmark, 4$\times$ gap             & Darcy     & Transolver & QuadNorm      & $\mathbf{+30}$                   & 3 seeds, 500 epochs        \\
			Darcy Benchmark, 4$\times$ gap             & Darcy     & FNO        & QuadNorm      & $\mathbf{+32}$                   & 3 seeds, 500 epochs        \\
			\midrule
			\multicolumn{6}{l}{FNO with large resolution gaps}                                                                                                \\
			Darcy Benchmark, 13$\times$ gap            & Darcy     & FNO        & QuadNorm      & $\mathbf{+42}$                   & [41, 44]                   \\
			Darcy, 8$\times$ gap                       & Darcy     & FNO        & QuadNorm      & $\mathbf{+35}$                   & [33, 38]                   \\
			\midrule
			\multicolumn{6}{l}{Model scaling, measured by transfer improvement at a 4$\times$ gap}                                                              \\
			307K $\to$ 9.6M                          & Darcy     & Transolver & QuadNorm      & $18\% \to 38\%$                  & monotonic across scales    \\
			\bottomrule
		\end{tabular}
	}
\end{table}

\subsection{Darcy Benchmark Summary Tables}
\label{sec:exp-real-summary}
\paragraph{Comprehensive Darcy Benchmark Setup.} The comprehensive Darcy benchmark evaluation in Section~\ref{sec:exp-real-comprehensive} evaluates FNO, Galerkin, and Transolver on the official FNO Darcy benchmark \citep{li2021fno} with 3 seeds, 500 epochs, cosine annealing, and a 10-epoch warmup. The benchmark subset contains 900 training examples, 100 validation examples, and 200 test examples. The training and validation examples come from \texttt{smooth1}, and the test examples come from \texttt{smooth2}. The native resolution is 421 by 421. Models are trained at $64^2$ and evaluated at $64^2$, $128^2$, and $256^2$ by subsampling from the full-resolution data.

Table~\ref{tab:real-benchmark-best-summary} summarizes the best improvement of either QuadNorm or BlendQuadNorm over LayerNorm on the Darcy benchmark configurations.

\begin{table}[t]
	\centering
	\caption{Larger improvement over LayerNorm between QuadNorm and BlendQuadNorm in percent on the Darcy benchmark with 500 epochs and 3 seeds. Positive values indicate that QuadNorm or BlendQuadNorm performs better, while negative values indicate that LayerNorm performs better. Improvement percentages are computed from the underlying unrounded means, so very small deltas may not be computed exactly from the rounded error columns.}
	\label{tab:real-benchmark-best-summary}
	\small
	\begin{tabular}{lccc}
		\toprule
		\tableheadrulebelow
		\rowcolor{tableheadgray}
		Architecture & Native & 2$\times$ & 4$\times$ \\
		\tableheadruleabove
		\midrule
		FNO          & 0.0\%  & +17.3\% & +31.7\% \\
		Galerkin     & -1.1\% & +5.8\%  & +13.9\% \\
		Transolver   & +8.2\% & +20.8\% & +30.1\% \\
		\bottomrule
	\end{tabular}
\end{table}

\paragraph{Summary of Darcy Benchmark Results.} Table~\ref{tab:real-benchmark-summary} synthesizes the Darcy benchmark results across the standard and extreme-gap settings. QuadNorm is the strongest method throughout, while BlendQuadNorm stays close to LayerNorm on the smaller-effect Darcy benchmark settings here, where the gains range from $+0.5\%$ to $+2.9\%$.

\begin{table}[t]
	\centering
	\caption{Summary of BlendQuadNorm and QuadNorm improvement over LayerNorm across Darcy benchmark experiments. Rows are organized by protocol, architecture, and train-to-test resolution transfer. $\Delta$ denotes relative improvement, so positive values indicate better performance than LayerNorm. Bold indicates the larger improvement between BlendQuadNorm and QuadNorm in each architecture--resolution block.}
	\label{tab:real-benchmark-summary}
	\resizebox{\textwidth}{!}{%
		\begin{tabular}{l l c c c c c c}
			\toprule
			\tableheadrulebelow
			\rowcolor{tableheadgray}
			Protocol        & Arch & Train $\to$ Test & LayerNorm        & BlendQuadNorm    & $\Delta$ & QuadNorm        & $\Delta$         \\
			\tableheadruleabove
			\midrule
			Darcy Benchmark & FNO  & $64^2 \to 128^2$ & 4.01 $\pm$ 0.11  & 3.98 $\pm$ 0.20  & +0.9\%   & 3.78 $\pm$ 0.13 & \textbf{+5.9\%}  \\
			Darcy Benchmark & FNO  & $64^2 \to 256^2$ & 5.30 $\pm$ 0.17  & 5.24 $\pm$ 0.32  & +1.0\%   & 4.07 $\pm$ 0.17 & \textbf{+23.2\%} \\
			Darcy Benchmark & FNO  & $64^2 \to 421^2$ & 5.77 $\pm$ 0.20  & 5.71 $\pm$ 0.37  & +1.0\%   & 4.16 $\pm$ 0.20 & \textbf{+28.0\%} \\
			Darcy Benchmark & Gal  & $64^2 \to 128^2$ & 8.87 $\pm$ 0.42  & 8.83 $\pm$ 0.29  & +0.5\%   & 7.82 $\pm$ 0.15 & \textbf{+11.9\%} \\
			Darcy Benchmark & Gal  & $64^2 \to 256^2$ & 9.39 $\pm$ 0.43  & 9.32 $\pm$ 0.26  & +0.7\%   & 7.96 $\pm$ 0.15 & \textbf{+15.2\%} \\
			Extreme         & FNO  & $32^2 \to 128^2$ & 9.90 $\pm$ 0.33  & 9.76 $\pm$ 0.60  & +1.4\%   & 6.19 $\pm$ 0.37 & \textbf{+37.4\%} \\
			Extreme         & FNO  & $32^2 \to 256^2$ & 11.31 $\pm$ 0.38 & 11.14 $\pm$ 0.70 & +1.5\%   & 6.66 $\pm$ 0.42 & \textbf{+41.1\%} \\
			Extreme         & FNO  & $32^2 \to 421^2$ & 11.83 $\pm$ 0.40 & 11.65 $\pm$ 0.74 & +1.5\%   & 6.81 $\pm$ 0.44 & \textbf{+42.4\%} \\
			Extreme         & Gal  & $32^2 \to 128^2$ & 11.40 $\pm$ 0.55 & 11.13 $\pm$ 0.37 & +2.4\%   & 8.40 $\pm$ 0.15 & \textbf{+26.3\%} \\
			Extreme         & Gal  & $32^2 \to 256^2$ & 12.33 $\pm$ 0.64 & 11.96 $\pm$ 0.41 & +2.9\%   & 8.67 $\pm$ 0.18 & \textbf{+29.6\%} \\
			\bottomrule
		\end{tabular}
	}
\end{table}

\paragraph{Darcy Benchmark Cross-Resolution Transfer.}\label{sec:exp-real-darcy} Table~\ref{tab:real-darcy-crossres} reports transfer on the official FNO benchmark, whose native resolution is 421 by 421, using subsets with 900 training and 200 test examples for FNO and Galerkin Transformer trained at $64^2$.

\begin{table}[t]
	\centering
	\caption{Cross-resolution transfer on Darcy benchmark data from the official FNO \texttt{.mat} files, with native resolution 421 by 421 and a subset of 900 training examples and 200 test examples. Models are trained at $64^2$. FNO is reported up to $421^2$, while Galerkin is reported up to $256^2$. Bold indicates the best entry within each architecture--resolution block.}
	\label{tab:real-darcy-crossres}
	\begin{tabular}{lcccc}
		\toprule
		\tableheadrulebelow
		\rowcolor{tableheadgray}
		\multicolumn{5}{l}{FNO}                                                                                                                   \\
		\tableheadruleabove
		\midrule
		\tableheadrulebelow
		\rowcolor{tableheadgray}
		Method                        & 64$^2$                   & 128$^2$                  & 256$^2$                  & 421$^2$                  \\
		\tableheadruleabove
		\midrule
		None                          & 2.62 $\pm$ 0.10          & 4.44 $\pm$ 0.18          & 5.84 $\pm$ 0.31          & 6.36 $\pm$ 0.36          \\
		LayerNorm                     & \textbf{2.28 $\pm$ 0.08} & 4.01 $\pm$ 0.11          & 5.30 $\pm$ 0.17          & 5.77 $\pm$ 0.20          \\
		QuadNorm                      & 3.38 $\pm$ 0.08          & \textbf{3.78 $\pm$ 0.13} & \textbf{4.07 $\pm$ 0.17} & \textbf{4.16 $\pm$ 0.20} \\
		BlendQuadNorm                 & 2.29 $\pm$ 0.08          & 3.98 $\pm$ 0.20          & 5.24 $\pm$ 0.32          & 5.71 $\pm$ 0.37          \\
		BlendQuadNorm, $\alpha = 0.1$ & 2.29 $\pm$ 0.08          & 4.03 $\pm$ 0.23          & 5.33 $\pm$ 0.38          & 5.82 $\pm$ 0.44          \\
		\bottomrule
	\end{tabular}

	\vspace{0.4em}
	\begin{tabular}{lccc}
		\toprule
		\tableheadrulebelow
		\rowcolor{tableheadgray}
		\multicolumn{4}{l}{Galerkin}                                                                                   \\
		\tableheadruleabove
		\midrule
		\tableheadrulebelow
		\rowcolor{tableheadgray}
		Method                        & 64$^2$                   & 128$^2$                  & 256$^2$                  \\
		\tableheadruleabove
		\midrule
		None                          & 9.42 $\pm$ 0.14          & 9.96 $\pm$ 0.15          & 10.48 $\pm$ 0.17         \\
		LayerNorm                     & 8.35 $\pm$ 0.42          & 8.87 $\pm$ 0.42          & 9.39 $\pm$ 0.43          \\
		QuadNorm                      & \textbf{7.63 $\pm$ 0.16} & \textbf{7.82 $\pm$ 0.15} & \textbf{7.96 $\pm$ 0.15} \\
		BlendQuadNorm                 & 8.32 $\pm$ 0.33          & 8.83 $\pm$ 0.29          & 9.32 $\pm$ 0.26          \\
		BlendQuadNorm, $\alpha = 0.1$ & 8.30 $\pm$ 0.30          & 8.80 $\pm$ 0.28          & 9.29 $\pm$ 0.26          \\
		\bottomrule
	\end{tabular}
\end{table}

Table~\ref{tab:real-darcy-crossres} gives the per-resolution breakdown for the official Darcy benchmark and confirms the earlier synthetic findings. For FNO, QuadNorm achieves $4.16\%$ at $421^2$, compared with $5.77\%$ for LayerNorm, which corresponds to a $28\%$ improvement; for Galerkin, it achieves $7.96\%$ at $256^2$, compared with $9.39\%$, which corresponds to a $15\%$ improvement. In both architectures, the advantage grows monotonically with the resolution gap.

At extreme gaps, namely $32^2 \to 421^2$ with a 13-fold refinement in the Darcy benchmark extreme cross-resolution study reported in Table~\ref{tab:real-darcy-extreme}, the advantage is even more dramatic. QuadNorm reaches $6.81\%$, compared with LayerNorm at $11.83\%$ for FNO, for a $42\%$ improvement. The exact bootstrap confidence interval for this comparison appears in Appendix Table~\ref{tab:bootstrap-cis}. Appendix Table~\ref{tab:real-benchmark-summary} collects a compact synthesis of these Darcy benchmark comparisons.

\paragraph{Transolver on the Darcy Benchmark.}\label{app:exp-transolver-real-darcy} We evaluate Transolver with width 64, depth 6, 8 physics slices, and 4 heads on the official FNO Darcy benchmark. This experiment uses a subset with 700 training examples and 200 test examples, trains at $64^2$ for 300 epochs with 5 seeds, as reported in Table~\ref{tab:transolver-darcy}. QuadNorm achieves an $18\%$ improvement over LayerNorm at a 4$\times$ gap, with a 95\% CI of [17\%, 19\%].

\begin{table}[t]
	\centering
	\caption{Transolver \citep{wu2024transolver} on Darcy benchmark data with a subset of 700 training examples and 200 test examples. Training uses $64^2$, and testing spans $64^2$ to $256^2$, which corresponds to a 4$\times$ gap. Results report relative $L^2$ error in percent as mean $\pm$ 95\% CI over 5 seeds. Bold indicates the best method in each column.}
	\label{tab:transolver-darcy}
	\begin{tabular}{lccc}
		\toprule
		\tableheadrulebelow
		\rowcolor{tableheadgray}
		Method                        & $64^2$ (native)          & $128^2$ (2$\times$)        & $256^2$ (4$\times$)        \\
		\tableheadruleabove
		\midrule
		None                          & 7.73 $\pm$ 0.14          & 8.32 $\pm$ 0.09          & 8.92 $\pm$ 0.12          \\
		LayerNorm                     & 6.85 $\pm$ 0.21          & 7.45 $\pm$ 0.20          & 8.09 $\pm$ 0.19          \\
		QuadNorm                      & \textbf{6.32 $\pm$ 0.22} & \textbf{6.50 $\pm$ 0.23} & \textbf{6.64 $\pm$ 0.25} \\
		BlendQuadNorm                 & 6.47 $\pm$ 0.15          & 7.19 $\pm$ 0.10          & 7.95 $\pm$ 0.11          \\
		BlendQuadNorm, $\alpha = 0.1$ & 6.45 $\pm$ 0.13          & 7.22 $\pm$ 0.06          & 8.01 $\pm$ 0.08          \\
		\bottomrule
	\end{tabular}
\end{table}

\paragraph{Darcy Benchmark Extreme Cross-Resolution Transfer.}\label{sec:exp-real-extreme} Table~\ref{tab:real-darcy-extreme} shows that at the 13$\times$ gap, QuadNorm reaches $6.81\%$, compared with LayerNorm at $11.83\%$, for a $42\%$ improvement for FNO; for Galerkin at $256^2$, it reaches $8.67\%$, compared with $12.33\%$, for a $30\%$ improvement.

\begin{table}[t]
	\centering
	\caption{Extreme cross-resolution transfer on Darcy benchmark data with training at $32^2$ and testing up to $421^2$, which corresponds to a 13$\times$ gap. FNO is reported up to $421^2$, while Galerkin is reported up to $256^2$.}
	\label{tab:real-darcy-extreme}
	\resizebox{\textwidth}{!}{%
		\begin{tabular}{lccccc}
			\toprule
			\tableheadrulebelow
			\rowcolor{tableheadgray}
			\multicolumn{6}{l}{FNO}                                                                                                                                              \\
			\tableheadruleabove
			\midrule
			\tableheadrulebelow
			\rowcolor{tableheadgray}
			Method                        & 32$^2$                   & 64$^2$                   & 128$^2$                  & 256$^2$                  & 421$^2$                  \\
			\tableheadruleabove
			\midrule
			None                          & 3.32 $\pm$ 0.06          & 7.74 $\pm$ 0.55          & 10.82 $\pm$ 0.78         & 12.37 $\pm$ 0.88         & 12.95 $\pm$ 0.92         \\
			LayerNorm                     & \textbf{3.13 $\pm$ 0.05} & 7.08 $\pm$ 0.23          & 9.90 $\pm$ 0.33          & 11.31 $\pm$ 0.38         & 11.83 $\pm$ 0.40         \\
			QuadNorm                      & 4.03 $\pm$ 0.05          & \textbf{5.23 $\pm$ 0.25} & \textbf{6.19 $\pm$ 0.37} & \textbf{6.66 $\pm$ 0.42} & \textbf{6.81 $\pm$ 0.44} \\
			BlendQuadNorm                 & 3.13 $\pm$ 0.06          & 7.00 $\pm$ 0.39          & 9.76 $\pm$ 0.60          & 11.14 $\pm$ 0.70         & 11.65 $\pm$ 0.74         \\
			BlendQuadNorm, $\alpha = 0.1$ & 3.14 $\pm$ 0.06          & 7.10 $\pm$ 0.46          & 9.92 $\pm$ 0.71          & 11.33 $\pm$ 0.83         & 11.85 $\pm$ 0.88         \\
			\bottomrule
		\end{tabular}
	}

	\vspace{0.4em}
	\begin{tabular}{lcccc}
		\toprule
		\tableheadrulebelow
		\rowcolor{tableheadgray}
		\multicolumn{5}{l}{Galerkin}                                                                                                              \\
		\tableheadruleabove
		\midrule
		\tableheadrulebelow
		\rowcolor{tableheadgray}
		Method                        & 32$^2$                   & 64$^2$                   & 128$^2$                  & 256$^2$                  \\
		\tableheadruleabove
		\midrule
		None                          & 9.31 $\pm$ 0.16          & 10.80 $\pm$ 0.19         & 12.35 $\pm$ 0.33         & 13.25 $\pm$ 0.41         \\
		LayerNorm                     & 8.26 $\pm$ 0.42          & 9.80 $\pm$ 0.43          & 11.40 $\pm$ 0.55         & 12.33 $\pm$ 0.64         \\
		QuadNorm                      & \textbf{7.58 $\pm$ 0.21} & \textbf{7.94 $\pm$ 0.16} & \textbf{8.40 $\pm$ 0.15} & \textbf{8.67 $\pm$ 0.18} \\
		BlendQuadNorm                 & 8.32 $\pm$ 0.29          & 9.70 $\pm$ 0.32          & 11.13 $\pm$ 0.37         & 11.96 $\pm$ 0.41         \\
		BlendQuadNorm, $\alpha = 0.1$ & 8.30 $\pm$ 0.28          & 9.66 $\pm$ 0.31          & 11.09 $\pm$ 0.37         & 11.91 $\pm$ 0.41         \\
		\bottomrule
	\end{tabular}
\end{table}

\subsection{Experimental Setup}
\paragraph{Normalization Baselines.}\label{sec:baselines} We compare several normalization methods, all applied identically within FNO blocks after the spectral convolution and pointwise linear layer and before the activation. The no-normalization baseline acts as the identity. LayerNorm, introduced by \citet{ba2016layer}, computes $\mu$ and $\sigma^2$ over all channels and spatial dimensions per sample, so its reduction axes are $(C, H, W)$; in our implementation, the affine parameters are per-channel and constant over spatial positions. InstanceNorm, introduced by \citet{ulyanov2016instance}, computes $\mu$ and $\sigma^2$ per channel and per sample over the spatial dimensions only, so its reduction axes are $(H, W)$. GroupNorm, introduced by \citet{wu2018group}, partitions channels into eight groups and computes statistics over the channels within each group together with the spatial dimensions, so its reduction axes are $(C/G, H, W)$. RMSNorm, introduced by \citet{zhang2019root}, applies root-mean-square scaling without mean centering and uses reduction axes $(H, W)$. QuadNorm uses pure trapezoidal quadrature weights with the native reduction axes of the layer under study, for example, per-channel spatial moments for InstanceNorm-style layers or the $\alpha = 0$ LayerNorm-style endpoint in LayerNorm-based sweeps. BlendQuadNorm blends LayerNorm with its quadrature-weighted LayerNorm counterpart, fixes $\alpha = 0.3$, and uses the same spatially constant affine parameterization described in Section~\ref{sec:blendquadnorm}.

In the extended baseline experiment in Section~\ref{sec:exp-extended-baselines}, we additionally compare SpectralNorm, which performs activation spectral normalization, RI-BandNorm, which performs spectral band normalization with resolution-invariant weighting, and QuadBandNorm, which combines quadrature with spectral bands. Learnable baselines are reported in Appendix~\ref{app:exp-learnable-baselines}.

\paragraph{Datasets.}\label{sec:datasets} For Darcy flow, we use the official FNO Darcy benchmark \citep{li2021fno}, specifically, the two \texttt{.mat} files \texttt{piececonst\_r421\_N1024\_smooth1.mat} and \texttt{piececonst\_r421\_N1024\_smooth2.mat} from its Zenodo release. We train at resolutions $32^2$, $64^2$, or $128^2$ and evaluate cross-resolution transfer to higher resolutions. The input is the log-permeability coefficient, and the output is the solution. The release provides two 1,024-sample files at native 421 by 421 resolution; our real-benchmark studies draw fixed subsets from these files, and the exact split is stated with each study that uses them.

For the Helmholtz equation, we generate solutions to $-\Delta u + k^2 u = f$ on $[0,1]^2$ with homogeneous Dirichlet boundary conditions, so $u = 0$ on $\partial\Omega$, and we use a random wavenumber $k \in [1, 10]$. The solver uses a discrete sine transform, which naturally enforces zero boundary conditions on an interior grid. This Helmholtz setting is a critical test case because periodic FFT grids make quadrature weights collapse to uniform weights by Proposition~\ref{prop:periodic}, whereas Dirichlet boundary conditions create a nontrivial interaction between normalization and the architecture's spectral assumptions. We generate Helmholtz solutions independently at each resolution because Dirichlet grid points do not nest across resolutions. We reuse the same random seeds across resolutions to match PDE instances. Sample counts vary by study and are stated with each study. Most studies train at $64^2$ and evaluate at $96^2$ and $128^2$, while additional experiments evaluate at $256^2$.

Additional synthetic nonperiodic PDEs used in the appendix include Poisson--Dirichlet, cavity flow, elasticity, variable-coefficient diffusion, and reaction-diffusion; their task-specific generation protocols are stated in the corresponding experiment subsections.

\paragraph{Training Protocol.}\label{sec:protocol} All models use AdamW \citep{loshchilov2019adamw} with cosine annealing, warmup, gradient clipping, and early stopping. Darcy training budgets vary by study. The early synthetic sweeps typically use up to 200 epochs with patience 30, while the stronger, later Darcy benchmarks increase the budget to 300 or 500 epochs as stated in the relevant study descriptions.

Seed counts are study-specific and are summarized in Appendix~\ref{app:reproducibility}; individual tables state the exact number of seeds used. Unless otherwise stated, we report mean $\pm$ 95\% CIs and perform paired $t$-tests with Holm--Bonferroni correction \citep{holm1979simple} over the family of tested hypotheses. Full reproducibility details are provided in Appendix~\ref{app:details} under the reproducibility statement.

Appendix experiments provide supporting ablations on nonuniform grids and extended baselines. The nonuniform grid bias analysis shows that quadrature weights reduce mean estimation bias by more than 1500 times, as shown in Figure~\ref{fig:nonuniform-bias}. The extended baseline study shows that spectral-domain methods fail catastrophically, exceeding $97\%$ error at $256^2$ in Appendix Table~\ref{tab:extended-baselines}; details appear in Appendix~\ref{sec:exp-extended-baselines}.

\subsection{Cross-Resolution Error Comparison}
\label{sec:exp-error-decomp}

\begin{table}[t]
	\centering
	\caption{Empirical comparison of cross-resolution degradation across different normalization strategies for Darcy, trained at $64^2$ and evaluated at $128^2$ and $256^2$ with 5 seeds. Native $L^2$ denotes the error at training resolution. Degradation denotes the increase in $L^2$ error at higher resolution and is reported in percentage points. The final column reports degradation relative to the no-normalization model. Values below 1 indicate that the normalization reduces cross-resolution degradation.}
	\label{tab:error-decomp}
	\begin{tabular}{llcccc}
		\toprule
		\tableheadrulebelow
		\rowcolor{tableheadgray}
		Method        & Test Res & Native $L^2$    & Transfer $L^2$  & Degradation & Baseline ratio    \\
		\tableheadruleabove
		\midrule
		No Norm       & $128^2$  & 3.41 $\pm$ 0.13 & 4.94 $\pm$ 0.11 & +1.53       & 1.00 times (ref.) \\
		              & $256^2$  &                 & 6.23 $\pm$ 0.23 & +2.82       & 1.00 times (ref.) \\
		\midrule
		LayerNorm     & $128^2$  & 2.92 $\pm$ 0.15 & 4.42 $\pm$ 0.13 & +1.50       & 0.98 times        \\
		              & $256^2$  &                 & 5.64 $\pm$ 0.17 & +2.71       & 0.96 times        \\
		\midrule
		QuadNorm      & $128^2$  & 4.21 $\pm$ 0.07 & 4.55 $\pm$ 0.14 & +0.34       & 0.22 times        \\
		              & $256^2$  &                 & 4.85 $\pm$ 0.22 & +0.64       & 0.23 times        \\
		\midrule
		BlendQuadNorm & $128^2$  & 2.92 $\pm$ 0.15 & 4.35 $\pm$ 0.17 & +1.42       & 0.93 times        \\
		              & $256^2$  &                 & 5.51 $\pm$ 0.25 & +2.58       & 0.91 times        \\
		\bottomrule
	\end{tabular}
\end{table}

To probe the role of normalization in cross-resolution transfer, we train models with and without normalization and compare degradation at $128^2$ and $256^2$. The no-normalization model provides a reference with no normalization-statistics mismatch, but this comparison is empirical rather than a strict additive decomposition because retraining changes the optimization trajectory as well as the test-time statistics.

Table~\ref{tab:error-decomp} reveals that LayerNorm matches the reference, with 0.98 times at $128^2$ and 0.96 times at $256^2$. QuadNorm shows only 0.22 times to 0.23 times the reference degradation, which is a reduction by a factor of 4.4 to 4.5. In this setting, it therefore yields less degradation than the no-normalization model. BlendQuadNorm achieves 0.91 times to 0.93 times while maintaining LayerNorm's native accuracy. This degradation comparison, visualized in Figure~\ref{fig:error-decomposition}, supports normalization consistency as an important mechanism for cross-resolution robustness.

\begin{figure}[t]
	\centering
	\includegraphics[width=\textwidth]{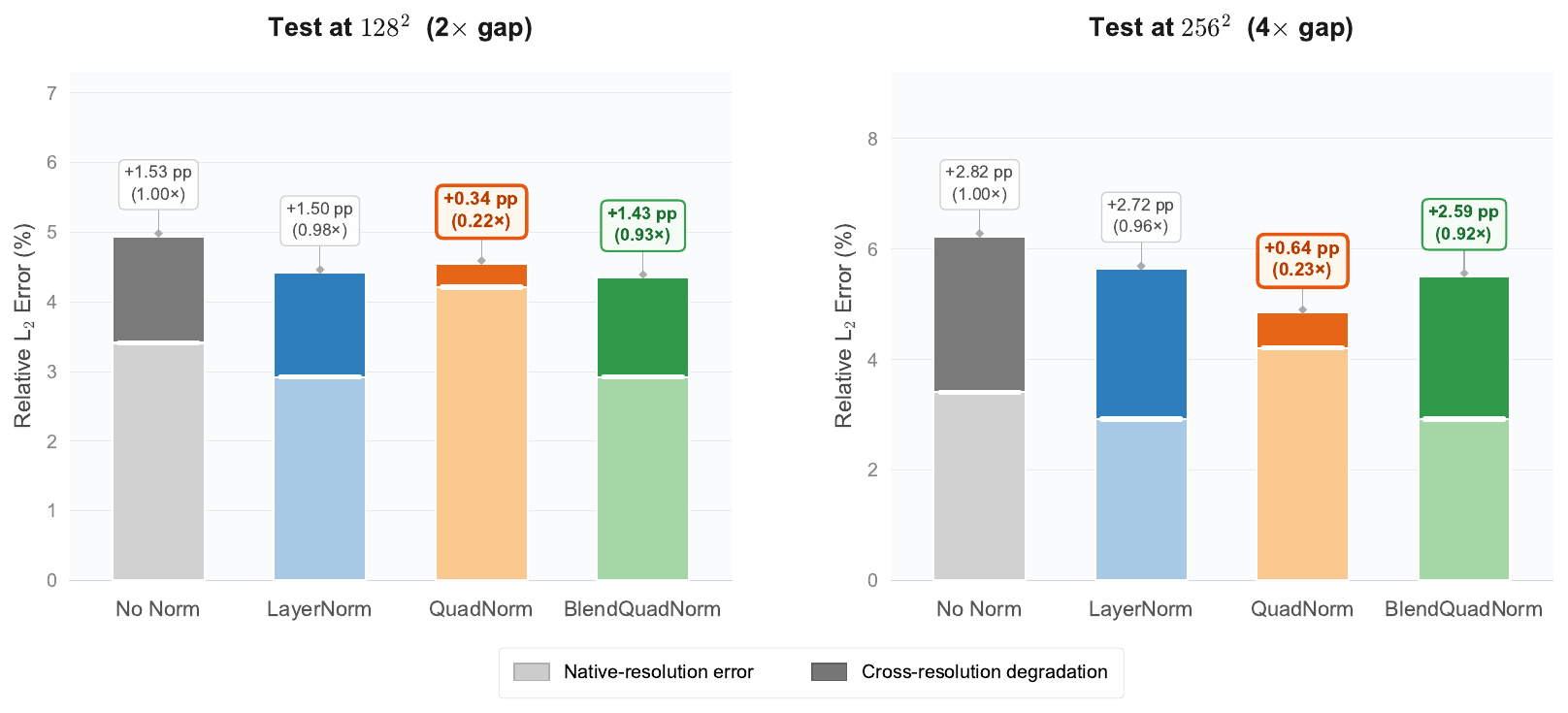}
	\caption{Empirical error-comparison visualization. Each bar shows native error in light shading and cross-resolution degradation in dark shading. The degradation ratio relative to the no-normalization reference shows that QuadNorm yields substantially less degradation than the no-normalization model in this setting, with ratios of 0.22 times to 0.23 times. BlendQuadNorm provides near-baseline degradation, namely 0.91 times to 0.93 times, while matching LayerNorm's native accuracy.}
	\label{fig:error-decomposition}
\end{figure}

\subsection{Transolver Results}
\label{sec:exp-transolver}

We evaluate Transolver \citep{wu2024transolver} by replacing its LayerNorm with QuadNorm or BlendQuadNorm via a spatially aware wrapper that reshapes tokens from $(B, N, D)$ to $(B, D, H, W)$. The Darcy benchmark Transolver transfer study is reported earlier in Appendix~\ref{sec:exp-real-summary} under the heading \hyperref[app:exp-transolver-real-darcy]{Transolver on the Darcy Benchmark}. We also report the nonperiodic Helmholtz and multi-resolution training settings here in the appendix, since they broaden the architecture and training-regime coverage beyond the standard Darcy transfer benchmark.

\paragraph{Transolver on Helmholtz.} Table~\ref{tab:transolver-helmholtz} evaluates Transolver on the nonperiodic Helmholtz equation with 10 seeds. QuadNorm achieves a $5.5\%$ improvement at $128^2$ in Table~\ref{tab:transolver-helmholtz}, with a 95\% CI of [1\%, 9\%]. The wider interval relative to Darcy reflects higher seed sensitivity because 7 of 10 seeds favor QuadNorm. At native resolution, the same experiment shows about twofold lower seed variance for QuadNorm, with $\sigma = 0.60\%$ compared with $\sigma = 1.21\%$ for LayerNorm, which provides more predictable performance.

\begin{table}[t]
	\centering
	\caption{Transolver on the nonperiodic Helmholtz problem. Training uses $64^2$ and testing spans $64^2$ to $128^2$, which corresponds to a 2$\times$ gap. Results report relative $L^2$ error in percent as mean $\pm$ 95\% CI over 10 seeds. QuadNorm outperforms LayerNorm by $5.5\%$ at $128^2$, with a 95\% CI of [1\%, 9\%], and QuadNorm has about twofold lower seed variance at native resolution. Bold indicates the best method in each column.}
	\label{tab:transolver-helmholtz}
	\begin{tabular}{lccc}
		\toprule
		\tableheadrulebelow
		\rowcolor{tableheadgray}
		Method                        & $64^2$ (native)           & $96^2$ (1.5$\times$)        & $128^2$ (2$\times$)         \\
		\tableheadruleabove
		\midrule
		None                          & 19.38 $\pm$ 1.06          & 19.14 $\pm$ 1.06          & 19.02 $\pm$ 1.05          \\
		LayerNorm                     & 14.03 $\pm$ 0.86          & 13.99 $\pm$ 0.91          & 14.02 $\pm$ 0.84          \\
		QuadNorm                      & \textbf{13.33 $\pm$ 0.43} & \textbf{13.13 $\pm$ 0.38} & \textbf{13.25 $\pm$ 0.49} \\
		BlendQuadNorm                 & 13.89 $\pm$ 1.04          & 13.68 $\pm$ 1.19          & 13.65 $\pm$ 1.17          \\
		BlendQuadNorm, $\alpha = 0.1$ & 13.89 $\pm$ 1.08          & 13.69 $\pm$ 1.21          & 13.65 $\pm$ 1.20          \\
		\bottomrule
	\end{tabular}
\end{table}

\paragraph{Transolver Multi-Resolution Training.} When trained jointly on Darcy data at $32^2$ and $64^2$ and tested from $32^2$ to $256^2$, as shown in Table~\ref{tab:transolver-multires} with 5 seeds, QuadNorm achieves a $24\%$ to $28\%$ improvement across all resolutions. It reaches $5.25\%$, compared with LayerNorm at $7.20\%$, at $32^2$, and $5.56\%$, compared with $7.36\%$, at $256^2$ in Table~\ref{tab:transolver-multires}. This multi-resolution training result confirms that the advantage extends to foundation-style multi-resolution training.

\begin{table}[t]
	\centering
	\caption{Transolver under multi-resolution training in a foundation-style scenario. Training uses both $32^2$ and $64^2$, and testing spans $32^2$ to $256^2$. Results report relative $L^2$ error in percent as mean $\pm$ 95\% CI over 5 seeds. QuadNorm achieves a $24\%$ to $28\%$ improvement over LayerNorm across all test resolutions in this setting.}
	\label{tab:transolver-multires}
	\begin{tabular}{lcccc}
		\toprule
		\tableheadrulebelow
		\rowcolor{tableheadgray}
		Method                        & $32^2$ (0.5$\times$)       & $64^2$ (native)          & $128^2$ (2$\times$)        & $256^2$ (4$\times$)        \\
		\tableheadruleabove
		\midrule
		None                          & 8.89 $\pm$ 0.47          & 8.48 $\pm$ 0.54          & 8.92 $\pm$ 0.46          & 9.03 $\pm$ 0.80          \\
		LayerNorm                     & 7.20 $\pm$ 0.49          & 6.78 $\pm$ 0.48          & 7.28 $\pm$ 0.45          & 7.36 $\pm$ 0.53          \\
		QuadNorm                      & \textbf{5.25 $\pm$ 0.20} & \textbf{4.86 $\pm$ 0.31} & \textbf{5.33 $\pm$ 0.33} & \textbf{5.56 $\pm$ 0.41} \\
		BlendQuadNorm                 & 6.61 $\pm$ 0.29          & 6.21 $\pm$ 0.33          & 6.60 $\pm$ 0.24          & 6.69 $\pm$ 0.38          \\
		BlendQuadNorm, $\alpha = 0.1$ & 6.59 $\pm$ 0.34          & 6.25 $\pm$ 0.35          & 6.67 $\pm$ 0.28          & 6.78 $\pm$ 0.44          \\
		\bottomrule
	\end{tabular}
\end{table}

\subsection{Model-Scale Amplification}
\label{sec:exp-scaling}

	\paragraph{Model scaling from 307K to 9.6M parameters.} Results for progressively larger Transolver variants, including an extra-large (XL) model, are reported in Appendix Tables~\ref{tab:transolver-darcy-largescale} and~\ref{tab:transolver-darcy-xl} with 5 seeds each; they show QuadNorm's transfer improvement growing monotonically: 27\% at 1.6M and 38\% at 9.6M on Darcy at a 4$\times$ transfer gap. On Helmholtz at the 9.6M scale, Appendix Table~\ref{tab:transolver-helmholtz-xl} shows that QuadNorm achieves $2.68\%$, compared with LayerNorm at $5.49\%$, at native resolution for a $51\%$ improvement. Table~\ref{tab:model-scaling} summarizes this monotonic trend, supporting the broader claim that discretization consistency becomes more important as model scale increases.

\begin{table}[t]
	\centering
	\caption{Scaling of the QuadNorm advantage with model size on Transolver for Darcy. All models train at $64^2$, and the reported improvement is measured at a 4$\times$ transfer gap, namely at $256^2$.}
	\label{tab:model-scaling}
	\begin{small}
		\begin{tabular}{lccc}
			\toprule
			\tableheadrulebelow
			\rowcolor{tableheadgray}
			Model                       & Params & LayerNorm error at 4$\times$ gap & QuadNorm improvement \\
			\tableheadruleabove
			\midrule
			Standard, width 64, depth 6 & 307K   & 8.09\%                     & 18\% [17 to 19\%]    \\
			Large, width 128, depth 8   & 1.6M   & 6.70\%                     & 27\% [24 to 30\%]    \\
			XL, width 256, depth 12     & 9.6M   & 6.20\%                     & 38\% [36 to 40\%]    \\
			\bottomrule
		\end{tabular}
	\end{small}
\end{table}

The Poisson--Dirichlet follow-up study extends this pattern, as reported in Appendix~\ref{app:exp-pd-galerkin-transolver}. Table~\ref{tab:galerkin-poisson-10seed} shows a $26.2\%$ QuadNorm improvement for Galerkin, with exact bootstrap confidence intervals in Table~\ref{tab:bootstrap-cis}, reinforcing that the benefit extends beyond Darcy on nonspectral architectures.

Additional studies in the appendix report the architecture-dependent $\alpha$ sweep in Appendix~\ref{sec:exp-alpha-matrix}, the bootstrap confidence summaries supporting the headline claims in Appendix~\ref{sec:bootstrap-cis}, and the moderate-gap nonspectral transfer comparison in Appendix~\ref{sec:2x-gap}.

\subsection{Architecture--PDE Alpha Sweep}
\label{sec:exp-alpha-matrix}

\paragraph{Systematic $\alpha$ sweep across architectures and PDEs.} We sweep $\alpha \in \{0, 0.1, 0.2, 0.3, 0.5\}$ across two nonspectral, nonperiodic architecture--PDE combinations, namely Galerkin with Poisson--Dirichlet and Transolver with Helmholtz, using 5 seeds for each setting in Table~\ref{tab:alpha-matrix}. The matrix reveals a clean pattern. In both settings, lower $\alpha$ values are favored for transfer, with the pure quadrature-weighted endpoint at $\alpha = 0$ giving the strongest transfer on the tested tasks. This transfer pattern motivates an empirical recommendation. For nonspectral architectures on nonperiodic PDEs, favor the pure quadrature endpoint. For example, one can choose $\alpha = 0$ in the LayerNorm-style family or use the matching pure quadrature variant for the architecture, while BlendQuadNorm remains the safest default when deployment is uncertain.

\begin{table}[t]
	\centering
	\caption{Architecture--PDE $\alpha$ recommendation matrix. The sweep uses $\alpha \in \{0, 0.1, 0.2, 0.3, 0.5\}$ across two architecture--PDE combinations with 5 seeds each. Results report relative $L^2$ error in percent as mean $\pm$ standard deviation. Bold indicates the best method in each column.}
	\label{tab:alpha-matrix}
	\begin{tabular}{lcccc}
		\toprule
		\tableheadrulebelow
		\rowcolor{tableheadgray}
		                              & \multicolumn{2}{c}{Galerkin+Poisson} & \multicolumn{2}{c}{Transolver+Helmholtz}                                                        \\
		\tableheadruleabove
		\cmidrule(lr){2-3} \cmidrule(lr){4-5}
		\tableheadrulebelow
		\rowcolor{tableheadgray}
		$\alpha$                      & Native                               & 2$\times$                                  & Native                    & 2$\times$                  \\
		\tableheadruleabove
		\midrule
		LayerNorm, baseline           & 24.26 $\pm$ 3.92                     & 21.06 $\pm$ 3.32                         & 11.00 $\pm$ 0.56          & 8.32 $\pm$ 0.39          \\
		$\alpha = 0$, pure quadrature & \textbf{21.08 $\pm$ 3.00}            & \textbf{16.91 $\pm$ 1.80}                & \textbf{10.19 $\pm$ 0.75} & \textbf{7.11 $\pm$ 0.44} \\
		$\alpha=0.1$                  & 24.42 $\pm$ 4.59                     & 20.31 $\pm$ 4.27                         & 10.86 $\pm$ 0.54          & 8.44 $\pm$ 0.29          \\
		$\alpha=0.2$                  & 23.20 $\pm$ 0.81                     & 18.95 $\pm$ 1.31                         & 10.86 $\pm$ 0.53          & 8.45 $\pm$ 0.30          \\
		$\alpha=0.3$ (BlendQuadNorm)  & 23.87 $\pm$ 3.25                     & 19.31 $\pm$ 2.16                         & 10.89 $\pm$ 0.56          & 8.45 $\pm$ 0.30          \\
		$\alpha=0.5$                  & 23.10 $\pm$ 1.47                     & 18.59 $\pm$ 1.03                         & 10.93 $\pm$ 0.59          & 8.44 $\pm$ 0.30          \\
		\bottomrule
	\end{tabular}
\end{table}

\paragraph{Statistical Confidence.}
\label{sec:bootstrap-cis}

All headline improvement claims based on 5 to 10 seeds are supported by 10,000-sample bootstrap CIs. Tables~\ref{tab:bootstrap-cis} and~\ref{tab:bootstrap-cis-breadth} tabulate the exact intervals cited most often. Appendix Table~\ref{tab:key-results-summary} may show rounded task-specific confidence intervals when the same setting is not duplicated there. The 3-seed strong-recipe benchmark results from the comprehensive Darcy benchmark evaluation are reported separately without bootstrap CI entries in Appendix Table~\ref{tab:key-results-summary}. CIs are tight; the $35\%$ result for the 8$\times$ gap has CI $[33\%, 38\%]$, the $42\%$ result for the 13$\times$ gap has CI $[41\%, 44\%]$, and the 10-seed Galerkin Poisson--Dirichlet result yields $26.2\%$ at native resolution with CI $[24.9\%, 27.4\%]$. Several 2$\times$ FNO comparisons, especially the standard Darcy BlendQuadNorm setting, have CIs that include zero, while the clearest positive gains appear at larger gaps or on nonspectral architectures.

\paragraph{Summary for FNO users.} Our strongest results arise on nonspectral architectures, but BlendQuadNorm is still a valuable upgrade for FNO. In mixed-resolution deployment, it serves as a conservative near-parity default whose FNO gains are scenario-dependent but often positive, while requiring only a single drop-in replacement. The spectral and band baselines in Appendix~\ref{app:bandnorm-ablation} further show that spectral filtering is not a robust substitute under resolution transfer: SpectralNorm, RI-BandNorm, and QuadBandNorm exceed $97\%$ error at $256^2$ in Table~\ref{tab:extended-baselines}.

\subsection{Moderate-Gap Nonspectral Transfer}
\label{sec:2x-gap}

As shown in Table~\ref{tab:galerkin-2xgap}, Galerkin Transformer results on two nonperiodic PDEs at $64 \to 128$ with 5 seeds show a consistent improvement at a 2$\times$ gap more clearly than the FNO results. The FNO gains at 2$\times$ gaps are smaller and more setting-dependent, which indicates that the benefit is architecture-dependent rather than purely gap-dependent.

\begin{table}[t]
	\centering
	\caption{Galerkin Transformer at the practical 2$\times$ gap from $64$ to $128$ on two nonperiodic PDEs, using 5 seeds for each setting. Nonspectral architectures benefit from QuadNorm even at moderate resolution gaps. Results report relative $L^2$ error in percent as mean $\pm$ standard deviation.}
	\label{tab:galerkin-2xgap}
	\small
	\begin{tabular}{lcccc}
		\toprule
		\tableheadrulebelow
		\rowcolor{tableheadgray}
		              & \multicolumn{2}{c}{Poisson (Dir.)} & \multicolumn{2}{c}{Darcy Benchmark}                                                       \\
		\tableheadruleabove
		\cmidrule(lr){2-3} \cmidrule(lr){4-5}
		\tableheadrulebelow
		\rowcolor{tableheadgray}
		Method        & Native                             & 2$\times$                             & Native                   & 2$\times$                  \\
		\tableheadruleabove
		\midrule
		LayerNorm     & 11.04 $\pm$ 0.50                   & 11.13 $\pm$ 0.50                    & 8.35 $\pm$ 0.30          & 8.87 $\pm$ 0.31          \\
		QuadNorm      & \textbf{7.66 $\pm$ 0.32}           & \textbf{7.87 $\pm$ 0.31}            & \textbf{7.63 $\pm$ 0.12} & \textbf{7.82 $\pm$ 0.11} \\
		BlendQuadNorm & 7.87 $\pm$ 0.25                    & 7.92 $\pm$ 0.25                     & 8.32 $\pm$ 0.24          & 8.83 $\pm$ 0.21          \\
		None          & 9.45 $\pm$ 0.30                    & 9.60 $\pm$ 0.29                     & 9.42 $\pm$ 0.10          & 9.96 $\pm$ 0.11          \\
		\bottomrule
	\end{tabular}
\end{table}

	\paragraph{Darcy Architecture Transfer.}\phantomsection\label{app:exp-arch-transfer} Table~\ref{tab:arch-transfer} reports Darcy transfer using Galerkin as an additional architecture-level comparison, showing that Galerkin achieves a $5.1\%$ improvement, decreasing the error from $19.36\%$ to $18.37\%$.

\begin{table}[t]
	\centering
	\caption{Architecture transfer for Darcy using Galerkin Transformer \citep{cao2021galerkin}, reporting relative $L^2$ error in percent as mean $\pm$ 95\% CI over 10 seeds. The table reports native performance at $64^2$ and transfer from $64$ to $128$.}
	\label{tab:arch-transfer}
	\begin{tabular}{lcc}
		\toprule
		\tableheadrulebelow
		\rowcolor{tableheadgray}
		Method        & Galerkin $64^2$           & Galerkin $64^2 \to 128^2$ \\
		\tableheadruleabove
		\midrule
		None          & 21.00 $\pm$ 1.65          & 21.30 $\pm$ 1.63          \\
		LayerNorm     & 19.17 $\pm$ 0.45          & 19.36 $\pm$ 0.47          \\
		RMSNorm       & 19.00 $\pm$ 0.50          & 19.17 $\pm$ 0.52          \\
		BlendQuadNorm & \textbf{18.16 $\pm$ 0.67} & \textbf{18.37 $\pm$ 0.68} \\
		\bottomrule
	\end{tabular}
\end{table}

\subsection{Bootstrap Confidence Intervals for Main Text Results}
\label{app:bootstrap-ci-details}

These intervals, reported in Table~\ref{tab:bootstrap-cis}, provide exact paired bootstrap confidence intervals for the headline improvement ratios cited in the main text. They reinforce the same pattern emphasized throughout the paper: the strongest support appears for larger-gap Darcy transfer, deeper models, and nonspectral nonperiodic operators, while near-zero intervals are confined to close comparisons where the architecture-dependent tradeoff between QuadNorm and the safer BlendQuadNorm is small.

\begin{table*}[t]
	\centering
	\caption{Bootstrap 95\% CIs for key improvement ratios, grouped by the headline claims emphasized in the paper. Improvement is defined by $1 - (\overline{e}_{\text{method}}/\overline{e}_{\text{LayerNorm}})$ using 10,000 paired resamples. Mean error is shown in percent as LayerNorm $\to$ the better of QuadNorm and BlendQuadNorm.}
	\label{tab:bootstrap-cis}
	\small
	\begingroup
	\setlength{\tabcolsep}{4pt}
	\renewcommand{\arraystretch}{1.12}
	\begin{tabularx}{\textwidth}{@{}>{\raggedright\arraybackslash}p{0.17\textwidth}>{\raggedright\arraybackslash}Xl c>{\raggedright\arraybackslash}p{0.17\textwidth}cc@{}}
		\toprule
		\tableheadrulebelow
		\rowcolor{tableheadgray}
		Block                       & Experiment                      & Best norm     & $n$ & Mean error (\%)     & Improvement     & 95\% CI          \\
		\tableheadruleabove
		\midrule
		\multicolumn{7}{@{}l}{\textbf{FNO on Darcy: gap and depth scaling}}                                                                            \\
		Synthetic Darcy             & $32^2 \to 256^2$, 8$\times$ gap   & QuadNorm      & 10  & 10.995 $\to$ 7.102  & \textbf{35.4\%} & [33.0\%, 37.9\%] \\
		Synthetic Darcy             & $32^2 \to 128^2$, 4$\times$ gap   & QuadNorm      & 10  & 9.619 $\to$ 6.600   & \textbf{31.4\%} & [29.0\%, 33.9\%] \\
		Darcy Benchmark             & $32^2 \to 421^2$, 13$\times$ gap  & QuadNorm      & 5   & 11.828 $\to$ 6.813  & \textbf{42.4\%} & [40.7\%, 44.1\%] \\
		Darcy Benchmark             & $32^2 \to 256^2$, 8$\times$ gap   & QuadNorm      & 5   & 11.312 $\to$ 6.659  & \textbf{41.1\%} & [39.4\%, 42.9\%] \\
		Darcy Benchmark             & $64^2 \to 256^2$                & QuadNorm      & 5   & 5.298 $\to$ 4.069   & \textbf{23.2\%} & [21.8\%, 24.5\%] \\
		Darcy Benchmark             & $64^2 \to 128^2$                & BlendQuadNorm & 5   & 4.012 $\to$ 3.976   & 0.9\%           & [-0.5\%, 2.3\%]  \\
		Synthetic Darcy, 8-layer    & $64^2 \to 256^2$                & QuadNorm      & 10  & 5.137 $\to$ 4.174   & \textbf{18.8\%} & [17.6\%, 19.8\%] \\
		\midrule
		\multicolumn{7}{@{}l}{\textbf{Galerkin on nonperiodic PDEs and the Darcy benchmark}}                                                          \\
		Helmholtz                   & $64^2 \to 128^2$                & BlendQuadNorm & 5   & 15.568 $\to$ 14.058 & \textbf{9.7\%}  & [6.8\%, 12.6\%]  \\
		Helmholtz                   & $64^2 \to 96^2$                 & BlendQuadNorm & 5   & 15.533 $\to$ 14.003 & \textbf{9.9\%}  & [7.0\%, 12.7\%]  \\
		Helmholtz                   & $64^2 \to 128^2$                & QuadNorm      & 5   & 15.568 $\to$ 14.917 & 4.2\%           & [-0.7\%, 7.5\%]  \\
		Cavity flow                 & $32^2 \to 128^2$                & BlendQuadNorm & 5   & 5.410 $\to$ 2.890   & \textbf{46.6\%} & [33.1\%, 57.0\%] \\
		Darcy Benchmark             & $64^2 \to 128^2$                & BlendQuadNorm & 5   & 8.872 $\to$ 8.827   & 0.5\%           & [-2.5\%, 2.4\%]  \\
		Darcy Benchmark             & $64^2 \to 256^2$                & BlendQuadNorm & 5   & 9.386 $\to$ 9.316   & 0.7\%           & [-2.0\%, 2.5\%]  \\
		Darcy Benchmark             & $32^2 \to 256^2$, 8$\times$ gap   & QuadNorm      & 5   & 12.327 $\to$ 8.674  & \textbf{29.6\%} & [26.9\%, 31.5\%] \\
		Poisson--Dirichlet          & native at $64^2 \to 64^2$, 10 seeds & QuadNorm      & 10  & 6.104 $\to$ 4.506   & \textbf{26.2\%} & [24.9\%, 27.4\%] \\
		Poisson--Dirichlet          & $64^2 \to 128^2$                & QuadNorm      & 5   & 34.679 $\to$ 32.346 & 6.7\%           & [-0.3\%, 13.5\%] \\
		Poisson--Dirichlet          & $64^2 \to 128^2$                & BlendQuadNorm & 5   & 34.679 $\to$ 32.136 & \textbf{7.3\%}  & [3.1\%, 11.2\%]  \\
		Poisson--Dirichlet          & native at $64^2 \to 64^2$       & QuadNorm      & 5   & 34.634 $\to$ 31.627 & \textbf{8.7\%}  & [2.0\%, 15.6\%]  \\
		Poisson--Dirichlet          & native at $64^2 \to 64^2$       & BlendQuadNorm & 5   & 34.634 $\to$ 31.915 & \textbf{7.9\%}  & [3.7\%, 11.6\%]  \\
		\midrule
		\multicolumn{7}{@{}l}{\textbf{Transolver and foundation-style training}}                                                                       \\
		Cross-fidelity foundation   & $32^2$, $64^2$, $128^2 \to 256^2$ & QuadNorm      & 5   & 6.271 $\to$ 5.439   & \textbf{13.3\%} & [6.2\%, 19.4\%]  \\
		Darcy                       & $64^2 \to 128^2$                & BlendQuadNorm & 5   & 7.449 $\to$ 7.194   & \textbf{3.4\%}  & [1.7\%, 4.9\%]   \\
		Darcy                       & $64^2 \to 256^2$                & QuadNorm      & 5   & 8.091 $\to$ 6.638   & \textbf{18.0\%} & [17.1\%, 18.8\%] \\
		Helmholtz                   & $64^2 \to 128^2$                & QuadNorm      & 10  & 14.017 $\to$ 13.251 & \textbf{5.5\%}  & [1.4\%, 9.2\%]   \\
		Helmholtz                   & $64^2 \to 128^2$                & BlendQuadNorm & 10  & 14.017 $\to$ 13.652 & 2.6\%           & [-0.4\%, 5.6\%]  \\
		Foundation multi-resolution & $32^2$, $64^2 \to 256^2$        & QuadNorm      & 5   & 7.360 $\to$ 5.560   & \textbf{24.5\%} & [21.5\%, 27.3\%] \\
		\bottomrule
	\end{tabularx}
	\endgroup
\end{table*}

\section{Extended Benchmarks, Ablations, and Robustness Checks}
\label{app:additional-experiments}

\subsection{Helmholtz Galerkin Transfer}
\label{app:exp-helmholtz-galerkin-strong}

\begin{table}[t]
	\centering
	\caption{Helmholtz equation with a stronger Galerkin Transformer baseline using width 48, 300 epochs, and patience 60. Training uses $64^2$, and testing covers $64^2$, $96^2$, and $128^2$. Results report relative $L^2$ error in percent as mean $\pm$ 95\% CI over 5 seeds. Bold indicates the best method in each column.}
	\label{tab:helmholtz-galerkin-strong}
	\begin{tabular}{lccc}
		\toprule
		\tableheadrulebelow
		\rowcolor{tableheadgray}
		Method                        & $64^2$ (native)           & $96^2$ (1.5$\times$)        & $128^2$ (2$\times$)         \\
		\tableheadruleabove
		\midrule
		None                          & 18.59 $\pm$ 3.85          & 18.64 $\pm$ 3.84          & 18.73 $\pm$ 3.84          \\
		LayerNorm                     & 15.51 $\pm$ 1.07          & 15.53 $\pm$ 1.06          & 15.57 $\pm$ 1.05          \\
		QuadNorm                      & 14.88 $\pm$ 1.31          & 14.86 $\pm$ 1.32          & 14.92 $\pm$ 1.33          \\
		BlendQuadNorm                 & \textbf{13.96 $\pm$ 1.31} & \textbf{14.00 $\pm$ 1.29} & \textbf{14.06 $\pm$ 1.29} \\
		BlendQuadNorm, $\alpha = 0.1$ & \textbf{13.96 $\pm$ 1.31} & 14.01 $\pm$ 1.30          & 14.07 $\pm$ 1.30          \\
		\bottomrule
	\end{tabular}
\end{table}

Table~\ref{tab:helmholtz-galerkin-strong} shows that at a 2$\times$ gap, BlendQuadNorm reaches $14.1\%$, compared with LayerNorm at $15.6\%$. This corresponds to a $9.7\%$ improvement with a 95\% CI of [6.8\%, 12.6\%]. This improvement is not a weak-baseline artifact. Table~\ref{tab:poisson-dirichlet-galerkin-strong} shows the same pattern for Poisson--Dirichlet. At a 2$\times$ gap, BlendQuadNorm remains the best for transfer, with $32.1\%$ compared with LayerNorm at $34.7\%$ and a $7.3\%$ improvement with a 95\% CI of [3.1\%, 11.2\%], while QuadNorm is slightly best natively but degrades at the 2$\times$ gap.

\begin{table}[t]
	\centering
	\caption{Poisson--Dirichlet with Galerkin Transformer using width 48 and 300 epochs. Training uses $64^2$, and testing covers $64^2$, $96^2$, and $128^2$. The results report relative $L^2$ error in percent over 5 seeds. Bold indicates the best method in each column.}
	\label{tab:poisson-dirichlet-galerkin-strong}
	\begin{tabular}{lccc}
		\toprule
		\tableheadrulebelow
		\rowcolor{tableheadgray}
		Method                        & $64^2$ (native)           & $96^2$ (1.5$\times$)        & $128^2$ (2$\times$)         \\
		\tableheadruleabove
		\midrule
		None                          & 38.98 $\pm$ 3.83          & 39.19 $\pm$ 3.74          & 39.42 $\pm$ 3.62          \\
		LayerNorm                     & 34.63 $\pm$ 2.92          & 34.64 $\pm$ 2.92          & 34.68 $\pm$ 2.91          \\
		QuadNorm                      & \textbf{31.63 $\pm$ 3.04} & \textbf{31.95 $\pm$ 3.02} & 32.35 $\pm$ 3.07          \\
		BlendQuadNorm                 & 31.91 $\pm$ 1.44          & 32.03 $\pm$ 1.36          & 32.14 $\pm$ 1.32          \\
		BlendQuadNorm, $\alpha = 0.1$ & 31.92 $\pm$ 1.67          & 32.00 $\pm$ 1.67          & \textbf{32.11 $\pm$ 1.65} \\
		\bottomrule
	\end{tabular}
\end{table}

\paragraph{Real Navier--Stokes Galerkin Transfer.}\phantomsection\label{app:exp-real-ns-galerkin} Table~\ref{tab:real-ns-galerkin} tests whether Galerkin, as a non-FFT architecture on real periodic Navier--Stokes data, avoids the catastrophic QuadNorm failure observed for FNO in the appendix study of FNO on periodic Navier--Stokes.

\begin{table}[t]
	\centering
	\caption{Real Navier--Stokes one-step prediction with Galerkin Transformer using width 48 and a dataset with 8000 training examples and 2000 test examples. This setting tests a non-FFT architecture on real periodic Navier--Stokes data. Results report relative $L^2$ error in percent as mean $\pm$ 95\% CI over 5 seeds. Bold indicates the best method in each column.}
	\label{tab:real-ns-galerkin}
	\begin{tabular}{lcc}
		\toprule
		\tableheadrulebelow
		\rowcolor{tableheadgray}
		Method                        & $64^2$ (native)          & $128^2$ (2$\times$)        \\
		\tableheadruleabove
		\midrule
		None                          & 5.78 $\pm$ 0.34          & 6.36 $\pm$ 0.52          \\
		LayerNorm                     & 3.90 $\pm$ 0.31          & 4.25 $\pm$ 0.44          \\
		QuadNorm                      & \textbf{3.52 $\pm$ 0.23} & 4.09 $\pm$ 0.30          \\
		BlendQuadNorm                 & 3.57 $\pm$ 0.14          & \textbf{3.89 $\pm$ 0.24} \\
		BlendQuadNorm, $\alpha = 0.1$ & 3.60 $\pm$ 0.15          & 3.98 $\pm$ 0.30          \\
		\bottomrule
	\end{tabular}
\end{table}

\subsection{Quadrature Rule Comparison}
\label{app:exp-quadrature-ablation}

\begin{table}[t]
	\centering
	\caption{Quadrature rule ablation on Darcy under cross-resolution transfer, reporting relative $L^2$ error in percent as mean $\pm$ 95\% CI over 5 seeds. The table compares trapezoidal and higher-order Newton--Cotes variants, namely Simpson and Boole, whose nominal orders are $O(h^4)$ and $O(h^6)$ when their composite compatibility conditions hold, together with trapezoidal BlendQuadNorm as a reference.}
	\label{tab:quadrature-ablation}
	\begin{tabular}{lccc}
		\toprule
		\tableheadrulebelow
		\rowcolor{tableheadgray}
		Quadrature Rule   & $64^2$          & $128^2$         & $256^2$         \\
		\tableheadruleabove
		\midrule
		QuadNorm          & 3.84 $\pm$ 0.09 & 4.19 $\pm$ 0.07 & 4.51 $\pm$ 0.13 \\
		QuadNorm, Simpson & 3.83 $\pm$ 0.06 & 4.19 $\pm$ 0.11 & 4.50 $\pm$ 0.18 \\
		QuadNorm, Boole   & 3.83 $\pm$ 0.06 & 4.19 $\pm$ 0.11 & 4.50 $\pm$ 0.18 \\
		BlendQuadNorm     & 2.64 $\pm$ 0.14 & 4.15 $\pm$ 0.17 & 5.35 $\pm$ 0.26 \\
		\bottomrule
	\end{tabular}
\end{table}

Table~\ref{tab:quadrature-ablation} shows that all three rules achieve essentially identical performance. Native errors range from $3.83\%$ to $3.84\%$, and cross-resolution differences are within $0.01$ percentage points. The simple trapezoidal rule suffices.

\subsection{Learnable Baselines}
\label{app:exp-learnable-baselines}

\begin{table}[t]
	\centering
	\caption{Learnable baselines on Darcy under cross-resolution transfer, reporting relative $L^2$ error in percent as mean $\pm$ 95\% CI over 5 to 10 seeds. Fixed BlendQuadNorm with $\alpha = 0.3$ is competitive with learned alternatives. Bold indicates the best method in each column.}
	\label{tab:learnable-baselines}
	\begin{tabular}{lccc}
		\toprule
		\tableheadrulebelow
		\rowcolor{tableheadgray}
		Normalization                       & $64^2$ (native)          & $128^2$ (transfer)       & $256^2$ (transfer)       \\
		\tableheadruleabove
		\midrule
		LayerNorm                           & 2.65 $\pm$ 0.04          & 4.13 $\pm$ 0.05          & 5.34 $\pm$ 0.13          \\
		BlendQuadNorm, fixed $\alpha = 0.3$ & 2.65 $\pm$ 0.05          & \textbf{4.03 $\pm$ 0.08} & \textbf{5.14 $\pm$ 0.13} \\
		BlendQuadNorm, learned $\alpha$     & 2.64 $\pm$ 0.14          & 4.14 $\pm$ 0.15          & 5.34 $\pm$ 0.23          \\
		ResConditioned                      & 2.67 $\pm$ 0.14          & 5.10 $\pm$ 0.64          & 7.80 $\pm$ 1.24          \\
		LearnableQuad                       & \textbf{2.63 $\pm$ 0.14} & 4.18 $\pm$ 0.17          & 5.41 $\pm$ 0.25          \\
		\bottomrule
	\end{tabular}
\end{table}

Table~\ref{tab:learnable-baselines} shows that learned $\alpha$ converges near $0.3$ to $0.4$, and fixed $\alpha = 0.3$ outperforms it. This comparison between learned and fixed blending supports the simpler choice.

\paragraph{Effect Size Analysis.}\phantomsection\label{app:exp-effect-sizes} Table~\ref{tab:effect-sizes} summarizes Cohen's $d$ effect sizes for BlendQuadNorm and QuadNorm relative to LayerNorm across representative cross-resolution experiments. For BlendQuadNorm, effect sizes are modest on standard Darcy transfer, with $d = 0.56$ at a 2$\times$ gap and $d = 0.14$ at a 4$\times$ gap, but they become large at the extreme 8$\times$ gap where $d = 1.02$. In settings that favor QuadNorm, the effect magnitudes are often much larger, ranging from $d = 4.62$ to $d = 9.62$.

\begin{table}[t]
	\centering
	\caption{Cohen's $d$ effect sizes for BlendQuadNorm relative to LayerNorm across cross-resolution experiments. Positive $d$ favors the method. Magnitudes below 0.2 are negligible; magnitudes from 0.2 to 0.5 are small; magnitudes from 0.5 to 0.8 are medium; and values above 0.8 are large.}
	\label{tab:effect-sizes}
	\begin{tabular}{lcccc}
		\toprule
		\tableheadrulebelow
		\rowcolor{tableheadgray}
		Setting                                & BlendQuadNorm $d$ & $p$    & QuadNorm $d$ & $p$    \\
		\tableheadruleabove
		\midrule
		Extreme $32^2 \to 256^2$               & 1.02              & <0.001 & 9.62         & <0.001 \\
		Deep 8L $64^2 \to 256^2$               & 0.30              & 0.22   & 8.40         & <0.001 \\
		Extended baselines $64^2 \to 256^2$    & 1.11              & <0.001 & 4.62         & <0.001 \\
		High-to-Low $128^2 \to 32^2$           & 0.53              & 0.034  & 13.99        & <0.001 \\
		High-to-Low $256^2 \to 64^2$           & 0.81              & 0.023  & 7.63         & <0.001 \\
		Multi-resolution training $64^2 \to 256^2$ & 0.86              & <0.001 & 7.46         & <0.001 \\
		\bottomrule
	\end{tabular}
\end{table}

\subsection{High-to-Low Resolution Transfer}
\label{app:exp-high-to-low}

\begin{table}[t]
	\centering
	\caption{High-to-Low resolution transfer on Darcy flow, reporting relative $L^2$ error in percent as mean $\pm$ 95\% CI over 5 seeds. Models are trained at $128^2$ or $256^2$ and tested at lower resolutions. Bold indicates the best method in each column of each subtable.}
	\label{tab:high-to-low}
	\begin{tabular}{lccc}
		\toprule
		\tableheadrulebelow
		\rowcolor{tableheadgray}
		Train 128$^2$ & $128^2$ (native)         & $64^2$ (transfer)        & $32^2$ (transfer)        \\
		\tableheadruleabove
		\midrule
		None          & 3.34 $\pm$ 0.12          & 4.72 $\pm$ 0.11          & 11.06 $\pm$ 0.89         \\
		LayerNorm     & \textbf{2.87 $\pm$ 0.14} & 4.25 $\pm$ 0.13          & 10.16 $\pm$ 0.49         \\
		InstanceNorm  & 4.10 $\pm$ 0.07          & 5.14 $\pm$ 0.06          & 10.14 $\pm$ 0.15         \\
		GroupNorm     & 2.95 $\pm$ 0.18          & 4.32 $\pm$ 0.17          & 10.11 $\pm$ 0.42         \\
		RMSNorm       & 3.38 $\pm$ 0.13          & 4.58 $\pm$ 0.15          & 9.90 $\pm$ 0.19          \\
		QuadNorm      & 4.12 $\pm$ 0.06          & 4.32 $\pm$ 0.08          & \textbf{5.91 $\pm$ 0.21} \\
		BlendQuadNorm & \textbf{2.87 $\pm$ 0.15} & \textbf{4.16 $\pm$ 0.17} & 9.90 $\pm$ 0.71          \\
		\bottomrule
	\end{tabular}

	\vspace{1em}

	\begin{tabular}{lccc}
		\toprule
		\tableheadrulebelow
		\rowcolor{tableheadgray}
		Train 256$^2$ & $256^2$ (native)         & $128^2$ (transfer)       & $64^2$ (transfer)        \\
		\tableheadruleabove
		\midrule
		None          & 3.40 $\pm$ 0.12          & 3.73 $\pm$ 0.08          & 5.96 $\pm$ 0.23          \\
		LayerNorm     & \textbf{2.91 $\pm$ 0.14} & 3.26 $\pm$ 0.12          & 5.42 $\pm$ 0.17          \\
		InstanceNorm  & 4.16 $\pm$ 0.07          & 4.41 $\pm$ 0.06          & 6.14 $\pm$ 0.06          \\
		GroupNorm     & 2.99 $\pm$ 0.18          & 3.34 $\pm$ 0.17          & 5.48 $\pm$ 0.20          \\
		RMSNorm       & 3.43 $\pm$ 0.13          & 3.73 $\pm$ 0.16          & 5.65 $\pm$ 0.18          \\
		QuadNorm      & 4.17 $\pm$ 0.07          & 4.17 $\pm$ 0.07          & \textbf{4.54 $\pm$ 0.12} \\
		BlendQuadNorm & \textbf{2.91 $\pm$ 0.14} & \textbf{3.22 $\pm$ 0.13} & 5.27 $\pm$ 0.27          \\
		\bottomrule
	\end{tabular}
\end{table}

Table~\ref{tab:high-to-low} shows that at $128^2 \to 32^2$, QuadNorm reaches $5.91\%$, compared with LayerNorm at $10.16\%$, for a $42\%$ improvement, while BlendQuadNorm matches LayerNorm bidirectionally.

\paragraph{No-Regression Guarantee via Equivalence Testing.}\phantomsection\label{app:exp-tost} Table~\ref{tab:tost} reports a TOST equivalence test in a 10-seed native Darcy comparison using FNO and formally establishes native-resolution equivalence between BlendQuadNorm and LayerNorm. The results confirm equivalence for this native Darcy comparison using FNO, with $p < 0.0001$, mean difference $+0.0024\%$, and 90\% CI $[-0.0011\%, +0.0059\%]$.

\begin{table}[t]
	\centering
	\caption{TOST equivalence test for a native Darcy comparison using FNO at $64^2 \to 64^2$ using 10 paired seeds. BlendQuadNorm is statistically equivalent to LayerNorm at native resolution with $p < 0.0001$ and equivalence margin $\delta = \pm 0.5\%$.}
	\label{tab:tost}
	\begin{tabular}{lc}
		\toprule
		\tableheadrulebelow
		\rowcolor{tableheadgray}
		Metric                      & Value               \\
		\tableheadruleabove
		\midrule
		LayerNorm mean              & 2.0944\%            \\
		BlendQuadNorm mean          & 2.0968\%            \\
		Mean difference             & 0.0024\%            \\
		90\% CI of difference       & [-0.0011, 0.0059]\% \\
		Equivalence margin $\delta$ & $\pm$0.5\%          \\
		TOST $p$-value              & $< 0.0001$          \\
		Equivalent at $p < 0.05$    & Yes                 \\
		$n$, paired seeds           & 10                  \\
		\bottomrule
	\end{tabular}
\end{table}

\subsection{Nonperiodic Breadth}
\label{app:exp-galerkin-breadth}

To address the concern that our nonperiodic PDE results are narrowly concentrated on Darcy flow, we evaluate the Galerkin Transformer with width 64 and depth 4 on three additional nonperiodic PDEs with Dirichlet boundary conditions. These are linear elasticity, variable-coefficient diffusion $-\nabla\cdot(a(x)\nabla u)=f$, and reaction-diffusion $-D\nabla^2 u + ru = f$. All settings use spectral and sine-basis data generation with exact or iterative solutions, train at $64^2$, test at $64^2$ and $128^2$, and use 5 seeds, as reported in Tables~\ref{tab:galerkin-elasticity}, \ref{tab:galerkin-varcoeff}, and~\ref{tab:galerkin-rxndiff}.

\begin{table}[t]
	\centering
	\caption{Galerkin Transformer on Elasticity with width 64, depth 4, training at $64^2$, Dirichlet boundary conditions, and 5 seeds. Bold indicates the best method in each column.}
	\label{tab:galerkin-elasticity}
	\begin{tabular}{lcc}
		\toprule
		\tableheadrulebelow
		\rowcolor{tableheadgray}
		Method        & $64^2$ (native)          & $128^2$ (2$\times$)        \\
		\tableheadruleabove
		\midrule
		LayerNorm     & 8.59 $\pm$ 3.43          & 6.49 $\pm$ 1.60          \\
		QuadNorm      & \textbf{7.28 $\pm$ 2.97} & \textbf{5.21 $\pm$ 1.03} \\
		BlendQuadNorm & 7.92 $\pm$ 2.75          & 5.76 $\pm$ 1.43          \\
		\bottomrule
	\end{tabular}
\end{table}

\begin{table}[t]
	\centering
	\caption{Galerkin Transformer on variable-coefficient diffusion with width 64, depth 4, training at $64^2$, Dirichlet boundary conditions, and 5 seeds. The task maps a log-normal coefficient field $a(x)$ to the solution $u(x)$, and the reference solutions are computed with preconditioned conjugate gradient.}
	\label{tab:galerkin-varcoeff}
	\begin{tabular}{lcc}
		\toprule
		\tableheadrulebelow
		\rowcolor{tableheadgray}
		Method        & $64^2$ (native)          & $128^2$ (2$\times$)        \\
		\tableheadruleabove
		\midrule
		LayerNorm     & 4.91 $\pm$ 0.25          & 5.38 $\pm$ 0.15          \\
		QuadNorm      & \textbf{4.04 $\pm$ 0.27} & \textbf{4.82 $\pm$ 0.22} \\
		BlendQuadNorm & 4.44 $\pm$ 0.25          & 5.06 $\pm$ 0.18          \\
		\bottomrule
	\end{tabular}
\end{table}

\begin{table}[t]
	\centering
	\caption{Galerkin Transformer on reaction-diffusion with width 64, depth 4, training at $64^2$, Dirichlet boundary conditions, and 5 seeds. The setup fixes $D=0.05$ and $r=1.0$ and maps the forcing $f$ to the solution $u$. Bold indicates the best method in each column.}
	\label{tab:galerkin-rxndiff}
	\begin{tabular}{lcc}
		\toprule
		\tableheadrulebelow
		\rowcolor{tableheadgray}
		Method        & $64^2$ (native)          & $128^2$ (2$\times$)        \\
		\tableheadruleabove
		\midrule
		LayerNorm     & 5.71 $\pm$ 1.24          & 4.57 $\pm$ 0.50          \\
		QuadNorm      & 5.27 $\pm$ 2.04          & \textbf{3.72 $\pm$ 0.94} \\
		BlendQuadNorm & \textbf{5.26 $\pm$ 1.28} & 3.93 $\pm$ 0.67          \\
		\bottomrule
	\end{tabular}
\end{table}

Tables~\ref{tab:galerkin-elasticity}, \ref{tab:galerkin-varcoeff}, and~\ref{tab:galerkin-rxndiff} show that QuadNorm achieves the best transfer on all three PDEs: a $19.7\%$ improvement on elasticity, a $10.4\%$ improvement on variable-coefficient diffusion, and an $18.6\%$ improvement on reaction-diffusion, extending the Galerkin+Poisson pattern to three additional nonperiodic PDEs.

\paragraph{Transolver on Variable-Coefficient Diffusion.}\label{app:exp-transolver-breadth} We extend the Transolver evaluation to variable-coefficient diffusion with Dirichlet boundary conditions, using width 64, depth 6, and 5 seeds. Table~\ref{tab:transolver-varcoeff} shows that on variable-coefficient diffusion, QuadNorm improves natively by $+21\%$ and $+9\%$ at a 2$\times$ gap.

\begin{table}[t]
	\centering
	\caption{Transolver on variable-coefficient diffusion with width 64, depth 6, training at $64^2$, Dirichlet boundary conditions, and 5 seeds.}
	\label{tab:transolver-varcoeff}
	\begin{tabular}{lcc}
		\toprule
		\tableheadrulebelow
		\rowcolor{tableheadgray}
		Method        & $64^2$ (native)          & $128^2$ (2$\times$)        \\
		\tableheadruleabove
		\midrule
		LayerNorm     & 4.25 $\pm$ 0.30          & 4.77 $\pm$ 0.09          \\
		QuadNorm      & \textbf{3.36 $\pm$ 0.33} & \textbf{4.35 $\pm$ 0.15} \\
		BlendQuadNorm & 3.87 $\pm$ 0.29          & 4.69 $\pm$ 0.11          \\
		\bottomrule
	\end{tabular}
\end{table}

\paragraph{Qualitative variable-coefficient diffusion examples.} Figure~\ref{fig:varcoeff-qualitative} complements Tables~\ref{tab:galerkin-varcoeff} and~\ref{tab:transolver-varcoeff} with representative native-resolution field visualizations on the nonperiodic variable-coefficient diffusion problem. For both Galerkin and Transolver, QuadNorm preserves the interior solution profile more faithfully and yields visibly smaller localized error, showing that the advantage already appears at native resolution on nonspectral, nonperiodic PDEs.

\begin{figure*}[t]
	\centering
	\includegraphics[width=\textwidth]{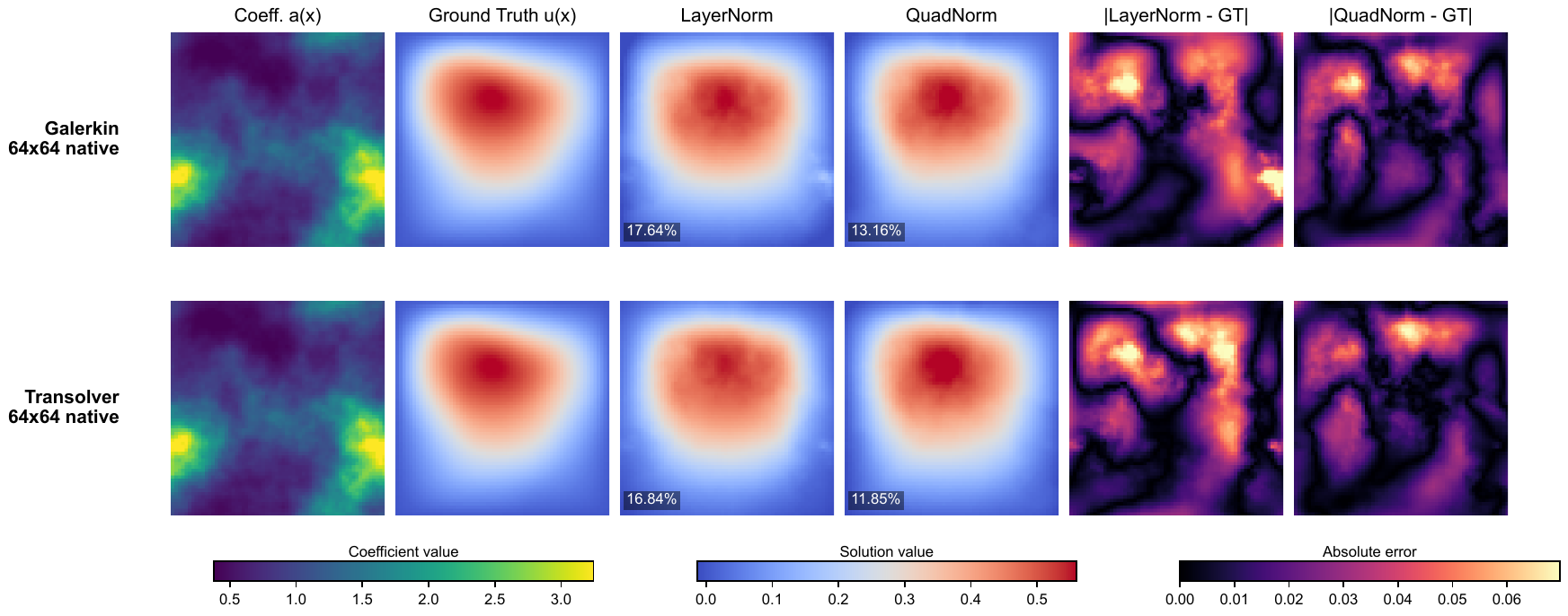}
	\caption{Qualitative predictions on the nonperiodic variable-coefficient diffusion problem at native resolution ($64^2 \to 64^2$). Rows correspond to Galerkin Transformer and Transolver, and columns show the coefficient field $a(x)$, the ground-truth solution $u(x)$, predictions from LayerNorm and QuadNorm, and the corresponding absolute error maps. Across both nonspectral architectures, QuadNorm reduces the spatially structured error and better preserves the native-resolution solution morphology.}
	\label{fig:varcoeff-qualitative}
\end{figure*}

\subsection{Scaling and Validation Details}

	\paragraph{Per-Scale Darcy Scaling Tables.}\phantomsection\label{app:scaling-darcy-tables} Tables~\ref{tab:transolver-darcy-largescale} and~\ref{tab:transolver-darcy-xl} unpack the model-scale trend from Section~\ref{sec:exp-scaling} by reporting the full Darcy error profile at each evaluation resolution. In both the large and XL Transolver settings, QuadNorm remains better than LayerNorm from native resolution to a 4$\times$ gap, indicating that the scaling gain is not driven by a single endpoint.

\begin{table}[t]
	\centering
	\caption{Large-scale Transolver on the Darcy benchmark with width 128, depth 8, and about 1.6M parameters. Training uses $64^2$, and testing spans $64^2$ to $256^2$. The results report relative $L^2$ error in percent as mean $\pm$ 95\% CI over 5 seeds.}
	\label{tab:transolver-darcy-largescale}
	\begin{tabular}{lccc}
		\toprule
		\tableheadrulebelow
		\rowcolor{tableheadgray}
		Method    & $64^2$ (native)          & $128^2$ (2$\times$)        & $256^2$ (4$\times$)        \\
		\tableheadruleabove
		\midrule
		LayerNorm & 5.06 $\pm$ 0.12          & 5.87 $\pm$ 0.09          & 6.70 $\pm$ 0.10          \\
		QuadNorm  & \textbf{4.55 $\pm$ 0.28} & \textbf{4.72 $\pm$ 0.26} & \textbf{4.88 $\pm$ 0.25} \\
		\bottomrule
	\end{tabular}
\end{table}

\begin{table}[t]
	\centering
	\caption{XL Transolver on the Darcy benchmark with width 256, depth 12, and about 9.6M parameters. Training uses $64^2$, and testing spans $64^2$ to $256^2$. The results report relative $L^2$ error in percent as mean $\pm$ 95\% CI over 5 seeds.}
	\label{tab:transolver-darcy-xl}
	\begin{tabular}{lccc}
		\toprule
		\tableheadrulebelow
		\rowcolor{tableheadgray}
		Method    & $64^2$ (native)          & $128^2$ (2$\times$)        & $256^2$ (4$\times$)        \\
		\tableheadruleabove
		\midrule
		LayerNorm & 4.19 $\pm$ 0.10          & 5.22 $\pm$ 0.11          & 6.20 $\pm$ 0.11          \\
		QuadNorm  & \textbf{3.47 $\pm$ 0.18} & \textbf{3.68 $\pm$ 0.18} & \textbf{3.85 $\pm$ 0.17} \\
		\bottomrule
	\end{tabular}
\end{table}

	\paragraph{10-Seed Helmholtz Validation.}\phantomsection\label{app:exp-helmholtz-10seed} Table~\ref{tab:galerkin-helmholtz-10seed} reports a 10-seed evaluation on nonperiodic Helmholtz, showing that the Galerkin improvement is not a small-sample artifact. BlendQuadNorm stays ahead of LayerNorm across all tested scales, reinforcing that nonspectral architectures benefit from this normalization family in this setting.

\begin{table}[t]
	\centering
	\caption{Galerkin Transformer on the nonperiodic Helmholtz problem with 10 seeds. Training uses $64^2$, and testing spans $64^2$ to $128^2$. Results report relative $L^2$ error in percent as mean $\pm$ 95\% CI. BlendQuadNorm achieves a 16\% improvement over LayerNorm with a 95\% CI of [10\%, 21\%].}
	\label{tab:galerkin-helmholtz-10seed}
	\begin{tabular}{lccc}
		\toprule
		\tableheadrulebelow
		\rowcolor{tableheadgray}
		Method        & $64^2$ (native)          & $96^2$ (1.5$\times$)       & $128^2$ (2$\times$)        \\
		\tableheadruleabove
		\midrule
		LayerNorm     & 6.57 $\pm$ 0.68          & 6.56 $\pm$ 0.59          & 6.68 $\pm$ 0.64          \\
		QuadNorm      & 8.43 $\pm$ 0.49          & 8.54 $\pm$ 0.46          & 8.56 $\pm$ 0.38          \\
		BlendQuadNorm & \textbf{5.54 $\pm$ 0.58} & \textbf{5.50 $\pm$ 0.50} & \textbf{5.61 $\pm$ 0.54} \\
		\bottomrule
	\end{tabular}
\end{table}

\paragraph{Per-Scale Helmholtz Scaling Tables.}\phantomsection\label{app:scaling-helmholtz-tables} Tables~\ref{tab:transolver-helmholtz-largescale} and~\ref{tab:transolver-helmholtz-xl} complement the Darcy scaling study on a second nonperiodic benchmark where the normalization choice matters strongly. QuadNorm improves the large Transolver modestly and the XL Transolver substantially, showing that the scale-amplification pattern carries beyond Darcy.

\begin{table}[t]
	\centering
	\caption{Large-scale Transolver on the nonperiodic Helmholtz problem with width 128, depth 8, and about 1.6M parameters. Training uses $64^2$, and testing spans $64^2$ to $128^2$. The results report relative $L^2$ error in percent as mean $\pm$ 95\% CI over 5 seeds.}
	\label{tab:transolver-helmholtz-largescale}
	\begin{tabular}{lccc}
		\toprule
		\tableheadrulebelow
		\rowcolor{tableheadgray}
		Method    & $64^2$ (native)          & $96^2$ (1.5$\times$)       & $128^2$ (2$\times$)        \\
		\tableheadruleabove
		\midrule
		LayerNorm & 8.45 $\pm$ 0.84          & 8.49 $\pm$ 0.85          & 8.46 $\pm$ 0.66          \\
		QuadNorm  & \textbf{7.62 $\pm$ 0.25} & \textbf{7.56 $\pm$ 0.32} & \textbf{7.83 $\pm$ 0.35} \\
		\bottomrule
	\end{tabular}
\end{table}

\begin{table}[t]
	\centering
	\caption{XL Transolver on the nonperiodic Helmholtz problem with width 256, depth 12, and about 9.6M parameters. Training uses $64^2$, and testing spans $64^2$ to $128^2$. Results report relative $L^2$ error in percent as mean $\pm$ standard deviation over 5 seeds. QuadNorm achieves about a 51\% improvement at native resolution.}
	\label{tab:transolver-helmholtz-xl}
	\begin{tabular}{lccc}
		\toprule
		\tableheadrulebelow
		\rowcolor{tableheadgray}
		Method    & $64^2$ (native)          & $96^2$ (1.5$\times$)       & $128^2$ (2$\times$)        \\
		\tableheadruleabove
		\midrule
		LayerNorm & 5.49 $\pm$ 0.88          & 5.52 $\pm$ 0.90          & 5.55 $\pm$ 0.79          \\
		QuadNorm  & \textbf{2.68 $\pm$ 0.19} & \textbf{2.76 $\pm$ 0.19} & \textbf{3.05 $\pm$ 0.18} \\
		\bottomrule
	\end{tabular}
\end{table}

	\paragraph{Nonuniform grids.}\label{sec:exp-nonuniform} On nonuniform grids, each weight $w_i$ is set to the cell area associated with grid point $i$ or, in higher dimensions, to the corresponding cell volume. These weights can be obtained, \textit{e.g.}, via Voronoi tessellation or the metric tensor, following the control-volume viewpoint standard in finite-volume discretizations \citep{eymard2000finitevolume, leveque2002finitevolume}. This control-volume weighting scheme makes QuadNorm applicable to arbitrary meshes once the corresponding weights are available. In many PDE and simulation settings, these weights are available from the discretization itself or can be precomputed once from the mesh geometry. Figure~\ref{fig:bias-convergence} shows how the statistic bias grows with mesh nonuniformity and how quadrature weighting suppresses it across mesh families.

\begin{figure}[t]
	\centering
	\includegraphics[width=\textwidth]{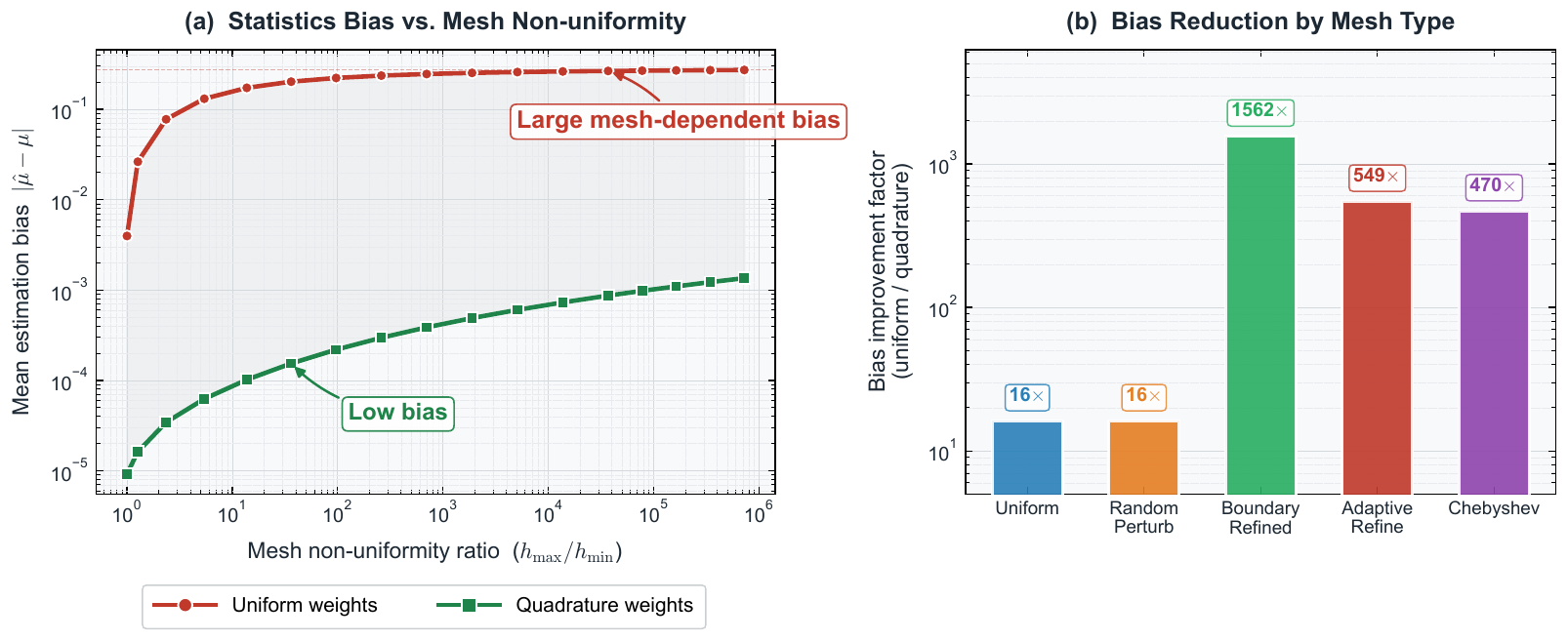}
	\caption{The first plot shows statistic bias against the mesh nonuniformity ratio for a fixed-resolution sweep over a smooth boundary-refined mesh family. Uniform weights shown in red exhibit substantial mesh-dependent bias as the mesh becomes more skewed, whereas quadrature weights shown in green remain far smaller across the full range. The second plot shows the bias reduction factor by mesh type. On highly nonuniform grids such as boundary-refined and Chebyshev meshes, the improvement exceeds 200 times.}
	\label{fig:bias-convergence}
\end{figure}

\paragraph{Nonuniform Grid Bias.} Figure~\ref{fig:nonuniform-bias} isolates the nonuniform-mesh failure mode motivating QuadNorm. It shows that uniform point averaging can produce large mesh-dependent statistic bias, whereas quadrature or control-volume weights sharply reduce that bias across the tested mesh families.

\begin{figure}[t]
	\centering
	\includegraphics[width=\textwidth]{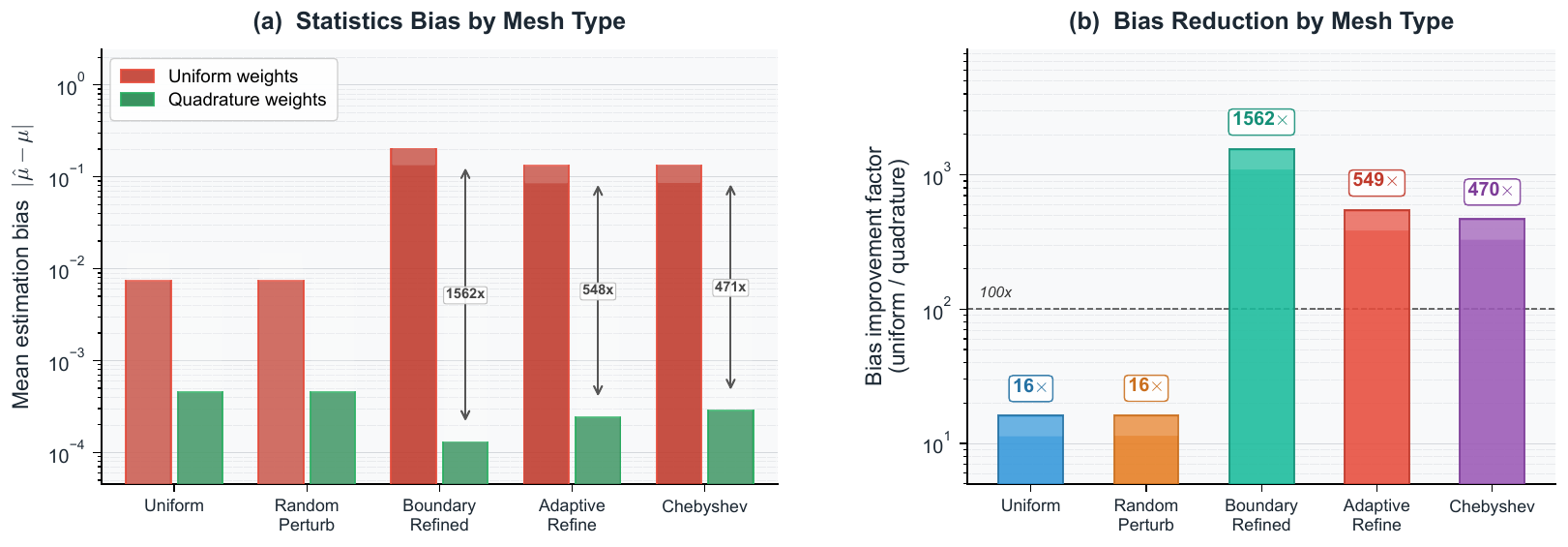}
	\caption{Nonuniform mesh bias analysis. The first plot shows statistic bias by mesh type, comparing uniform weights in red with quadrature weights in green. On highly nonuniform grids, the gap exceeds three orders of magnitude. The second plot shows the bias improvement factor: quadrature weights reduce mean estimation bias by more than 1500 times on boundary-refined grids.}
	\label{fig:nonuniform-bias}
\end{figure}

\paragraph{Qualitative boundary-refined mesh mechanism.} Figure~\ref{fig:nonuniform-mechanism} turns the quantitative bias result in Figure~\ref{fig:nonuniform-bias} into a single concrete example. Rather than showing end-to-end PDE prediction, it isolates the normalization mechanism on a deterministic PDE-like scalar field sampled on the same boundary-refined mesh family used in the nonuniform-grid study. The key point is that the mesh places many samples in a thin boundary strip, but each of those samples represents only a tiny physical area. Because the field is small near the boundary and large in the interior, uniform point averaging counts too many low-value boundary samples and too little of the high-value interior, whereas quadrature or control-volume weighting restores the correct area contribution.

\begin{figure*}[t]
	\centering
	\includegraphics[width=\textwidth]{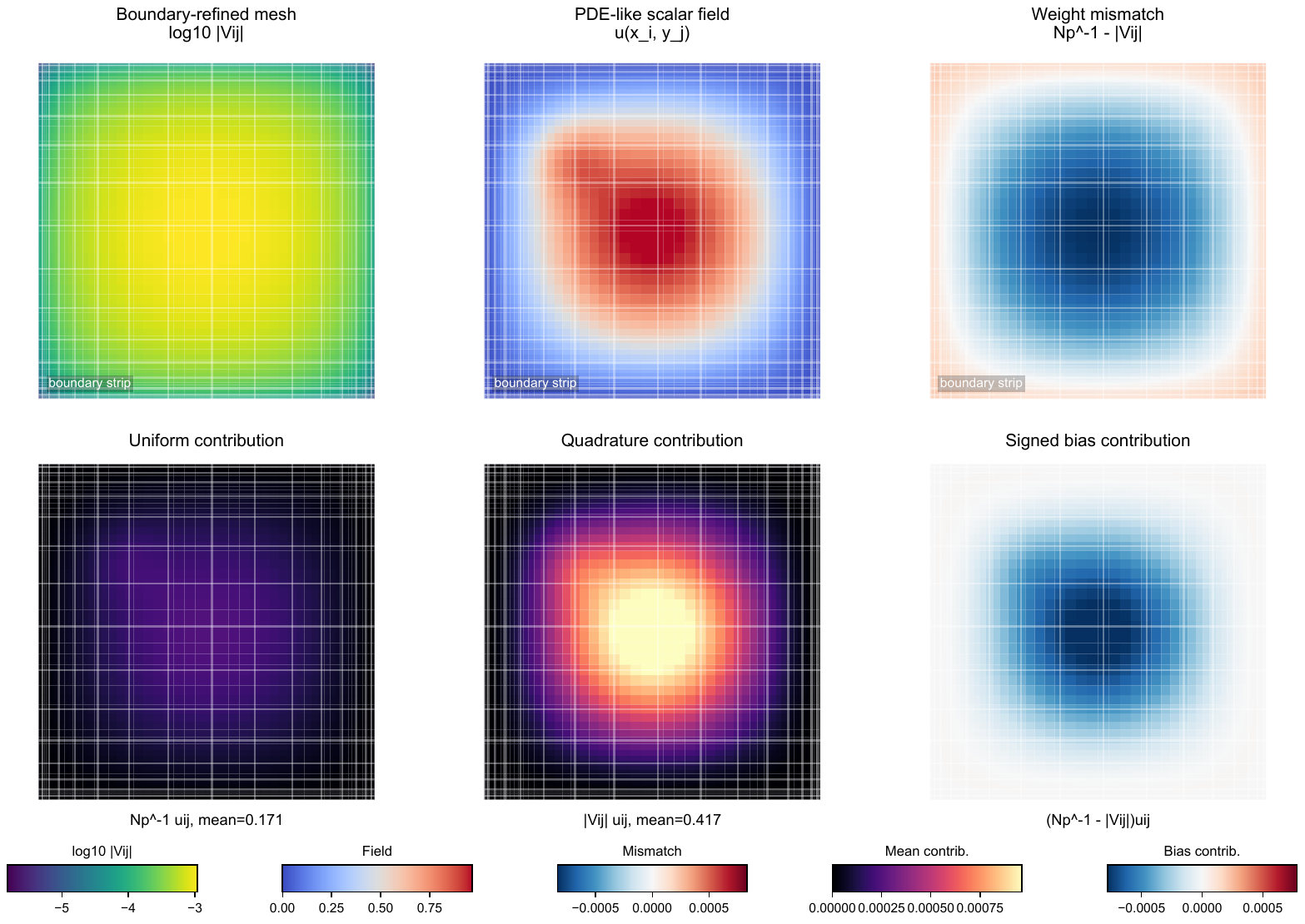}
	\caption{Mechanism visualization on a boundary-refined mesh at $64^2$. Read the panels from left to right. The top row shows that boundary points have very small control volumes, that the scalar field is weakest in the boundary strip and strongest in the interior, and that the weight mismatch $N_p^{-1} - |V_{ij}|$ gives too much weight to boundary points and too little to interior points under uniform averaging. The bottom row converts those weights into local contributions to the estimated mean for uniform averaging, quadrature weighting, and their signed difference. As a result, the uniform estimate is pulled down to $0.171$ because it overcounts many low-value boundary samples, while the quadrature estimate recovers the reference area average $0.417$ by weighting each sample according to the area it represents.}
	\label{fig:nonuniform-mechanism}
\end{figure*}

\paragraph{Extended Baseline Comparison.}\label{sec:exp-extended-baselines} Table~\ref{tab:extended-baselines} tests whether the observed gains are specific to classical normalization baselines or persist against broader alternatives. The results show that spectral and band-based normalization baselines are highly unstable under resolution transfer, reinforcing that discretization-consistent activation statistics are the relevant design axis here.

\begin{table}[t]
	\centering
	\caption{Extended baseline comparison on Darcy under cross-resolution transfer from $64$ to $128$ and $256$. The table includes SpectralNorm, RI-BandNorm, and QuadBandNorm alongside the standard normalization baselines.}
	\label{tab:extended-baselines}
	\begin{tabular}{lccc}
		\toprule
		\tableheadrulebelow
		\rowcolor{tableheadgray}
		Normalization & $64^2$                   & $128^2$                  & $256^2$                  \\
		\tableheadruleabove
		\midrule
		None          & 3.02 $\pm$ 0.09          & 4.73 $\pm$ 0.08          & 6.12 $\pm$ 0.11          \\
		LayerNorm     & \textbf{2.65 $\pm$ 0.06} & \textbf{4.14 $\pm$ 0.09} & 5.34 $\pm$ 0.13          \\
		InstanceNorm  & 3.82 $\pm$ 0.06          & 5.02 $\pm$ 0.05          & 6.10 $\pm$ 0.04          \\
		GroupNorm     & \textbf{2.76 $\pm$ 0.07} & 4.25 $\pm$ 0.08          & 5.47 $\pm$ 0.11          \\
		RMSNorm       & 3.07 $\pm$ 0.08          & 4.48 $\pm$ 0.06          & 5.68 $\pm$ 0.08          \\
		QuadNorm      & 3.86 $\pm$ 0.06          & 4.25 $\pm$ 0.08          & \textbf{4.59 $\pm$ 0.10} \\
		BlendQuadNorm & \textbf{2.65 $\pm$ 0.05} & \textbf{4.03 $\pm$ 0.08} & 5.14 $\pm$ 0.13          \\
		SpectralNorm  & 12.43 $\pm$ 0.10         & 78.91 $\pm$ 1.14         & 98.22 $\pm$ 0.52         \\
		RI-BandNorm   & 38.46 $\pm$ 0.23         & 89.53 $\pm$ 5.28         & 97.40 $\pm$ 2.16         \\
		QuadBandNorm  & 38.31 $\pm$ 1.07         & 89.31 $\pm$ 9.46         & 98.48 $\pm$ 3.64         \\
		\bottomrule
	\end{tabular}
\end{table}

\paragraph{Consistent-Kernel Combination.} Table~\ref{tab:consistent-combo} checks whether normalization consistency remains useful once the operator kernel is also made discretization-consistent. The combined results show additive gains, supporting our broader claim that kernel mismatch and normalization mismatch are complementary error sources rather than redundant fixes.

\begin{table}[t]
	\centering
	\caption{Combining the quadrature normalization family with discretization-consistent kernels in the consistent FNO (cFNO) setting. Additive improvements confirm orthogonal error sources. Results use 5 seeds.}
	\label{tab:consistent-combo}
	\begin{tabular}{lccc}
		\toprule
		\tableheadrulebelow
		\rowcolor{tableheadgray}
		Configuration        & $64^2$ (native)          & $128^2$ (2$\times$)        & $256^2$ (4$\times$)        \\
		\tableheadruleabove
		\midrule
		\multicolumn{4}{l}{Standard FNO kernel}                                                               \\
		FNO + LayerNorm      & \textbf{2.92 $\pm$ 0.15} & 4.42 $\pm$ 0.13          & 5.64 $\pm$ 0.17          \\
		FNO + QuadNorm       & 4.21 $\pm$ 0.07          & 4.55 $\pm$ 0.14          & 4.85 $\pm$ 0.22          \\
		FNO + BlendQuadNorm  & \textbf{2.92 $\pm$ 0.15} & \textbf{4.35 $\pm$ 0.17} & 5.51 $\pm$ 0.25          \\
		\midrule
		\multicolumn{4}{l}{Consistent FNO kernel from \citet{gao2025discretization}}                          \\
		cFNO + LayerNorm     & \textbf{2.88 $\pm$ 0.10} & 4.36 $\pm$ 0.08          & 5.56 $\pm$ 0.17          \\
		cFNO + QuadNorm      & 3.85 $\pm$ 0.12          & \textbf{4.21 $\pm$ 0.14} & \textbf{4.53 $\pm$ 0.18} \\
		cFNO + BlendQuadNorm & \textbf{2.88 $\pm$ 0.10} & \textbf{4.19 $\pm$ 0.09} & 5.27 $\pm$ 0.15          \\
		\bottomrule
	\end{tabular}
\end{table}

\paragraph{Nonperiodic Breadth Summary.}\label{app:exp-breadth-summary} Table~\ref{tab:nonperiodic-breadth-summary} summarizes the nonperiodic breadth experiments. On Galerkin, these settings improve by roughly $10\%$ to $20\%$. On Transolver, the variable-coefficient diffusion task also improves.

\begin{table}[t]
	\centering
	\caption{Summary of nonperiodic breadth experiments across elasticity, reaction-diffusion, and variable-coefficient diffusion. All experiments train at $64^2$ and test at $64^2$ and $128^2$.}
	\label{tab:nonperiodic-breadth-summary}
	\resizebox{\textwidth}{!}{%
		\begin{tabular}{llcccc}
			\toprule
			\tableheadrulebelow
			\rowcolor{tableheadgray}
			Architecture & PDE                            & LayerNorm at native & LayerNorm at 2$\times$ gap & Best of QuadNorm/BlendQuadNorm at 2$\times$ gap & Improvement (\%) \\
			\tableheadruleabove
			\midrule
			Galerkin     & Elasticity                     & 8.59\%              & 6.49\%               & \textbf{5.21\%}, QuadNorm                 & \textbf{19.7}    \\
			Galerkin     & Variable-coefficient diffusion & 4.91\%              & 5.38\%               & \textbf{4.82\%}, QuadNorm                 & \textbf{10.4}    \\
			Galerkin     & Reaction-diffusion             & 5.71\%              & 4.57\%               & \textbf{3.72\%}, QuadNorm                 & \textbf{18.6}    \\
			Transolver   & Variable-coefficient diffusion & 4.25\%              & 4.77\%               & \textbf{4.35\%}, QuadNorm                 & \textbf{8.8}     \\
			\bottomrule
		\end{tabular}
	}
\end{table}

\paragraph{Bootstrap CIs for Nonperiodic Breadth.}\phantomsection\label{app:exp-bootstrap-breadth} Table~\ref{tab:bootstrap-cis-breadth} separates the breadth results that are clearly above zero from the settings where seed noise remains comparable to the mean gain. The strongest support remains for Galerkin elasticity, Galerkin variable-coefficient diffusion, and Transolver variable-coefficient diffusion, matching our emphasis on nonspectral, nonperiodic operators.

\begin{table}[t]
	\centering
	\caption{Bootstrap 95\% CIs for the nonperiodic breadth experiments, using 10,000 paired resamples.}
	\label{tab:bootstrap-cis-breadth}
	\small
	\resizebox{\textwidth}{!}{%
		\begin{tabular}{llllrr}
			\toprule
			\tableheadrulebelow
			\rowcolor{tableheadgray}
			Architecture & PDE                            & Method        & Gap          & Improvement (\%) & 95\% CI      \\
			\tableheadruleabove
			\midrule
			Galerkin     & Elasticity                     & QuadNorm      & $64^2 \to 128^2$ & \textbf{19.4}    & [11.0, 29.7] \\
			Galerkin     & Elasticity                     & BlendQuadNorm & $64^2 \to 128^2$ & 11.0             & [-4.1, 25.7] \\
			Galerkin     & Elasticity                     & QuadNorm      & $64^2 \to 64^2$  & \textbf{15.3}    & [9.3, 20.0]  \\
			Galerkin     & Elasticity                     & BlendQuadNorm & $64^2 \to 64^2$  & 7.6              & [-3.3, 18.0] \\
			Galerkin     & Reaction-diffusion             & QuadNorm      & $64^2 \to 128^2$ & \textbf{18.4}    & [2.5, 31.1]  \\
			Galerkin     & Reaction-diffusion             & BlendQuadNorm & $64^2 \to 128^2$ & \textbf{13.9}    & [3.7, 23.8]  \\
			Galerkin     & Reaction-diffusion             & QuadNorm      & $64^2 \to 64^2$  & 8.0              & [-4.1, 19.2] \\
			Galerkin     & Reaction-diffusion             & BlendQuadNorm & $64^2 \to 64^2$  & \textbf{8.0}     & [5.4, 9.8]   \\
			Galerkin     & Variable-coefficient diffusion & QuadNorm      & $64^2 \to 128^2$ & \textbf{10.4}    & [8.0, 12.8]  \\
			Galerkin     & Variable-coefficient diffusion & BlendQuadNorm & $64^2 \to 128^2$ & \textbf{6.0}     & [3.7, 8.0]   \\
			Galerkin     & Variable-coefficient diffusion & QuadNorm      & $64^2 \to 64^2$  & \textbf{17.9}    & [16.5, 19.1] \\
			Galerkin     & Variable-coefficient diffusion & BlendQuadNorm & $64^2 \to 64^2$  & \textbf{9.5}     & [7.6, 11.7]  \\
			\midrule
			Transolver   & Variable-coefficient diffusion & QuadNorm      & $64^2 \to 128^2$ & \textbf{8.9}     & [7.1, 10.6]  \\
			Transolver   & Variable-coefficient diffusion & BlendQuadNorm & $64^2 \to 128^2$ & 1.6              & [-0.4, 3.7]  \\
			Transolver   & Variable-coefficient diffusion & QuadNorm      & $64^2 \to 64^2$  & \textbf{21.1}    & [18.6, 23.2] \\
			Transolver   & Variable-coefficient diffusion & BlendQuadNorm & $64^2 \to 64^2$  & \textbf{9.1}     & [7.0, 11.1]  \\
			\bottomrule
		\end{tabular}
	}
\end{table}

\subsection{Non-Darcy Extreme-Gap Results}
\label{app:exp-extreme-gap}

We extend the 4$\times$ gap evaluation to non-Darcy PDE settings using Galerkin and Transolver, training at $64^2$ and testing at $256^2$.

\paragraph{Galerkin on Helmholtz at a 4$\times$ gap.} BlendQuadNorm achieves $6.21\%$, compared with LayerNorm at $6.82\%$, for a $+8.9\%$ improvement, while QuadNorm underperforms at native resolution, as shown in Table~\ref{tab:galerkin-helmholtz-4x}.

\begin{table}[t]
	\centering
	\caption{Galerkin Transformer on Helmholtz at a 4$\times$ gap with training at $64^2$ and 5 seeds. Bold indicates the best method in each column.}
	\label{tab:galerkin-helmholtz-4x}
	\begin{tabular}{lccc}
		\toprule
		\tableheadrulebelow
		\rowcolor{tableheadgray}
		Method        & $64^2$ (native)          & $128^2$ (2$\times$)        & $256^2$ (4$\times$)        \\
		\tableheadruleabove
		\midrule
		LayerNorm     & 8.41 $\pm$ 1.68          & 6.69 $\pm$ 1.26          & 6.82 $\pm$ 1.23          \\
		QuadNorm      & 10.30 $\pm$ 0.44         & 6.30 $\pm$ 0.55          & 6.63 $\pm$ 0.54          \\
		BlendQuadNorm & \textbf{7.34 $\pm$ 1.51} & \textbf{6.04 $\pm$ 1.31} & \textbf{6.21 $\pm$ 1.29} \\
		\bottomrule
	\end{tabular}
\end{table}

\paragraph{Galerkin on Elasticity at a 4$\times$ gap.} Table~\ref{tab:galerkin-elasticity-4x} shows that QuadNorm achieves $5.64\%$, compared with LayerNorm at $7.03\%$, for a $+19.8\%$ improvement. Dirichlet boundary conditions amplify the boundary weight mismatch.

\begin{table}[t]
	\centering
	\caption{Galerkin Transformer on Elasticity at a 4$\times$ gap with training at $64^2$, 5 seeds, and Dirichlet boundary conditions. Bold indicates the best method in each column.}
	\label{tab:galerkin-elasticity-4x}
	\begin{tabular}{lccc}
		\toprule
		\tableheadrulebelow
		\rowcolor{tableheadgray}
		Method        & $64^2$ (native)          & $128^2$ (2$\times$)        & $256^2$ (4$\times$)        \\
		\tableheadruleabove
		\midrule
		LayerNorm     & 8.65 $\pm$ 2.21          & 6.73 $\pm$ 1.82          & 7.03 $\pm$ 1.73          \\
		QuadNorm      & \textbf{7.04 $\pm$ 3.23} & \textbf{5.02 $\pm$ 0.83} & \textbf{5.64 $\pm$ 0.78} \\
		BlendQuadNorm & 8.35 $\pm$ 3.49          & 6.04 $\pm$ 1.91          & 6.36 $\pm$ 1.83          \\
		\bottomrule
	\end{tabular}
\end{table}

\paragraph{Qualitative elasticity examples at a 4$\times$ gap.} Figure~\ref{fig:elasticity-4x-qualitative} complements Table~\ref{tab:galerkin-elasticity-4x} with a representative Galerkin transfer example from $64^2$ to $256^2$ on the nonperiodic elasticity problem. QuadNorm reduces the high-resolution transfer error most clearly near the Dirichlet boundary region, which matches the mechanism's claim that boundary weight mismatch becomes more pronounced as the resolution gap widens.

\begin{figure*}[t]
	\centering
	\includegraphics[width=\textwidth]{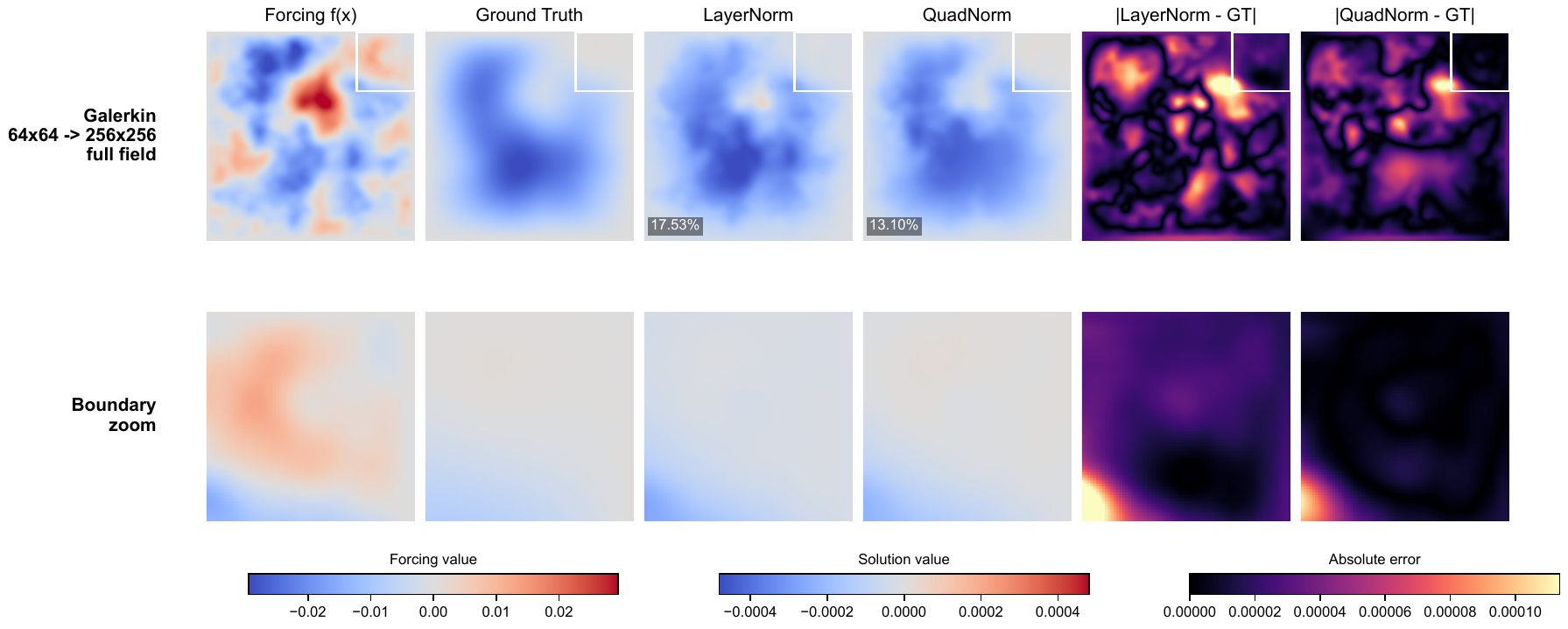}
	\caption{Qualitative prediction on the nonperiodic elasticity problem at a 4$\times$ gap ($64^2 \to 256^2$) using Galerkin Transformer. The top row shows the full forcing field, the ground-truth solution, predictions from LayerNorm and QuadNorm, and the corresponding absolute error maps; the white box marks the automatically selected boundary crop. The bottom row shows the corresponding boundary zoom. QuadNorm suppresses the boundary-aligned transfer error more effectively than LayerNorm while preserving the interior structure.}
	\label{fig:elasticity-4x-qualitative}
\end{figure*}

\paragraph{Transolver on Helmholtz at a 4$\times$ gap.} Table~\ref{tab:transolver-helmholtz-4x} shows that QuadNorm achieves $6.08\%$, compared with LayerNorm at $6.56\%$, for a $+7.3\%$ improvement, while BlendQuadNorm does not improve. This 4$\times$ Helmholtz result is consistent with Transolver favoring QuadNorm.

\begin{table}[t]
	\centering
	\caption{Transolver on Helmholtz at a 4$\times$ gap with training at $64^2$ and 5 seeds. Bold indicates the best method in each column.}
	\label{tab:transolver-helmholtz-4x}
	\begin{tabular}{lccc}
		\toprule
		\tableheadrulebelow
		\rowcolor{tableheadgray}
		Method        & $64^2$ (native)          & $128^2$ (2$\times$)        & $256^2$ (4$\times$)        \\
		\tableheadruleabove
		\midrule
		LayerNorm     & 9.33 $\pm$ 0.51          & 6.39 $\pm$ 0.36          & 6.56 $\pm$ 0.38          \\
		QuadNorm      & \textbf{8.65 $\pm$ 0.94} & \textbf{5.69 $\pm$ 0.45} & \textbf{6.08 $\pm$ 0.41} \\
		BlendQuadNorm & 9.02 $\pm$ 0.68          & 6.40 $\pm$ 0.41          & 6.54 $\pm$ 0.38          \\
		\bottomrule
	\end{tabular}
\end{table}

\paragraph{Extreme gap summary.} Table~\ref{tab:extreme-gap-summary} summarizes the 4$\times$ gap settings. Elasticity improves by $19.8\%$, and Helmholtz improves by 7\% to 9\%, confirming that gap-scaling generalizes beyond Darcy on nonspectral architectures.

\begin{table}[t]
	\centering
	\caption{Summary of 4$\times$ gap experiments on non-Darcy PDEs. The resolution-gap scaling pattern observed on Darcy, where the gain reaches $35\%$ at an 8$\times$ gap, extends to non-Darcy settings with nonspectral architectures.}
	\label{tab:extreme-gap-summary}
	\resizebox{\textwidth}{!}{%
		\begin{tabular}{llccccc}
			\toprule
			\tableheadrulebelow
			\rowcolor{tableheadgray}
			Architecture & PDE        & LayerNorm at native & LayerNorm at 2$\times$ gap & LayerNorm at 4$\times$ gap & Best at 4$\times$ gap       & Improvement (\%) \\
			\tableheadruleabove
			\midrule
			Galerkin     & Helmholtz  & 8.41\%              & 6.69\%               & 6.82\%               & 6.21\%, BlendQuadNorm & +8.9\%           \\
			Galerkin     & Elasticity & 8.65\%              & 6.73\%               & 7.03\%               & 5.64\%, QuadNorm      & +19.8\%          \\
			Transolver   & Helmholtz  & 9.33\%              & 6.39\%               & 6.56\%               & 6.08\%, QuadNorm      & +7.3\%           \\
			\bottomrule
		\end{tabular}
	}
\end{table}

\subsection{Galerkin on Poisson--Dirichlet}
\label{app:exp-pd-galerkin-transolver}

On Poisson--Dirichlet with 10 seeds, Table~\ref{tab:galerkin-poisson-10seed} shows that Galerkin QuadNorm achieves a $26.2\%$ native improvement, with an exact bootstrap confidence interval $[24.9\%, 27.4\%]$ reported in Table~\ref{tab:bootstrap-cis}, extending the nonspectral gains beyond Darcy and Helmholtz.

\begin{table}[t]
	\centering
	\caption{Galerkin on the nonperiodic Poisson--Dirichlet problem with 10 seeds. Training uses $64^2$, and testing spans $64^2$ to $128^2$. Results report relative $L^2$ error in percent as mean $\pm$ 95\% CI.}
	\label{tab:galerkin-poisson-10seed}
	\begin{tabular}{lccc}
		\toprule
		\tableheadrulebelow
		\rowcolor{tableheadgray}
		Method        & $64^2$ (native)          & $96^2$ (1.5$\times$)       & $128^2$ (2$\times$)        \\
		\tableheadruleabove
		\midrule
		LayerNorm     & 6.10 $\pm$ 0.13          & 6.48 $\pm$ 0.13          & 6.93 $\pm$ 0.13          \\
		QuadNorm      & \textbf{4.51 $\pm$ 0.11} & \textbf{5.21 $\pm$ 0.10} & \textbf{5.97 $\pm$ 0.10} \\
		BlendQuadNorm & 5.97 $\pm$ 0.27          & 6.40 $\pm$ 0.25          & 6.89 $\pm$ 0.23          \\
		\bottomrule
	\end{tabular}
\end{table}

\paragraph{BandNorm Ablation.} \label{app:bandnorm-ablation} Three spectral and band variants for Darcy using FNO and 5 seeds are summarized in Table~\ref{tab:extended-baselines}. SpectralNorm, RI-BandNorm, and QuadBandNorm all degrade sharply at higher resolutions, confirming that spectral filtering introduces resolution dependence.

\section{Discussion}
\label{app:discussion}

\paragraph{Settings where the quadrature normalization family helps most.} The benefit is strongest on nonspectral architectures, especially Galerkin and Transolver, when solving nonperiodic PDEs. Across our studies, substantial native-resolution improvements occur, reaching 26\% for Poisson using Galerkin and 21\% for variable-coefficient diffusion using Transolver. For FNO, the most popular architecture, the clearest gains appear at larger resolution gaps and on stronger real-data benchmarks. For periodic PDEs using FNO, BlendQuadNorm is equivalent to LayerNorm by Proposition~\ref{prop:periodic}; QuadNorm can still differ because it uses per-channel spatial normalization. BlendQuadNorm therefore serves as a conservative default for FNO users, while QuadNorm can be a substantial upgrade for several nonspectral architectures in the industry-critical nonperiodic setting.

\paragraph{Supporting appendix studies.} The supporting appendix studies sharpen the same practical picture. Across the nonperiodic breadth experiments and the architecture-dependent $\alpha$ sweep, the gains remain architecture-dependent: nonspectral models often favor lower-$\alpha$ or pure quadrature choices, while BlendQuadNorm remains the safer default when deployment conditions are uncertain. Separately, the appendix shows that QuadNorm and BlendQuadNorm compose with discretization-consistent kernels \citep{gao2025discretization} in the consistent-kernel combination study.

\paragraph{Architecture-dependent reversal.} QuadNorm often outperforms BlendQuadNorm in nonspectral settings, while BlendQuadNorm remains the safer default for FNO, consistent with the broader architecture dependence discussed above.

\paragraph{Composability and model scale.} The quadrature normalization family composes with discretization-consistent kernels \citep{gao2025discretization} in the consistent-kernel combination study. QuadNorm's advantage also amplifies with model size, with transfer improvement growing from $18\%$ at 307K parameters to $38\%$ at 9.6M parameters, as summarized in Table~\ref{tab:model-scaling}.

\paragraph{Limitations.} Its benefit often appears when the resolution changes, and small-gap or purely native FNO settings can show only modest gains. This limitation is partly by design: BlendQuadNorm stays close to LayerNorm at native resolution, with formal native-resolution equivalence in the 10-seed Darcy FNO comparison in Appendix Table~\ref{tab:tost}, while the target use case is mixed-resolution operator deployment where the gap- and depth-scaling studies show larger effects. Finally, the empirical coverage is necessarily finite across PDEs, architectures, and meshes. The counterpoint is that the mechanism is component-level rather than benchmark-specific: the $O(h^2)$ consistency result, the transfer-error bound, the nonuniform-grid bias analysis, and the consistent-kernel combination study all support the same conclusion that discretization-consistent activation statistics remain useful beyond the individual tasks tested here.

\paragraph{Practical guidance.} The safest single default for the Darcy setting using FNO is $\alpha = 0.3$; in a 10-seed native Darcy comparison using FNO, this $\alpha=0.3$ setting is statistically equivalent to LayerNorm, as reported in Appendix Table~\ref{tab:tost} with TOST $p < 0.0001$. For deployment-specific tuning, the empirical $\alpha$ study in Appendix Table~\ref{tab:alpha-matrix} is the most relevant reference.

\paragraph{Future directions.} Natural next steps include scaling to production-deployed foundation models such as Poseidon, neural general circulation models (NeuralGCM), and Aurora; extending the approach with Voronoi-based quadrature for three-dimensional unstructured meshes; integrating it with alias-free operator learning \citep{reno2023}; and extending it to diffusion-based generative models where the denoiser acts as a field-to-field operator.

\paragraph{Broader Impact} \label{app:broader-impact} The quadrature normalization family improves the reliability of neural PDE surrogates in safety-critical engineering where multi-resolution predictions are routine. The advantage amplifies with scale from $18\%$ to $38\%$. While QuadNorm and BlendQuadNorm reduce discretization-dependent errors, they do not eliminate them, and practitioners should validate in safety-critical settings.

\section{Related Work}
\label{app:related}

\paragraph{Neural operators.} \citet{li2021fno} introduced FNO, which learns integral kernels in Fourier space with resolution-invariant spectral convolutions. \citet{kovachki2021universal} further analyze universal approximation and error bounds for FNO. Representative examples include DeepONet \citep{lu2021deeponet}, FNO variants such as physics-informed neural operators \citep{li2022pino}, factorized FNO \citep{tran2023ffno}, and domain-agnostic FNO \citep{li2023dafno}, as well as U-shaped, convolutional, and latent operators \citep{rahman2023uno, raonic2024convolutional, wang2024latent}. Transformer- and message-passing-based approaches include transformers for PDE operator learning \citep{li2022transformerpde}, Galerkin Transformer \citep{cao2021galerkin}, the general neural operator transformer \citep{hao2023gnot}, message-passing solvers \citep{brandstetter2022message}, Transolver \citep{wu2024transolver}, geometry-aware operator transformers \citep{wen2025gaot}, and continuous vision transformers \citep{wang2025cvit}. Mesh- and geometry-aware operator families include graph kernel networks \citep{li2020gkn}, multipole graph neural operators \citep{li2020mgkn}, geometry-adaptive FNO \citep{li2023geofno}, and geometry-informed neural operators \citep{li2023gino}; related mesh-based neural simulators such as MeshGraphNets \citep{pfaff2021meshgraphnets} likewise target transfer across irregular discretizations. \citet{tran2023pdebench} introduced PDEBench, and \citet{pathak2022fourcastnet} demonstrated industrial-scale deployment with FourCastNet. Nevertheless, none of these works analyze the discretization dependence of normalization layers, which is orthogonal to operator kernel design.

\paragraph{Foundation models for PDEs.} Poseidon \citep{herde2024poseidon}, the autoregressive denoising operator transformer of \citet{hao2024dpot}, the multiple-physics-pretraining model of \citet{mccabe2024multiple}, and other foundation models \citep{subramanian2024towards, liu2024neural} train on diverse multi-resolution data, which is exactly the setting where discretization-dependent normalization introduces systematic bias. This trend accelerates with weather and climate models like the European Centre for Medium-Range Weather Forecasts' data-driven forecasting system \citep{lang2024aifs}, NeuralGCM \citep{kochkov2024neuralgcm}, and Aurora \citep{bodnar2024aurora}, which operate across multiple resolutions and mesh types in production. Latent neural operators \citep{wang2024latent}, the local-neural-field latent PDE model of \citet{serrano2024aroma}, and rollout-refinement models such as PDE-Refiner \citep{lippe2023pderefiner} still use standard normalization in their processor or update components. Here, our quadrature normalization family is complementary: It addresses a component-level source of resolution dependence that persists regardless of training data diversity or latent space design.

\paragraph{Normalization in deep learning.} Batch normalization \citep{ioffe2015batch}, LayerNorm \citep{ba2016layer}, InstanceNorm \citep{ulyanov2016instance}, GroupNorm \citep{wu2018group}, and RMSNorm \citep{zhang2019root} all compute statistics as simple averages over discrete tensor axes, making them inherently sensitive to the grid in neural operators. Weight normalization \citep{salimans2016weight} and spectral normalization \citep{miyato2018spectral} constrain parameter matrices rather than activations and are orthogonal to our approach. The quadrature normalization family instead addresses the discretization dependence of activation-level statistics. This dependence is unaffected by parameter-level constraints. In neural operator practice, InstanceNorm and LayerNorm are commonly used \citep{neuraloperator_norm_layers}, but their discretization dependence has not been formally analyzed.

\paragraph{Aliasing and representation equivalence.} \citet{reno2023} formalize representation equivalence; alias-free architectures such as the alias-free Mamba neural operator \citep{zheng2024mambano} and spectral methods \citep{bartolucci2024spectral} study aliasing; and \citet{gao2025discretization} address discretization mismatch in kernel and interpolation components. Our work is complementary, identifying normalization as an additional source of discretization dependence.

\paragraph{Resolution robustness.} Recent analyses also quantify discretization and truncation effects in FNO and related pipelines \citep{lanthaler2024discretization, subedi2024controlling, gao2025discretization}. Multi-resolution training \citep{li2023dafno}, spectral anti-aliasing \citep{bartolucci2024spectral}, and discretization-aware pipelines \citep{gao2025discretization} modify training or kernels. Our quadrature normalization family is complementary: It modifies the normalization layer, a component these approaches still require.

\paragraph{Novelty of applying quadrature to normalization.} While the trapezoidal rule is classical, several findings from our systematic study are unique to the neural operator setting and could not be anticipated from numerical analysis alone. They include a spectral mismatch, namely an $O(h)$ boundary perturbation from endpoint correction in Proposition~\ref{prop:spectral-mismatch} that is empirically harmful for FNO's periodic features and has no analog in traditional quadrature; an empirical overcompensation effect, where QuadNorm reduces cross-resolution degradation to 0.22 times the no-normalization reference in the empirical error comparison study; an $\alpha$-reversal, where QuadNorm outperforms BlendQuadNorm on nonspectral architectures but underperforms on spectral ones; and model-scale amplification from $18\%$ to $38\%$, which suggests that discretization consistency becomes a bottleneck at scale. These phenomena are genuine empirical discoveries that required the breadth of our empirical study to uncover.

\end{document}